\documentclass[twoside,11pt]{article}
\usepackage[abbrvbib,preprint]{jmlr2e}
\usepackage{amsmath}
\usepackage{mathrsfs}
\usepackage{hyperref}
\usepackage{svg}
\usepackage{caption}
\usepackage{amsfonts}
\usepackage{amssymb}
\usepackage{colortbl} 
\usepackage{xcolor}
\usepackage{graphicx}
\usepackage{diagbox}
\usepackage{subfig}
\usepackage{epsfig}%
\usepackage{color}
\usepackage{multirow}
\usepackage{setspace}
\usepackage{enumitem}
\usepackage{hhline}
\usepackage[linesnumbered,ruled,vlined]{algorithm2e}
\usepackage{pbox}
\usepackage{booktabs}
\usepackage{lscape}
\usepackage{array}
\usepackage{epstopdf}
\usepackage[flushleft]{threeparttable}
\usepackage{float}
\usepackage{supertabular}
\usepackage[title]{appendix}
\usepackage{tikz}
\usetikzlibrary{calc}
\usepackage{pifont}
\newcommand{\cmark}{\ding{51}}%
\newcommand{\xmark}{\ding{55}}%



\def\hyphenateAndTtWholeString #1{\xHyphenate#1$\wholeString\unskip}

\def\xHyphenate#1#2\wholeString {\if#1$%
    \else\transform{#1}%
    \takeTheRest#2\ofTheString\fi}

\def\takeTheRest#1\ofTheString\fi
{\fi \xHyphenate#1\wholeString}

\def\transform#1{\url{#1}\hskip 0pt plus 1pt}

\def\urlx #1{\href{#1}{\hyphenateAndTtWholeString{#1}}}

\newlength\thetawidth
\newlength\thetaheight

\usepackage{lastpage}
\jmlrheading{24}{2023}{1-\pageref{LastPage}}{3/23; Revised 5/22}{9/22}{21-0000}{Forough Fazeli-Asl and Michael Minyi Zhang}

\ShortHeadings{A BNP Approach to Generative Models: Integrating VAEs and GANs using WMMD}{Fazeli-Asl and Zhang}
\firstpageno{1}

\SetKwProg{Fn}{Posterior sampling from $\mathrm{DP}(a+n_{mb},H^{\ast})$}{}{}
\SetKwProg{Proc}{Procedure}{}{}
\SetKwProg{Distrain}{Updating Discriminator}{}{}
\SetKwProg{Gentrain}{Updating Encoder \& Generator/Decoder}{}{}
\SetKwProg{CGentrain}{Updating Code-Generator}{}{}
\SetKwProg{gp}{computing gradient penalty}{}{}
\SetKwProg{recterm}{computing reconstruction loss}{}{}
\SetKwProg{regterm}{compute regularization loss}{}{}

\begin{document}

\title{A Bayesian Non-parametric Approach to Generative Models: Integrating Variational Autoencoders and Generative Adversarial Networks using the Wasserstein Distance and Maximum Mean Discrepancy}

\author{\name Forough Fazeli-Asl \email foroughf@hku.hk \\
	\addr Department of Statistics and Actuarial Science\\
	University of Hong Kong\\
	Pok Fu Lam, Hong Kong
	\AND
	\name Michael Minyi Zhang \email mzhang18@hku.hk \\
	\addr Department of Statistics and Actuarial Science\\
	University of Hong Kong\\
	Pok Fu Lam, Hong Kong
    }

\editor{Rajarshi Guhaniyogi}
\maketitle






\begin{abstract}
We propose a novel generative model 
within the \textcolor{black}{Bayesian non-parametric learning (BNPL)} framework to address some notable failure modes in generative adversarial networks (GANs) and variational autoencoders (VAEs)--these being overfitting in the GAN case and noisy samples in the VAE case. We will demonstrate that the BNPL framework enhances training stability and provides robustness and accuracy guarantees when incorporating the Wasserstein distance and maximum mean discrepancy measure (WMMD) into our model's loss function. 
Moreover, we introduce a so-called ``triple model'' that combines the GAN, the VAE, 
and further incorporates a code-GAN (CGAN) to explore the latent space of the VAE. 
\textcolor{black}{This triple model design generates high-quality, diverse samples, while the BNPL framework, leveraging the WMMD loss function, enhances training stability. Together, these components enable our model to achieve superior performance across various generative tasks.} 
These claims are supported by both theoretical analyses and empirical validation on a wide variety of datasets. 

\end{abstract}


\section{\textcolor{black}{Introduction}}
Generative modeling is a method to synthetically create new realistic data samples from an existing dataset. The original GAN, also known as Vanilla GAN, was introduced by \cite{Goodfellow}, and since then, various types of adversarially trained generative models have been developed. These models have exhibited outstanding performance in creating sharp and realistic images by training two neural networks concurrently, one for generating images and the other for discriminating between authentic and counterfeit images. 
These networks are trained in an adversarial manner until the discriminator cannot distinguish between real and fake samples. 
   

Despite GANs being a powerful class of deep-learning models, they still face some notable challenges including mode collapse and training instability. The mode collapse issue in GANs occurs when the generator starts to memorize the training data instead of learning the underlying patterns of the data. This results in the generator becoming too specialized in generating the same samples repeatedly, which leads to a lack of diversity in the generated samples. Unlike GANs, VAEs use a probabilistic approach to encode and decode data, enabling them to learn the underlying distribution and generate diverse samples, albeit with some blurriness. Consequently, the integration of the GAN and the VAE has garnered attention as an intriguing idea for generating high-quality, realistic data samples. This approach allows for the full exploitation of the strengths of both generative models while mitigating their shortcomings. 

Moreover, we can look at deep generative models from the perspective of frequentist and Bayesian methods in both its parametric and non-parametric forms. Parametric generative models assume a specific form for the data distribution, which can lead to overfitting \citep{ma2024multivariate,mescheder2017adversarial}. This occurs when the generator fits the training data too closely, limiting its ability to generalize to new data. 
The frequentist non-parametric (FNP) approach computes the loss function of a generative model with respect to the empirical data distribution, which assumes that the observed data are representative samples of the true population. However, especially with small sample sizes or complex data-generating mechanisms, the FNP approach can also contribute to overfitting, resulting in poor generalization and unstable out-of-sample performance, with the GAN becoming overly sensitive to minor data perturbations \citep{bariletto2024bayesian}. 

In contrast, BNP methods address overfitting in GANs by enabling the generator to adapt to the complexity of the data without adhering to a parametric distribution \textit{a priori}. By avoiding overfitting, we can obtain more robust results in BNPL compared to its frequentist counterparts. 
However, developing a BNP procedure in training deep generative models remains a significant challenge because obtaining a posterior distribution for these models is intractable. 
 

\textcolor{black}{To address the aforementioned problems, we introduce a novel triple generative model that integrates a GAN, a VAE, and a code-GAN (CGAN) within the BNPL framework.} The GAN functions as the primary component, while the VAE and CGAN enhance the model's capabilities. Specifically, to generate more diverse samples, \textcolor{black}{the GAN's generator is replaced with the decoder of a VAE model. Additionally, an extra GAN is employed in the code space to generate supplementary sample codes. The CGAN complements the VAE by exploring underrepresented regions of the code space, promoting comprehensive coverage and mitigating mode collapse.}

To fit this model within the framework of BNPL, we first develop a stochastic representation of the Wasserstein distance using the a Dirichlet process (DP). This allows us to tractably estimate the Wasserstein distance between the DP posterior and the generator's distribution, which is then incorporated into a DP-estimated MMD loss function. \textcolor{black}{By minimizing both the Wasserstein distance and MMD as terms in the loss function of the triple model, 
our proposed BNPL model 
yields robustness during training. 
\textit{In summary, our triple model enhances the quality, diversity, and visualization of the generated outputs, while the BNPL framework improves the training stability by enhancing out-of-sample performance, addresses model misspecification, and provides robustness guarantees against outliers, making it an excellent choice for high-quality image generation.}
}


The structure of the paper is organized as follows:  
In Section~\ref{sec:background}, we provide an overview of traditional learning approaches for generative models, including VAEs and GANs, compared to the BNPL framework. Next, in Section~\ref{sec:wasserstein}, we introduce a probabilistic method to calculate the Wasserstein distance using a DP prior. Then, in Section~\ref{sec:model}, \textcolor{black}{we introduce our triple generative model, a VAE-GAN calibrated with a CGAN, within the BNPL framework.} Afterwards, in Section~\ref{sec:experiments}, we provide experimental results of our novel generative model. Lastly, we conclude our paper in Section~\ref{sec:conclusion} and provide some new directions based on the research presented in this paper. \textcolor{black}{In the appendix, we provide key definitions to the WMMD metric, along with additional theoretical results on the accuracy and robustness grantees of the estimated generator parameters under our BNPL approach.}

\section{Background Work}\label{sec:background}

We begin this section with a review of classic learning approaches to deep generative models from the frequentist perspective. Then, this section continues with an introduction to Bayesian non-parametric learning methods.

\subsection{\textcolor{black}{Deep Generative Modeling}}
\subsubsection{Vanilla GAN}
In the vanilla GAN, the generator tries to minimize the probability that the discriminator correctly identifies the fake sample, while the discriminator tries to maximize the probability of correctly identifying real and fake samples \citep{Goodfellow}. The GAN can be mathematically represented as a minimax game between the generator $\lbrace Gen_{\boldsymbol{\omega}}\rbrace_{\boldsymbol{\omega}\in\boldsymbol{\Omega}}$, \textcolor{black}{a family of neural network functions parameterized by $\boldsymbol{\omega}$ within a compact domain $\boldsymbol{\Omega}\subset\mathbb{R}^{t_1}$}, which maps from the latent space $\mathbb{R}^p$ to the data space $\mathbb{R}^d$, $p<d$, and a discriminator $\lbrace Dis_{\boldsymbol{\theta}}\rbrace_{\boldsymbol{\theta}\in\boldsymbol{\Theta}}$, \textcolor{black}{a family of neural network functions parameterized by $\boldsymbol{\theta}$ within a compact domain $\boldsymbol{\Theta}\subset\mathbb{R}^{t_2}$}, which maps from the data space to the interval $[0,1]$. Specifically, discriminator outputs represent how likely the generated sample is drawn from the true data distribution. 
For a real sample $\mathbf{X}$ from a data distribution $F$, this can be expressed as the objective function
\begin{align*}
    \arg\min\limits_{\boldsymbol{\omega}\in\boldsymbol{\Omega}}\max\limits_{\boldsymbol{\theta}\in\boldsymbol{\Theta}} \mathcal{L}(Gen_{\boldsymbol{\omega}},Dis_{\boldsymbol{\theta}}),
\end{align*}
where $\mathcal{L}(Gen_{\boldsymbol{\omega}},Dis_{\boldsymbol{\theta}})=E_{F}[\ln(Dis_{\boldsymbol{\theta}}(\mathbf{X}))]+E_{F_{\mathbf{Z}_r}}[\ln(1-Dis_{\boldsymbol{\theta}}(Gen_{\boldsymbol{\omega}}(\mathbf{Z}_r)))],$
 $ F_{\mathbf{Z}_r} $ is the distribution of the random noise vector $ \mathbf{Z}_r $, and $ \ln(\cdot) $ denotes the natural logarithm. \textcolor{black}{Throughout the paper, it is assumed that $ F_{\mathbf{Z}_r} $ follows a standard Gaussian distribution}.

GANs are further extended by modifying the generator loss function, $\mathcal{L}_{\text{Gen}}(\boldsymbol{\omega})$, and discriminator loss function, $\mathcal{L}_{\text{Dis}}(\boldsymbol{\theta})$.
To facilitate fair comparisons among different GAN models, we reformulate the vanilla GAN objective function as a sum of 
$\mathcal{L}_{\text{Gen}}(\boldsymbol{\omega})$ and $\mathcal{L}_{\text{Dis}}(\boldsymbol{\theta})$ by defining: 
\begin{align}
    \mathcal{L}_{\text{Gen}}(\boldsymbol{\omega})&= E_{F_{\mathbf{Z}_r}}[\ln(1-Dis_{\boldsymbol{\theta}}(Gen_{\boldsymbol{\omega}}(\mathbf{Z}_r)))],\label{V_Gen}\\
    \mathcal{L}_{\text{Dis}}(\boldsymbol{\theta})&=-E_{F}[\ln(Dis_{\boldsymbol{\theta}}(\mathbf{X}))]-E_{F_{\mathbf{Z}_r}}[\ln(1-Dis_{\boldsymbol{\theta}}(Gen_{\boldsymbol{\omega}}(\mathbf{Z}_r)))]\label{V_dis}.
\end{align}
Now, the vanilla GAN is trained by iteratively updating $\boldsymbol{\omega}$ and $\boldsymbol{\theta}$ 
using stochastic gradient descent to minimize the loss functions \eqref{V_Gen} and \eqref{V_dis}, respectively.
The training procedure for the GAN is finished when the generator can produce samples that are indistinguishable from the real samples, and the discriminator cannot differentiate between them.

\subsubsection{Mode Collapse in GANs}
Instability in GANs occurs when the generator and the discriminator are incapable of converging to a stable equilibrium \citep{kodali2017convergence}. A significant factor contributing to these issues is the tendency of the gradients used to update network parameters to become exceedingly small during the training process, causing the vanishing gradient that leads to a slowdown or even prevention of learning \citep{Arjovsky}. 

To mitigate mode collapse and enhance the stability property in GANs, \cite{salimans2016improved} proposed using a batch normalization technique to normalize the output of each generator layer, which can help reduce the impact of vanishing gradients. 
The authors also implemented mini-batch discrimination, an additional technique to diversify the generated output. This involves computing a pairwise distance matrix among examples within a mini-batch. This matrix is then used to augment the input data before being fed into the model.

Another strategy, widely suggested in the literature to overcome GAN limitations, is incorporating statistical distances into the GAN loss function.
\cite{arjovsky2017wasserstein} suggested updating GAN parameters by minimizing the Wasserstein distance between the distribution of the real and fake data (WGAN). They noted that this distance possesses a superior property compared to other metrics and measures like Kulback-Leibler, Jenson Shanon, and total variation measures. This is due to its ability to serve as a sensible loss function for learning distributions supported by low-dimensional manifolds. 
\cite{arjovsky2017wasserstein} used the weight clipping technique to constrain the discriminator to the $1$-Lipschitz constant. This condition ensures discriminator's weights are bounded to prevent the discriminator from becoming too powerful and overwhelming the generator.  Additionally, this technique helped to ensure that the gradients of the discriminator remained bounded, which is crucial for the stability of the overall training process.

However, weight clipping has some drawbacks. For instance, \cite{gulrajani2017improved} noted that it may limit the capacity of the discriminator, which can prevent it from learning complex functions. Moreover, it can result in a ``dead zone" where some of the discriminator's outputs are not used, which can lead to inefficiencies in training. To address these issues, \cite{gulrajani2017improved} proposed to force the 1-Lipschitz constraint on the discriminator in an alternative way. They improved  WGAN using a gradient penalty term in the loss function to present the WGPGAN model. They showed that it helps to avoid mode collapse and makes the training process more stable.

Instead of comparing the overall distribution of the data and the generator, \cite{salimans2016improved} remarked on adopting the feature matching technique as a stronger method to prevent the mode collapse in GANs and make them more stable. In this strategy, the discriminator is a feature matching distance and the generator is trained to deceive the discriminator by producing images that match the features of real images, rather than just assessing the overall distribution of the data.

\subsubsection{\textcolor{black}{Generative Moment Matching Networks}}
\cite{dziugaite2015training} and \cite{Li} independently proposed the similar techniques, which demonstrated impressive results by applying the MMD measure within a simple GAN framework. \cite{Li} referred to this technique as Generative Moment Matching Networks (GMMNs). Previous work in GMMNs has used the MMD as a loss function to train both the generators to match the fake and real samples, as well as to train the latent space embeddings.

\paragraph{Embedding in the data space:}
\cite{dziugaite2015training} considered the loss function \eqref{MMD-GAN-li} to train the generator:
\begin{align}\label{MMD-GAN-li}
    \arg\min\limits_{\boldsymbol{\omega}\in\boldsymbol{\Omega}} \text{MMD}^2(F,F_{Gen_{\boldsymbol{\omega}}}).
\end{align}

However, \cite{Li} mentioned that by incorporating the square root of the MMD measure into the GAN loss function, the gradients used to update the generator can be more stable, preventing them from becoming too small and leading to gradient vanishing.
\paragraph{Embedding in the latent space:}
\cite{Li} applied the GMMN concept in the code space of an autoencoder (AE+GMMNs) to propose a bootstrap autoencoder model. The training procedure for this model begins with greedy layer-wise pretraining of the auto-encoder \citep{bengio2006greedy}, where each layer is trained individually in an unsupervised manner to compress data into a lower-dimensional representation (code space) that captures the essential features of the original data. This compressed representation is then used to reconstruct (or decode) the original input. After the pretraining phase, the auto-encoder is fine-tuned using backpropagation to improve its reconstruction performance. Once the auto-encoder is fully trained, its weights are frozen, and a GMMN is trained to model the distribution of the code layer. 

The GMMN generator, which is fed with sub-latent noise, generates codes in the code space. The objective is to minimize the MMD between the generated codes and the real data codes, optimizing the GMMN objective on the final encoding layer of the auto-encoder. \cite{Li} demonstrated that this approach significantly reduced noise in the decoded samples. However, since it does not model the latent space probabilistically, it cannot sample or generate novel data points beyond those already seen during the bootstrap procedure.

\subsubsection{Standard VAE}
The VAE consists of an encoder $\lbrace Enc_{\boldsymbol{\eta}}\rbrace_{\boldsymbol{\eta}\in\boldsymbol{\mathfrak{H}}}$\textcolor{black}{, a family of neural network functions parameterized by $\boldsymbol{\eta}$ within a compact domain $\boldsymbol{\mathfrak{H}}\subset\mathbb{R}^{s_1}$,} which maps the input data $\mathbf{X}\sim F$ to a latent representation $\mathbf{Z}_{e}$, and a decoder $\lbrace Dec_{\boldsymbol{\gamma}}\rbrace_{\boldsymbol{\gamma}\in\boldsymbol{\Gamma}}$\textcolor{black}{, a family of neural network functions parameterized by $\boldsymbol{\gamma}$ within a compact domain $\boldsymbol{\Gamma}\subset\mathbb{R}^{s_2}$,} which reconstructs $\mathbf{Z}_{e}$ back to the data space \citep{kingma2013auto}. It uses a hierarchical distribution to model the underlying distribution of the data. Precisely, a prior distribution $F_{\mathbf{Z}_r}$ is first placed on the latent space, $\mathbf{Z}_r\sim F_{\mathbf{Z}_r}$, to specify the distribution of the encoder (variational distribution), $\mathbf{Z}_e:=\mathbf{Z}_r|\mathbf{X}\sim F_{Enc_{\boldsymbol{\eta}}}$, and the distribution of the decoder, $\mathbf{X}|\mathbf{Z}_r\sim F_{Dec_{\boldsymbol{\gamma}}}$, by the reparametrization trick. Then, the intractable data likelihood is approximated by maximizing the marginal log-likelihood:
\begin{align}\label{log-f-dec}
   \ln f_{Dec_{\boldsymbol{\gamma}}}(\mathbf{x})=\ln \int f_{Dec_{\boldsymbol{\gamma}}}(\mathbf{x}|\mathbf{z}_r)f_{\mathbf{Z}_r}(\mathbf{z}_r)\,d\mathbf{z}_r, 
\end{align}
where $f_{Dec_{\boldsymbol{\gamma}}}$ represents the density function corresponding to $F_{Dec_{\boldsymbol{\gamma}}}$. It can be shown that maximizing \eqref{log-f-dec} is equivalent to minimizing:
\begin{align}\label{VAE-loss}
    \mathcal{L}_{\text{VAE}}(\boldsymbol{\eta},\boldsymbol{\gamma})&=
    \text{KL}\left( f_{Enc_{\boldsymbol{\eta}}}(\mathbf{z}_r|\mathbf{x}),f_{\mathbf{Z}_r}(\mathbf{z}_r)
    \right)
    -E_{F_{Enc_{\boldsymbol{\eta}}}(\mathbf{z}_r|\mathbf{x})}\left(
    \ln f_{Dec_{\boldsymbol{\gamma}}}(\mathbf{x}|\mathbf{z}_r)
    \right)\nonumber\\
    &=\mathcal{L}_{\text{Reg}}+\mathcal{L}_{\text{Rec}}
\end{align}
with respect to $\boldsymbol{\eta}$ and $\boldsymbol{\gamma}$. Here, $\text{KL}(\cdot,\cdot)$ denotes Kullback-Leibler divergence, and $\mathcal{L}_{\text{Reg}}$ and $\mathcal{L}_{\text{Rec}}$ represent the regularization and reconstruction errors, respectively. In fact, $\mathcal{L}_{\text{Rec}}$ is the cross-entropy that measures how well the model can reconstruct the input data from the latent space, while $\mathcal{L}_{\text{Reg}}$ encourages the approximate posterior to be close to the prior distribution over the latent space.

Although the latent space in VAEs is a powerful tool for learning the underlying structure of data, it can face limitations in its capacity to capture the complex features of an input image. When the latent space is not able to fully represent all of the intricate details of an image, the resulting reconstructions can be less accurate and lead to unclear outputs. They tend to distribute probability mass diffusely over the data space, increasing the tendency of VAEs to generate blurry images, as pointed out by \cite{theis2015note}. To mitigate the blurriness issue in VAEs, researchers have proposed various modifications such as considering the adversarial loss \citep{makhzani2015adversarial,mescheder2017adversarial} in the VAE objective, improving the encoder and decoder network architectures \citep{yang2017improved, kingma2016improved}, and using denoising techniques \citep{im2017denoising, creswell2018denoising}. However, the methods mentioned earlier still produce images that exhibit a degree of blurriness.

Meanwhile, the idea of integrating GANs and VAEs was first suggested by \cite{larsen2016autoencoding}, using the decoder in the VAE as the generator in the GAN. This model, known as the VAE-GAN, provides an alternative approach to addressing the challenges previously mentioned.
The paper demonstrates the effectiveness of the VAE-GAN model on several benchmark datasets, showing that it outperforms other unsupervised learning methods in terms of sample quality and diversity. \cite{donahue2016adversarial} and \cite{dumoulin2016adversarially} independently proposed another similar approach in the same manner as in the paper by \cite{larsen2016autoencoding}. However, their method incorporated a discriminator that not only distinguished between real and fake samples but also jointly compared the real and code samples (encoder output) with the fake and noise (generator input) samples, which sets it apart from the VAE-GAN model.

Another stronger method was proposed by \cite{rosca2017variational} called $\alpha$-GAN. In this approach, the decoder in a VAE was also replaced with the generator of a GAN. 
Two discriminator networks were then used to optimize the reconstruction and regularization errors of the VAE adversarially. Moreover, a zero-mean Laplace distribution was assigned to the reconstruction data distribution to add an extra term for reconstruction error. This term was considered to provide weights to all parts of the model outputs. Several proxy metrics were employed for evaluating $\alpha$-GAN models. The findings revealed that the WGPGAN is a robust competitor to the $\alpha$-GAN and can even surpass it in certain scenarios. Recently, \cite{kwon2019generation} proposed a 3D GAN by extending $\alpha$-GAN for 3D generations. The authors addressed the stability issues of $\alpha$-GAN and proposed a new hybrid generative model, 3D $\alpha$-WGPGAN, which employs WGPGAN loss to the $\alpha$-GAN loss to enhance training stability. They validated the effectiveness of the 3D $\alpha$-WGPGAN on some 3D MRI brain datasets, outperforming all previously mentioned models.

\subsubsection{3D $\alpha$-WGPGAN}
When $f_{Dec_{\boldsymbol{\gamma}}}(\mathbf{x}|\mathbf{z}_r)$ in \eqref{VAE-loss} is unknown, $\mathcal{L}_{\text{Rec}}$ cannot be employed directly in the training process. One way is assigning a specific distribution to $f_{Dec_{\boldsymbol{\gamma}}}(\mathbf{x}|\mathbf{z}_r)$, like the Laplace distribution which is a common choice in many VAE-based procedures, and then minimize 
$\mathcal{L}_{\text{Rec}}$ \citep{ulyanov2018takes}. However, this procedure can increase subjective constraints in the model \citep{rosca2017variational}. An alternative method in approximating $f_{Dec_{\boldsymbol{\gamma}}}(\mathbf{x}|\mathbf{z}_r)$ is to treat $Dec_{\boldsymbol{\gamma}}$ as the generator of a GAN to train the decoder by playing an adversarial game with the GAN discriminator.  It guarantees the available training data is fully explored through the training process, thereby preventing mode collapse. \textcolor{black}{Hereinafter, we unify the notation by representing both the VAE decoder and GAN generator as \( Gen_{\boldsymbol{\omega}} \), establishing the connection point between the two models.}

The 3D \( \alpha \)-WGPGAN incorporates both of the above structures to form a VAE-GAN model \citep{kwon2019generation}.
\textcolor{black}{The training begins with the VAE, consisting of the \( Enc_{\boldsymbol{\eta}} \) and \( Gen_{\boldsymbol{\omega}} \) components, optimized based on regularization and reconstruction error, as outlined below.} This model 
avoids using Kullback-Leibler divergence in $\mathcal{L}_{\text{Reg}}$ to minimize the regularization error which often considers a simple form like a Gaussian for $f_{Enc_{\boldsymbol{\eta}}}(\mathbf{z}_r|\mathbf{x})$. It defines $\mathcal{L}_{\text{Reg}}$  \textcolor{black}{using the Wasserstein distance \eqref{wasserstein-general-def} as
\begin{align}\label{genLoss-alpha-wgan-cdis}
    \mathcal{L}_{\text{Reg}}:=\text{W}(F_{\mathbf{Z}_r},F_{Enc_{\boldsymbol{\eta}}})=\sup\limits_{\left\|CDis\right\|_{L}\leq1} 
    E_{F_{\mathbf{Z}_r}}[CDis(\mathbf{Z}_{r})]-
    E_{F_{Enc_{\boldsymbol{\eta}}}(\mathbf{z}_r|\mathbf{x})}[CDis(\mathbf{Z}_e)],
\end{align}
where $\mathbf{Z}_{r} \sim F_{\mathbf{Z}_r}$ and $\mathbf{Z}_e \sim F_{Enc_{\boldsymbol{\eta}}}(\mathbf{z}_r|\mathbf{x})$.} 

\textcolor{black}{
Due to the flexibility of neural networks \citep{arjovsky2017wasserstein, gulrajani2017improved}, \cite{kwon2019generation} considered a family of neural network functions with parameter  $\boldsymbol{\xi}$ in a compact domain $\boldsymbol{\Xi}\subset\mathbb{R}^{t}$, defined in \eqref{space-cdis}, to approximate \eqref{genLoss-alpha-wgan-cdis}.
\begin{align}\label{space-cdis}
    \mathcal{D}_{\boldsymbol{\xi}} := \{ CDis_{\boldsymbol{\xi}}\mid\,\small{\textstyle E_{F_{\widehat{\mathbf{Z}}_r}}[\|\nabla_{\widehat{\mathbf{Z}}_r} CDis_{\boldsymbol{\xi}}(\widehat{\mathbf{Z}}_r) \|_2 - 1]^2=0,  \widehat{\mathbf{Z}}_r=(1-u)\mathbf{Z}_r+u\mathbf{Z}_{e}, 0\leq u\leq1, \boldsymbol{\xi}\in\boldsymbol{\Xi}}\}.
\end{align}
Consequently, $\mathcal{L}_{\text{Reg}}$ in \eqref{genLoss-alpha-wgan-cdis} can be computed by
\begin{align*}
    \mathcal{L}_{\text{Reg}}=\max\limits_{\boldsymbol{\boldsymbol{\xi}}\in\,\boldsymbol{\Xi}} 
    E_{F_{\mathbf{Z}_r}}[CDis_{\boldsymbol{\xi}}(\mathbf{Z}_{r})]-
    E_{F_{Enc_{\boldsymbol{\eta}}}(\mathbf{z}_r|\mathbf{x})}[CDis_{\boldsymbol{\xi}}(\mathbf{Z}_e)].
\end{align*}
}

\textcolor{black}{
Consideration of the family of neural networks in Eq. \eqref{space-cdis}  constrains \( CDis_{\boldsymbol{\xi}} \) to satisfy the $1$-Lipschitz condition \citep{gulrajani2017improved}, thereby achieving a meaningful Wasserstein distance metric between the distributions.
\cite{kwon2019generation} referred to $CDis_{\boldsymbol{\xi}}$ as a code-discriminator, as it} engages in an adversarial game with $Enc_{\boldsymbol{\eta}}$ during the minimization of $\mathcal{L}_{\text{Reg}}$ to approximate $f_{Enc_{\boldsymbol{\eta}}}(\mathbf{z}_r|\mathbf{x})$ such that the latent posterior matches to the latent prior $f_{\mathbf{Z}_r}(\mathbf{z}_r)$. 

\cite{kwon2019generation} assigned a Laplace distribution with mean zero and scale parameter $\lambda$ to the generator distribution for decoding codes, expressed as $f_{Gen_{\boldsymbol{\omega}}}(\mathbf{x}|\mathbf{z}_r)\propto e^{-\lambda\left \| \mathbf{x}-Gen_{\boldsymbol{\omega}}(\mathbf{z}_e) \right \|_{1}}$, where $\mathbf{x}$ represents the original input data observed from the distribution $F$. \textcolor{black}{Under this formulation, the reconstruction loss $\mathcal{L}_{\mathrm{Rec}}$ is expressed as:} 
\textcolor{black}{\begin{align}\label{rec-loss-laplace}
    \mathcal{L}_{\mathrm{Rec}}:=\lambda E_{F_{Enc_{\boldsymbol{\eta}}}(\mathbf{z}_r|\mathbf{x})}[\left \| \mathbf{X}-Gen_{\boldsymbol{\omega}}(\mathbf{Z}_e) \right \|_{1}].
\end{align}}
The VAE is then trained by minimizing \( \mathcal{L}_{\mathrm{Reg}} \) and \( \mathcal{L}_{\mathrm{Rec}} \), as defined in Eqs.~\eqref{genLoss-alpha-wgan-cdis} and \eqref{rec-loss-laplace}, respectively.

\textcolor{black}{Next, the 3D \( \alpha \)-WGPGAN focuses on training the GAN, comprising both the \( Gen_{\boldsymbol{\omega}} \) and \( Dis_{\boldsymbol{\theta}} \) components. The GAN is trained by updating the parameters of \( Gen_{\boldsymbol{\omega}} \) and \( Dis_{\boldsymbol{\theta}} \) through minimizing the Wasserstein distance between the fake and real samples, formulated as \eqref{w-dist}, over $\boldsymbol{\Omega}$.
\begin{align}\label{w-dist}
    \text{W}(F,F_{Gen_{\boldsymbol{\omega}}}) = \max\limits_{\boldsymbol{\theta} \in \boldsymbol{\Theta}} E_F[Dis_{\boldsymbol{\theta}}(\mathbf{X})] - E_{F_{Gen_{\boldsymbol{\omega}}}}[Dis_{\boldsymbol{\theta}}(\widetilde{\mathbf{X}})].
\end{align}
}
\textcolor{black}{
Here, \( \widetilde{\mathbf{X}} := (Gen_{\boldsymbol{\omega}}(\mathbf{Z}_r), Gen_{\boldsymbol{\omega}}(\mathbf{Z}_e)) \) represents any generated synthetic samples from the generator, and \( Dis_{\boldsymbol{\theta}} \) belongs to the class of neural network functions \( \mathcal{D}_{\boldsymbol{\theta}} \), defined over a compact domain \( \boldsymbol{\Theta} \). This class, specified in \eqref{space-dis}, imposes the 1-Lipschitz constraint on the discriminator \citep{gulrajani2017improved}.
\begin{align}\label{space-dis}
    \mathcal{D}_{\boldsymbol{\theta}} := \{ Dis_{\boldsymbol{\theta}} \mid E_{F_{\widehat{\mathbf{X}}}}[\|\nabla_{\widehat{\mathbf{X}}} Dis_{\boldsymbol{\theta}}(\widehat{\mathbf{X}}) \|_2 - 1]^2=0,  \widehat{\mathbf{X}} = (1-u) \mathbf{X} + u \widetilde{\mathbf{X}}, \; 0 \leq u \leq 1, \boldsymbol{\theta}\in\boldsymbol{\Theta}\}.
\end{align}
}

\begin{figure*}[t!]
    \centering
    {\includegraphics[width=1\linewidth]{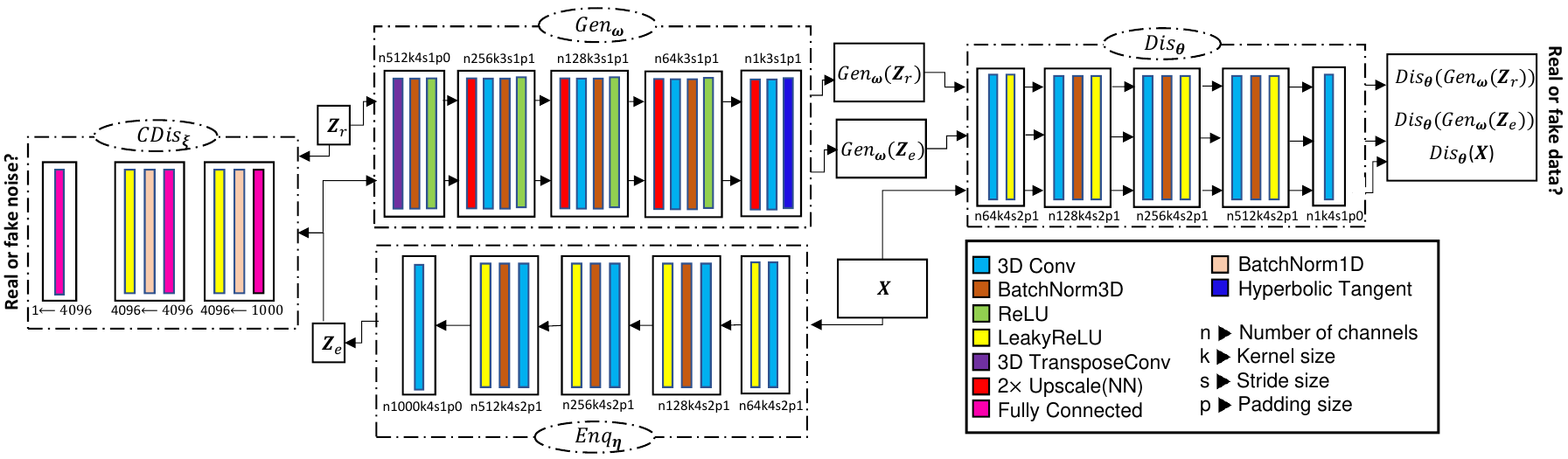}}
    \caption{The general architecture of 3D $\alpha$-WGPGAN comprises three convolutional networks (encoder, generator, and discriminator) and a fully connected-based network (code-discriminator).}
    \label{diagram-3d-Gan}%
\end{figure*}
The training of the 3D \( \alpha \)-WGPGAN is finally achieved by combining the tasks above and minimizing the hybrid loss function:
\begin{align}
    \mathcal{L}_{\text{EGen}}(\boldsymbol{\omega}, \boldsymbol{\eta})&=-E_{F_{Enc_{\boldsymbol{\eta}}}(\mathbf{z}_e|\mathbf{x})}[Dis_{\boldsymbol{\theta}}(Gen_{\boldsymbol{\omega}}(\mathbf{Z}_e))]
    -E_{F_{\mathbf{Z}_r}}[Dis_{\boldsymbol{\theta}}(Gen_{\boldsymbol{\omega}}(\mathbf{Z}_{r}))]\nonumber\\
    &~~~~+\lambda E_{F_{Enc_{\boldsymbol{\eta}}}(\mathbf{z}_e|\mathbf{x})}[\left \| \mathbf{X}-Gen_{\boldsymbol{\omega}}(\mathbf{Z}_e) \right \|_{1}],\label{L-G-wgp}\\
    \mathcal{L}_{\text{Dis}}(\boldsymbol{\theta})&=E_{F_{Enc_{\boldsymbol{\eta}}}(\mathbf{z}_e|\mathbf{x})}[Dis_{\boldsymbol{\theta}}(Gen_{\boldsymbol{\omega}}(\mathbf{Z}_e))]
    +E_{F_{\mathbf{Z}_r}}[Dis_{\boldsymbol{\theta}}(Gen_{\boldsymbol{\omega}}(\mathbf{Z}_{r}))]
    -2E_F[Dis_{\boldsymbol{\theta}}(\mathbf{X})]
    \label{L_D_wgan}\nonumber\\
    &~~~~+\lambda_{p}E_{F_{\widehat{\mathbf{X}}}}[( \| \nabla_{\widehat{\mathbf{X}}}Dis_{\boldsymbol{\theta}}(\hat{\mathbf{X}}) \|_2-1)^2],\\
    \mathcal{L}_{\text{CDis}}(\boldsymbol{\xi})&=E_{F_{Enc_{\boldsymbol{\eta}}}(\mathbf{z}_e|\mathbf{x})}[CDis_{\boldsymbol{\xi}}(\mathbf{Z}_e)]
    -E_{F_{\mathbf{Z}_r}}[CDis_{\boldsymbol{\xi}}(\mathbf{Z}_{r})]
    \nonumber\\
    &~~~~+\lambda_{p}E_{F_{\widehat{\mathbf{Z}}_r}}[( \| \nabla_{\widehat{\mathbf{Z}}_r}Dis_{\boldsymbol{\theta}}(\widehat{\mathbf{Z}}_r) \|_2-1)^2].\nonumber
\end{align}

In Equation \eqref{L_D_wgan}, \( \lambda_{p} \) is the coefficient of the gradient penalty term \( E_{F_{\hat{\mathbf{X}}}}[(\| \nabla_{\hat{\mathbf{x}}}Dis_{\boldsymbol{\theta}}(\hat{\mathbf{x}}) \|_2 - 1)^2] \). Additionally, the first expectation in Equation \eqref{w-dist}, which is independent of \( \boldsymbol{\omega} \), is excluded from the gradient descent minimization with respect to \( \boldsymbol{\omega} \) and thus omitted in Equation \eqref{L-G-wgp}. All expectations mentioned above are approximated using the empirical distribution in practice.

\cite{kwon2019generation} designed the encoder and discriminator networks with five 3D convolutional layers followed by batch normalization (BatchNorm) layers and leaky rectified linear unit (LeakyReLU) activation functions. 
The generator network includes a transpose convolutional (TransposeConv) layer, four 3D convolutional layers, and a BatchNorm layer with a ReLU activation function in each layer. Typically, BatchNorm and ReLU are applied to ensure network stability.
The TransposeConv layer enables the network to ``upsample" the input noise vector and generate an output tensor with a larger spatial resolution. The upscale layers are also implemented in the last four layers of the generator network to increase the spatial resolution of the input feature maps. The code-discriminator consists of three fully connected layers followed by BatchNorm and LeakyReLU activation functions. Figure \ref{diagram-3d-Gan} provides a detailed illustration of the architecture of the 3D $\alpha$-WGPGAN.

\subsection{\textcolor{black}{Bayesian Non-parametric Learning}}
The fundamental concept of BNPL involves placing a prior on the data distribution. 
We begin by introducing the DP prior, a key element of the BNPL framework.

\subsubsection{Dirichlet Process Prior}
The DP, introduced by \cite{Ferguson}, is a widely used prior in BNP methods. It can be seen as a generalization of the Dirichlet distribution, where a random probability measure $F$ is constructed around a fixed probability measure $H$ (the base measure) with variation controlled by a positive real number $a$ (the concentration parameter). Within the context of this statement, $H$ represents the extent of the statistician's expertise in data distribution, while $a$ denotes the level of intensity of this knowledge.

\begin{definition}[\citealt{Ferguson}]\label{DP-foraml-Def}
    The probability measure $F$ is a DP on a space $\mathfrak{X}$ with a $\sigma$-algebra $\mathcal{A}$ of subsets of $\mathfrak{X}$ if, for every measurable partition $A_{1},\ldots,A_{K}$ of $\mathfrak{X}$ with $K\geq2$, the joint distribution of the vector $(F(A_{1}),\ldots,F(A_{K}))$ has a Dirichlet distribution with parameters $(aH(A_{1}),\ldots,aH(A_{K}))$. Moreover, it is assumed that $H(A_{j})=0$ implies $F(A_{j})=0$ with probability one.
\end{definition}

One of the most important properties of the DP is its conjugacy property, where the posterior distribution of $F$ given a sample $\mathbf{X}_{1:n}$ drawn from $F\sim DP(a,H)$, denoted by $F^{\mathrm{pos}}$, is also a DP with concentration parameter $a+n$ and base measure 
$H^{\ast}=a(a+n)^{-1}H+n(a+n)^{-1}F_{\mathbf{X}_{1:n}}$,
where $F_{\mathbf{X}_{1:n}}$ is the empirical cumulative distribution function of the generated sample $\mathbf{X}_{1:n}:=({\mathbf{X}}_{1},\ldots,{\mathbf{X}}_{n})$. This property allows for easy computation of the posterior distribution of $F$.

Alternative definitions for DP have been proposed, including infinite series representations by \cite{bondesson1982simulation} and \cite{sethuraman1994constructive}. 
The method introduced by \cite{sethuraman1994constructive} is commonly referred to as the stick-breaking representation and is widely used for DP inference. However, \cite{zarepour2012rapid} noted that, unlike the series of \cite{bondesson1982simulation}, the stick-breaking representation lacks normalization terms that convert it into a probability measure. Additionally, simulating from infinite series is only feasible with a truncation approach for the terms inside the series.
\cite{Ishwaran} introduced an approximation of DP in the form of a finite series, as shown in \eqref{approx of DP}, which can be easily simulated:
\begin{align}\label{approx of DP}
    F^{\mathrm{pos}}_{N}=\sum_{i=1}^{N}J^{\mathrm{pos}}_{i,N}\delta_{\mathbf{X}^{\mathrm{pos}}_i},
\end{align}
where
\begin{equation*}
    \mathbf{X}^{\mathrm{pos}}_{1},\ldots,\mathbf{X}^{\mathrm{pos}}_{N}\overset{i.i.d.}{\sim}H^{\ast},\hspace{.2 cm}
    (J^{\mathrm{pos}}_{1,N},\ldots,J^{\mathrm{pos}}_{N,N})\sim \mbox{Dir}(\frac{a+n}{N},\ldots,\frac{a+n}{N}),
\end{equation*}
\textcolor{black}{and $\delta$ is the Dirac delta measure, which acts as a point mass at each $\mathbf{X}^{\mathrm{pos}}_{i}$. Specifically, $\delta_{\mathbf{X}^{\mathrm{pos}}_{i}}(A)=1$ if $\mathbf{X}^{\mathrm{pos}}_{i}\in A$ and $\delta_{\mathbf{X}^{\mathrm{pos}}_{i}}(A)=0$ otherwise.} In this paper, the variables $J^{\mathrm{pos}}_{i,N}$ and $\mathbf{X}^{\mathrm{pos}}_{i}$ are used to represent the weight and location of the DP, respectively.

\cite{Ishwaran} demonstrated that $(F^{\mathrm{pos}}_{N})_{N\geq 1}$ converges in distribution to $F^{\mathrm{pos}}$, where $F^{\mathrm{pos}}_{N}$ and $F^{\mathrm{pos}}$ are random values in the space $M_{1}(\mathbb{R})$ of probability measures on $\mathbb{R}$ endowed with the topology of weak convergence. To generate $(J^{\mathrm{pos}}_{i,N})_{1\leq i\leq N}$, one can put $J^{\mathrm{pos}}_{i,N}=\Gamma_{i,N}/\sum_{i=1}^{N}\Gamma_{i,N}$, where $(\Gamma_{i,N})_{1\leq i\leq N}$ is a sequence of independent and identically distributed $\mbox{Gamma}((a+n)/N, 1)$ random variables that are independent of $(\mathbf{X}^{\mathrm{pos}}_{i})_{1\leq i\leq N}$. This form of approximation has been used so far in various applications, including hypothesis testing and GANs, and reflected outstanding results.
It also leads to some excellent outcomes in subsequent sections.

\textcolor{black}{The DP approximation given by \eqref{approx of DP} requires determining the appropriate number of terms to ensure that the truncated series closely resembles the true DP. To determine the appropriate number of terms, we employ a random stopping rule, following the approach introduced by \cite{zarepour2012rapid}. This rule adaptively stops the approximation when the contribution of the current term becomes negligible relative to previous terms. Specifically, given a predefined threshold $\varsigma \in (0,1)$, the stopping point is defined as:}
\begin{align}\label{random-stopping}
    N = \inf\left\{ j\mid\, \frac{\Gamma_{j,j}}{\sum_{i=1}^{j}\Gamma_{i,j}} < \varsigma \right\}.
\end{align}
\textcolor{black}{This method ensures that the series is truncated once additional terms no longer significantly impact the overall approximation, optimizing both computational efficiency and accuracy.}

\subsubsection{\textcolor{black}{BNPL using Minimum Distance Estimation}}
Given the placement of a DP prior on the data distribution, \cite{dellaporta2022robust} developed a procedure to estimate the parameter of interest using minimum distance estimation, following the BNPL strategy outlined in \cite{lyddon2018nonparametric,lyddon2019general,fong2019scalable}. When considering the generator's parameter as the parameter of interest, this strategy defines the generator's parameter as a function of $F^{\text{pos}}$, specifically:
  \begin{equation}\label{bnpl-measure}
   \boldsymbol{\omega}^{\ast}(F^{\text{pos}}):=\arg\min\limits_{\boldsymbol{\omega}\in\boldsymbol{\Omega}} \Delta(F^{pos},F_{Gen_{\boldsymbol{\omega}}}),   
  \end{equation}  
where $\Delta$ is a statistical distance quantifying the difference between two probability measures.

  The key idea is that any posterior distribution over the generator's parameter space, $\boldsymbol{\Omega}$, can be derived by mapping $F^{\text{pos}}$ through the push-forward measure in \eqref{bnpl-measure}, which is visually represented in \citet[Figure 1]{dellaporta2022robust}. Specifically, \cite{dellaporta2022robust} considered $\Delta$ as the MMD measure and applied the following DP approximation:
  \begin{align}\label{app-dp-della}
F^{pos}_{n+N}=\sum_{\ell=1}^{n}\widetilde{J}_{\ell,n}\delta_{\mathbf{X}_{\ell}}+\sum_{t=1}^{N}J_{t,N}\delta_{\mathbf{V}_{t}},
\end{align}
where $(\widetilde{J}_{1:n,n},J_{1:N,N})\sim \mbox{Dirichlet}(1,\ldots,1,\frac{a}{N},\ldots,\frac{a}{N})$, $(\mathbf{X}_{1:n})\overset{i.i.d.}{\sim}F$, and $(\mathbf{V}_{1:N})\overset{i.i.d.}{\sim}H$.

The BNPL approach effectively addresses model misspecification by adapting a non-parametric prior \citep{fong2019scalable,lyddon2018nonparametric,lyddon2019general,lee2024enhancing}, supporting a stable learning process, and enhancing generalization across generative tasks. Furthermore, it reduces the model’s sensitivity to outliers and sample variability by incorporating prior regularization, which enhances the stability of out-of-sample performance compared to traditional learning methods that rely solely on the empirical data distribution for computing loss functions \citep{bariletto2024bayesian}.


\subsubsection{\textcolor{black}{Semi-BNP MMD GAN}}\label{semi-BNP-MMD-sec}
Recently, \cite{fazeli2023semi} proposed training GANs within a BNPL framework using MMD estimation. This method placed a DP prior on the data distribution. Their ``Semi-BNP MMD GAN'' is a robust BNPL procedure that simulates the posterior distribution in the generator’s parameter space by minimizing the MMD distance between the DP posterior and the generator’s distribution. 
Although \citeauthor{fazeli2023semi} established a generalization error bound and a robustness guarantee for their model under the BNPL approach—enhancing training stability and flexibility—the model still remains prone to mode collapse due to its reliance solely on a GAN framework. 
This DP approximation offers a notable advantage over the approximation \eqref{app-dp-della} by significantly reducing the number of terms, thus simplifying both the computational and theoretical complexities. Specifically, the method in \cite{fazeli2023semi} scales with $O(N)$, while the approach in \cite{dellaporta2022robust} scales with $O(N+n)$.  

More preciously, for generated samples $\mathbf{Y}_1,\ldots,\mathbf{Y}_m\sim F_{Gen_{\boldsymbol{\omega}}}$, $\mathbf{Y}_i:=Gen_{\boldsymbol{\omega}}(\mathbf{Z}_{r_i})$, and $\mathbf{Z}_{r_i}\sim F_{\mathbf{Z}_r}$, \cite{fazeli2023semi} propose to consider $\Delta$ in \eqref{bnpl-measure} as the square root of the posterior-based distance defined by: 
\begin{align}\label{BNP-pos-MMD}
\text{MMD}^2(F^{\mathrm{pos}}_{N},F_{Gen_{\boldsymbol{\omega}}(\mathbf{Z}_{r_{1:m}})})&= \sum_{\ell,t=1}^{N} J^{\mathrm{pos}}_{\ell,N}J^{\mathrm{pos}}_{t,N}k(\mathbf{X}^{\mathrm{pos}}_{\ell},\mathbf{X}^{\mathrm{pos}}_{t})
-\dfrac{2}{m}\sum_{\ell=1}^{N}\sum_{t=1}^{m} J^{\mathrm{pos}}_{\ell,N}k(\mathbf{X}^{\mathrm{pos}}_{\ell},\mathbf{Y}_{t})\nonumber\\
&+\dfrac{1}{m^2}\sum_{\ell,t=1}^{m} k(\mathbf{Y}_{\ell},\mathbf{Y}_{t}).
\end{align}

\cite{fazeli2023semi} considered $k(\cdot,\cdot)$ as a  mixture of Gaussian kernels using various bandwidth parameters. For instance, for a set of fixed bandwidth parameters such as $\lbrace \sigma_1,\ldots,\sigma_T \rbrace$ and two vectors $\mathbf{X}^{\mathrm{pos}}_{\ell}$ and $\mathbf{Y}_{t}$, $k(\mathbf{X}^{\mathrm{pos}}_{\ell},\mathbf{Y}_{t})=\sum_{t^{\prime}=1}^{T}\exp\frac{-\|\mathbf{X}^{\mathrm{pos}}_{\ell}-\mathbf{Y}_{t}\|^{2}}{2\sigma_{t^{\prime}}^{2}}$. \textcolor{black}{Particularly, they used the set of values \( \sigma \in \lbrace 2, 5, 10, 20, 40, 80 \rbrace \), which have been shown to perform well in MMD-based training processes \citep{Li}. These values have also been considered in our paper to ensure effective performance in MMD-based computations.}
The generator was implemented using the original architecture proposed by \cite{Goodfellow}, and the model architecture is illustrated in Figure \ref{diagram-rb-BNP1}.

\begin{figure}[t!]
    \centering
    {\includegraphics[width=.99\linewidth]{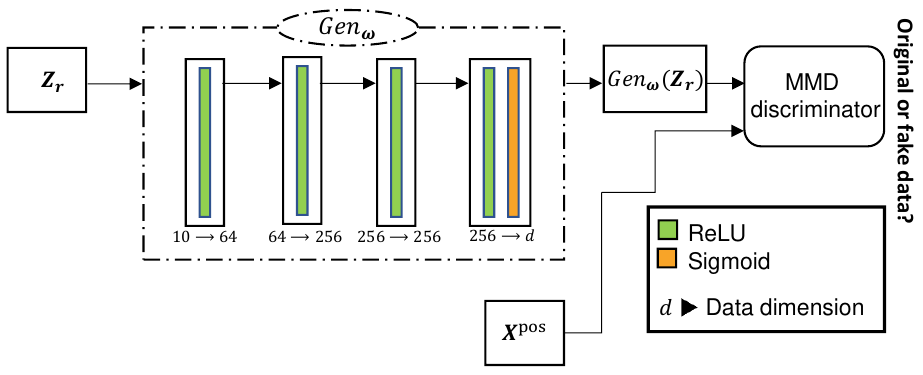}}
    \caption{The general architecture of semi-BNP MMD GAN with four ReLU layers and a sigmoid activation function in the final layer.}
    \label{diagram-rb-BNP1}%
\end{figure}
\section{A Stochastic Representation for Wasserstein Distance}\label{sec:wasserstein}

In this section, we present a stochastic procedure for measuring the Wasserstein distance between a fixed probability measure and a random probability measure modeled by a DP prior, \textcolor{black}{ specifically tailored for applications in BNPL}. 
For a fixed value of $a$ and a probability measure $H$, we model the unknown data distribution $F$ as a random probability measure using the following model:
\begin{align}
    \mathbf{X}&\sim F,\label{cdf}\\
    F^{\mathrm{pri}}:=F&\sim DP(a,H),\label{DP-pri}\\
    F^{\mathrm{pos}}:=F|\mathbf{X}&\sim DP(a+n,H^{\ast}),\label{DP-pos}
\end{align}
where $\mathbf{X}=(\mathbf{X}_1,\ldots,\mathbf{X}_n)$ represent $n$ samples in $\mathbb{R}^d$.
Let \( G \) be any fixed distribution with the ability to simulate samples of any size, and  \textcolor{black}{let \( Dis_{\boldsymbol{\theta}} \) belong to the 1-Lipschitz family of neural networks \( \mathcal{D}^{\prime}_{\boldsymbol{\theta}} \), defined over the compact domain \( \boldsymbol{\Theta}\subset\mathbb{R}^{t_2} \) in \eqref{space-dis-bnp}.
}
\textcolor{black}{
    \begin{align}\label{space-dis-bnp}
    \mathcal{D}^{\prime}_{\boldsymbol{\theta}} := \left\lbrace Dis_{\boldsymbol{\theta}} \mid \small{\textstyle \sum_{i=1}^{N}\frac{[\|\nabla_{\widehat{\mathbf{X}}^{\mathrm{pos}}_i} Dis_{\boldsymbol{\theta}}(\widehat{\mathbf{X}}^{\mathrm{pos}}_i) \|_2 - 1]^2}{N}=0,  \widehat{\mathbf{X}}^{\mathrm{pos}}_i = (1-u) \mathbf{X}^{\mathrm{pos}}_i + u \mathbf{Y}_i, \; 0 \leq u \leq 1, \boldsymbol{\theta}\in\boldsymbol{\Theta}} \right\rbrace.
\end{align}
}

We propose an approximation for Wasserstein between $F^{\mathrm{pos}}$ and $G$ as
\begin{align}\label{W-BNP}
    \text{W}(F^{\mathrm{pos}}_N,G_{\mathbf{Y}_{1:N}})=\max\limits_{\boldsymbol{\theta}\in\boldsymbol{\Theta}} 
    \sum_{i=1}^{N}\left( J^{\mathrm{pos}}_{i,N}Dis_{\boldsymbol{\theta}}(\mathbf{X}_{i}^{\mathrm{pos}})
    -\dfrac{Dis_{\boldsymbol{\theta}}(\mathbf{Y}_i)}{N}\right),
\end{align}
where $(J^{\mathrm{pos}}_{1,N},\ldots,J^{\mathrm{pos}}_{N,N})\sim Dir(\frac{a+n}{N},\ldots,\frac{a+n}{N})$, $\mathbf{X}^{\mathrm{pos}}_{1},\ldots,\mathbf{X}^{\mathrm{pos}}_{N}\overset{i.i.d.}{\sim}H^{\ast}$, and $G_{\mathbf{Y}_{1:N}}$ denotes the empirical distribution corresponding to the sample 
 $Y_1,\ldots,Y_N\overset{i.i.d.}{\sim} G$. \textcolor{black}{We defined \eqref{W-BNP} using the neural network \( Dis_{\boldsymbol{\theta}} \) specifically to facilitate its use in the optimization process via gradient descent, as outlined in the following section.}
The next theorem presents some asymptotic properties of $\text{W}(F^{\mathrm{pos}},G)$ with respect to $N,n,$ and $a$.
\begin{theorem}\label{thm-pos-conv}
Assuming \eqref{cdf}-\eqref{DP-pos}, it follows that for any fixed probability distribution $G$:
\begin{itemize}
    \item[$i.$] $\text{W}(F^{\mathrm{pos}}_N,G_{\mathbf{Y}_{1:N}})\xrightarrow{p}\text{W}(F,G)$ as $N,n\rightarrow\infty$.
    \item[$ii.$] $\text{W}(F^{\mathrm{pos}}_N,G_{\mathbf{Y}_{1:N}})\xrightarrow{p}\text{W}(H,G)$ as $N,a\rightarrow\infty$,
\end{itemize}
\textcolor{black}{where ``$\xrightarrow{p}$" represents convergence in probability} and
$\text{W}(F,G)$ is defined 
in \eqref{w-dist} with 
$F_{Gen_{\boldsymbol{\omega}}}=G$.
\end{theorem}
\begin{proof}
Given that $E_{F^{\mathrm{pos}}}(J^{\mathrm{pos}}_{i,N})=\frac{1}{N}$ for all $i\in\lbrace 1,\ldots,N\rbrace$, we can use Chebyshev's inequality to obtain 
\begin{align}\label{cheb}
    \operatorname {Pr}\left\lbrace \left| J^{\mathrm{pos}}_{i,N}-1/N\right|\geq\epsilon\right\rbrace \leq \frac{Var_{F^{\mathrm{pos}}}(J^{\mathrm{pos}}_{i,N})}{\epsilon^2},
\end{align}
for any $\epsilon>0$. Substituting $Var_{F^{\mathrm{pos}}}(J^{\mathrm{pos}}_{i,N})=\frac{N-1}{N^2(a+n+1)}$ into \eqref{cheb} and choosing $\varsigma$ in \eqref{random-stopping} sufficiently small so that $N>n$, 
yields 
\begin{align}\label{cheb-weight}
    \operatorname {Pr}\left\lbrace \left| J^{\mathrm{pos}}_{i,N}-1/N\right|\geq\epsilon\right\rbrace \leq \frac{N-1}{N^2(a+n+1)\epsilon^2}\leq \frac{1}{n^2}.
\end{align}
The convergence of the series $\sum_{n=1}^{\infty}n^{-2}$ implies that $\sum_{n=1}^{\infty} \operatorname {Pr}\left\lbrace \left| J^{\mathrm{pos}}_{i,N}-1/N\right|\geq\epsilon\right\rbrace<\infty$. 
As $n\rightarrow\infty$ in \eqref{cheb-weight}, the first Borel-Cantelli lemma implies that 
$ J^{\mathrm{pos}}_{i,N}\xrightarrow{a.s.}1/N$.
 Additionally, as $n$ approaches infinity, the Glivenko-Cantelli theorem implies that $F_{\mathbf{X}_{1:n}}$ converges to $F$ and subsequently, $H^{\ast}$ converges to $F$. This convergence indicates that as $n \rightarrow \infty$, the probability of drawing a sample from $F$ approaches 1, implying that $\mathbf{X}_i^{\mathrm{pos}} \rightarrow \mathbf{X}_i$, where $\mathbf{X}_i$ is a random variable distributed according to $F$, for $i = 1, \ldots, N$. By applying the continuous mapping theorem, it follows that  $Dis_{\boldsymbol{\theta}}(\mathbf{X}^{\mathrm{pos}}_{i})$ converges to  $Dis_{\boldsymbol{\theta}}(\mathbf{X}_{i})$ for any $\boldsymbol{\theta}\in\boldsymbol{\Theta}$, as $n\rightarrow\infty$, and thus, we have
 \begin{align}\label{conv-sum}
     \mathcal{I}^{\mathrm{pos}}(\boldsymbol{\theta}):=\sum_{i=1}^{N}\left( J^{\mathrm{pos}}_{i,N}Dis_{\boldsymbol{\theta}}(\mathbf{X}_{i}^{\mathrm{pos}})
    -\frac{Dis_{\boldsymbol{\theta}}(\mathbf{Y}_i)}{N}\right)\xrightarrow{p} \dfrac{1}{N} \sum_{i=1}^{N}(Dis_{\boldsymbol{\theta}}(\mathbf{X}_i)-Dis_{\boldsymbol{\theta}}(\mathbf{Y}_i)).
 \end{align}
Applying the strong law of large numbers to the right-hand side of Equation \eqref{conv-sum} yields
\begin{align*}
    \mathcal{I}^{\mathrm{pos}}(\boldsymbol{\theta})\xrightarrow{p} E_{F}[Dis_{\boldsymbol{\theta}}(\mathbf{X}_1)]-E_{G}[Dis_{\boldsymbol{\theta}}(\mathbf{Y}_1)],
\end{align*}
as $N\rightarrow\infty.$ Since $\max(\cdot)$ is a continuous function, the proof of (i) is completed by using the continuous mapping theorem. The proof of (ii) is followed by a similar approach.
\end{proof}

The next corollary demonstrates a crucial property of the $\text{W}(F^{\mathrm{pos}},G)$ metric, which makes it a convenient tool for comparing two models.
\begin{corollary} 
    Assuming the conditions of Theorem \ref{thm-pos-conv}, then $\text{W}(F^{\mathrm{pos}}_N,G_{\mathbf{Y}_{1:N}})\xrightarrow{p} 0$ , if and only if $G=F$, as $N,n\rightarrow \infty$.
\end{corollary}
\begin{proof}
    The proof is completed by invoking the equality condition of the Wasserstein distance in part (i) of Theorem \ref{thm-pos-conv}.
\end{proof}

The next Lemma provides a lower bound for the expectation of $\text{W}(F_{N}^{\mathrm{pos}},G_{\mathbf{Y}_{1:N}})$. 
\begin{lemma}
    Under assumptions of Theorem \ref{thm-pos-conv}, we have 
    \begin{align*}
        E[\text{W}(F_{N}^{\mathrm{pos}},G_{\mathbf{Y}_{1:N}})]\geq \text{W}(H^{\ast},G). 
    \end{align*}
\end{lemma}
\begin{proof}
    By virtue of the convexity property of the maximum function, Jensen's inequality implies
    \begin{align}
        E\left[\max\limits_{\boldsymbol{\theta}\in\boldsymbol{\Theta}} 
    \sum_{i=1}^{N}\left( J^{\mathrm{pos}}_{i,N}Dis_{\boldsymbol{\theta}}(\mathbf{X}_{i}^{\mathrm{pos}})
    -\dfrac{Dis_{\boldsymbol{\theta}}(\mathbf{Y}_i)}{N}\right)\right]&\geq \max\limits_{\boldsymbol{\theta}\in\boldsymbol{\Theta}}E\left[ \sum_{i=1}^{N}\left( J^{\mathrm{pos}}_{i,N}Dis_{\boldsymbol{\theta}}(\mathbf{X}_{i}^{\mathrm{pos}})
    -\dfrac{Dis_{\boldsymbol{\theta}}(\mathbf{Y}_i)}{N}\right)\right]\nonumber\\
    &=\max\limits_{\boldsymbol{\theta}\in\boldsymbol{\Theta}} \sum_{i=1}^{N}\dfrac{1}{N}\left( E_{H^{\ast}}[Dis_{\boldsymbol{\theta}}(\mathbf{X}_{i}^{\mathrm{pos}})]
    -E_{G}[Dis_{\boldsymbol{\theta}}(\mathbf{Y}_i)]\right)\label{eq2}\\
    &=\max\limits_{\boldsymbol{\theta}\in\boldsymbol{\Theta}}
    E_{H^{\ast}}[Dis_{\boldsymbol{\theta}}(\mathbf{X}_{1}^{\mathrm{pos}})]
    -E_{G}[Dis_{\boldsymbol{\theta}}(\mathbf{Y}_1)].\label{eq3}
    \end{align}
Equation \eqref{eq2} is derived from the property of the Dirichlet distribution, while Equation \eqref{eq3} is a result of identical random variables $(\mathbf{X}_{i}^{\mathrm{pos}})_{1\leq i\leq N}$ and $(\mathbf{Y}_{i})_{1\leq i\leq N}$. 
\end{proof}

\textcolor{black}{As noted, the approximation  \eqref{W-BNP} is specifically designed for use within the BNPL procedure, and its application outside this context is unwarranted. Nevertheless, a brief comparison with its frequentist counterpart, based on simple resampling of the empirical data, is informative. While some may argue that a straightforward resampling approach offers a more direct path, this discussion aims to clarify key distinctions and reinforce the rationale behind choosing the DP framework. In approximating the Wasserstein distance using a resampling method, it suffices to substitute \( J^{\mathrm{pos}}_{i,N} = 1/N \) and \( \mathbf{X}^{\mathrm{pos}}_{i} = \mathbf{X}^{\prime}_{i} \) into \eqref{W-BNP}, where \( \mathbf{X}^{\prime}_{1}, \ldots, \mathbf{X}^{\prime}_{N} \overset{i.i.d.}{\sim} F_{\mathbf{X}_{1:n}} \) are obtained through simple resampling with replacement. This yields an estimator \( \text{W}(F_{\mathbf{X}^{\prime}_{1:N}}, G_{\mathbf{Y}_{1:N}}) \) directly based on the data points.}



\textcolor{black}{The base measure \( H \) in the DP provides distributional regularization by incorporating prior knowledge into the model, reducing sensitivity to sample variability and enhancing stability in the approximation \eqref{W-BNP}. In contrast, resampling may only increase sample size without introducing any prior influence or regularization. As a result, the approximation \( \text{W}(F_{\mathbf{X}^{\prime}_{1:N}}, G_{\mathbf{Y}_{1:N}}) \) may still exhibit non-smooth behavior, with sharper changes in the approximation, particularly in the presence of high variability or outliers. This effect becomes even more pronounced in mini-batch settings with small samples, leading to greater variability in the Wasserstein approximation across iterations and potential instability in training.} 

\textcolor{black}{These points highlight the improved out-of-sample performance achieved by BNP estimation procedures compared to the FNP methods, which, by only considering the empirical data distribution, suffer from a lack of out-of-sample performance \citep{bariletto2024bayesian}. Additionally, while \( N \) is selected in the DP context according to \eqref{random-stopping}, there is no principled basis for choosing the value of \( N \) in this resampling approach. These points also hold in the context of the MMD distance and will be demonstrated in Section \ref{sec:BNP vs FNP} with supporting experiments.}

\section{\textcolor{black}{A Triple Generative Model within BNPL Perspective}}\label{sec:model}
In this section, we introduce a triple generative model—comprising a VAE, a CGAN, and a GAN—within the BNPL framework. The BNPL approach is selected for its ability to provide distributional regularization by embedding expert knowledge through the prior distribution, resulting in a robust and stabilized training process with enhanced out-of-sample performance. Furthermore, it effectively addresses model misspecification, enhancing its applicability to challenging generative scenarios.

\subsection{\textcolor{black}{Learning Structure}}
Consider the discriminator \( Dis_{\boldsymbol{\theta}} \), which belongs to the 1-Lipschitz family of neural networks \( \mathcal{D}^{\prime}_{\boldsymbol{\theta}} \), and the generator \( \{ Gen_{\boldsymbol{\omega}} \}_{\boldsymbol{\omega} \in \boldsymbol{\Omega}} \), a family of neural network functions parameterized by \( \boldsymbol{\omega} \) over a compact domain \( \boldsymbol{\Omega} \). The generator maps latent inputs \( \{ \mathbf{Z}_i \}_{i=1}^N \), where \( \mathbf{Z}_i \in \mathbb{R}^p \), into the data space \( \mathbb{R}^d \).
 Here, \( \mathbf{Z}_i \) can be considered any latent variable generated either from the noise distribution or from the encoder distribution.
 We then employ BNPL by selecting \( \Delta \) as the distance defined by \eqref{w-mmd-pos}—a combination of BNP distances given by \eqref{W-BNP} and \eqref{BNP-pos-MMD}—within objective function \eqref{bnpl-measure}, leading to an implicit approximation of the posterior distribution over \( \boldsymbol{\Omega} \).
\begin{align}
    \text{WMMD}(F^{\mathrm{pos}}_{N},F_{Gen_{\boldsymbol{\omega}}(\mathbf{Z}_{{1:N}})})&:=\text{W}(F^{\mathrm{pos}}_{N},F_{Gen_{\boldsymbol{\omega}}(\mathbf{Z}_{1:N})})+\text{MMD}(F^{\mathrm{pos}}_{N},F_{Gen_{\boldsymbol{\omega}}(\mathbf{Z}_{1:N})}).\label{w-mmd-pos}
\end{align}

This choice of $\Delta$ not only leverages overall distribution comparison but also results in better feature matching between the generated and original samples by capturing different aspects of the two distributions, thereby improving the stability and quality of the generated samples. The WMMD will contribute to the both components of a VAE-GAN. 
\textcolor{black}{The complete framework of the triple model—a VAE-GAN calibrated with a CGAN—is then proposed through three main tasks. The GAN serves as the core of the model, which is enhanced by substituting the GAN generator with the VAE decoder. Figure \ref{diagram-2d-Gan-BNP} illustrates this connection across all three tasks.} This modification plays a crucial role in mitigating mode collapse in the generator and improving its ability to produce sharp images. Our CGAN draws inspiration from the AE+GMMN framework, embedding a GMMN within the latent space of the VAE. This additional step encourages the generator to produce images with less noise.

\subsubsection{\textcolor{black}{Task 1: Training VAE (Encoder+Generator)}}
We design our VAE model with an encoder \( \lbrace Enc_{\boldsymbol{\eta}} \rbrace_{\boldsymbol{\eta} \in \boldsymbol{\mathfrak{H}}} \), a family of neural network functions parameterized by \( \boldsymbol{\eta} \) over a compact domain \( \boldsymbol{\mathfrak{H}} \), and the generator \( \lbrace Gen_{\boldsymbol{\omega}} \rbrace_{\boldsymbol{\omega} \in \boldsymbol{\Omega}} \)—serving as the connection point between Task 1 and Task 3—by defining its relevant errors as follows. This task is then performed by minimizing these errors to update $\boldsymbol{\eta}$ and $\boldsymbol{\omega}$.

\paragraph{Regularization error:}
Consider the posterior sample \(  \mathbf{X}^{\mathrm{pos}}_{1:N}  \) generated by \eqref{approx of DP}, along with random noise vectors \textcolor{black}{\( \mathbf{Z}_{r_1}, \ldots, \mathbf{Z}_{r_N} \overset{i.i.d.}{\sim} F_{\mathbf{Z}_r} \)} and real codes \textcolor{black}{\( \mathbf{Z}_{e_1}, \ldots, \mathbf{Z}_{e_N} \overset{i.i.d.}{\sim} F_{Enc_{\boldsymbol{\eta}}} \)}, where $\mathbf{Z}_{r_i}$ and $\mathbf{Z}_{e_i}:=Enc_{\boldsymbol{\eta}}(\mathbf{X}^{\mathrm{pos}}_i)$ belongs to $\mathbb{R}^p$. To approximate the variational distribution \( F_{Enc_{\boldsymbol{\eta}}} \), we propose minimizing \( \mathcal{L}_{\text{Reg}} \) over $\boldsymbol{\mathfrak{H}}$ using the MMD distance \eqref{MMD-ecdf}, defined as:

\begin{align}\label{reg-mmd}
\mathcal{L}_{\text{Reg}} := \text{MMD}(F_{\mathbf{Z}_{r_{1:N}}}, F_{\mathbf{Z}_{e_{1:N}}}),    
\end{align}
where \( F_{\mathbf{Z}_{r_{1:N}}} \) and \( F_{\mathbf{Z}_{e_{1:N}}} \) are the empirical distributions corresponding to the samples generated from \( F_{\mathbf{Z}_r} \) and \( F_{Enc_{\boldsymbol{\eta}}} \), respectively.
Here, $F_{\mathbf{Z}_r}$ treats as the distribution of the real noise while $F_{Enc_{\boldsymbol{\eta}}}$ treats as the distribution of the fake noise, and thus, minimizing \eqref{reg-mmd} implies that $\mathbf{Z}_{e_i}$ is well-matched to $\mathbf{Z}_{r_i}$, enabling the generator to thoroughly cover the decoded space. This suggests that the generator has effectively prevented mode collapse \citep{kwon2019generation, jafari2023improved}. It is also noted that the distance used in \eqref{reg-mmd} fell outside the scope of the BNP framework, as it involves comparisons between parametric distributions.

In contrast to \cite{kwon2019generation}, which employed the Wasserstein distance in the code space (a code-discriminator network) to compute $\mathcal{L}_{\text{Reg}}$, leading to significant computational overhead, the computation of  $\mathcal{L}_{\text{Reg}}$ using \eqref{reg-mmd} is network-free, resulting in reduced computational cost of VAE training.


\paragraph{Reconstruction error:} We use the posterior-based WMMD distance \eqref{w-mmd-pos} to compute $\mathcal{L}_{\text{Rec}}$, defined as:
\begin{align}\label{loss-rec-bnp}
    \mathcal{L}_{\text{Rec}}:=\text{WMMD}(F^{\mathrm{pos}}_N,F_{Gen_{\boldsymbol{\omega}}(\mathbf{Z}_{e_{1:N}})}).
\end{align}
Therefore, minimizing \( \mathcal{L}_{\text{Rec}} \) 
over $\boldsymbol{\Omega}$ ensures that the reconstructed samples \( Gen_{\boldsymbol{\omega}}(\mathbf{Z}_{e_i}) \) closely match the posterior samples \( \mathbf{X}^{\mathrm{pos}}_i \) in the data space.

\subsubsection{\textcolor{black}{Task 2: Training CGAN (A GMMN in the Latent Space)}}

To enhance the coverage of the code space, we incorporate a GMMN into the code space of our VAE model. This serves as our CGAN, providing a calibration mechanism for the VAE's code space. The CGAN includes a generator \( \lbrace CGen_{\boldsymbol{\omega}^\prime} \rbrace_{\boldsymbol{\omega}^\prime \in \boldsymbol{\Omega}^\prime} \), a family of neural network functions parameterized by \( \boldsymbol{\omega}^\prime \) over a compact domain \( \boldsymbol{\Omega}^\prime \subset\mathbb{R}^{t^{\prime}}\). The code-generator takes random noise samples \textcolor{black}{\( \mathbf{Z}^\prime_{r_1}, \ldots, \mathbf{Z}^\prime_{r_N} \overset{i.i.d.}{\sim} F_{\mathbf{Z}^{\prime}_{r}} \)} from a sub-latent space \( \mathbb{R}^q \) and outputs fake code samples \textcolor{black}{\( \widetilde{\mathbf{Z}}_{e_1}, \ldots, \widetilde{\mathbf{Z}}_{e_N} \overset{i.i.d.}{\sim} F_{CGen_{\boldsymbol{\omega}^\prime}} \)} in the latent space \( \mathbb{R}^p \), where  
$
\widetilde{\mathbf{Z}}_{e_i} := CGen_{\boldsymbol{\omega}^\prime}(\mathbf{Z}^\prime_{r_i})$, $ q < p,
$
and \textcolor{black}{\( F_{\mathbf{Z}^{\prime}_{r}} \) is typically assumed to follow a standard Gaussian distribution}. 

The objective function \eqref{cgen-obj} is optimized to train \( CGen_{\boldsymbol{\omega}^\prime} \), treating \( F_{Enc_{\boldsymbol{\eta}}} \) as the distribution of the real code and \( F_{CGen_{\boldsymbol{\omega}^\prime}} \) as the distribution of the fake code:
\begin{align}\label{cgen-obj}
    \arg\min\limits_{\boldsymbol{\omega}^\prime\in\boldsymbol{\Omega}^{\prime}} \text{MMD}(F_{\mathbf{Z}_{e_{1:N}}}, F_{\widetilde{\mathbf{Z}}_{e_{1:N}}}).
\end{align}

The CGAN fills in gaps or unexplored areas of the code space that the VAE may have missed, resulting in better code space coverage and reducing the risk of mode collapse. By generating more diverse code samples, \( CGen_{\boldsymbol{\omega}^\prime} \) improves the VAE’s performance. After training the code-generator, the updated weights \( \boldsymbol{\omega}^\prime \) are frozen, and an additional reconstruction error \eqref{rec-wmmd-extra} is incorporated into the VAE training:
\begin{align}\label{rec-wmmd-extra}
\mathcal{L}^{\prime}_{\text{Rec}} := \text{WMMD}(F^{\mathrm{pos}}_N, F_{Gen_{\boldsymbol{\omega}}(\widetilde{\mathbf{Z}}_{e_{1:N}})}).
\end{align}

This encourages \( Gen_{\boldsymbol{\omega}} \) to produce higher-quality images with reduced noise, particularly in small-sample settings like mini-batches \citep{Li}.

\subsubsection{\textcolor{black}{Task 3: Training GAN (Generator+Discriminator)}}
The GAN task is finally achieved by minimizing \eqref{wmmd-pos-final}, which updates \( \boldsymbol{\omega} \) and \( \boldsymbol{\theta} \):
\begin{align}\label{wmmd-pos-final}
    \mathcal{L}_{\text{GAN}}:=\text{WMMD}(F^{\mathrm{pos}}_N, F_{Gen_{\boldsymbol{\omega}}(\mathbf{Z}_{r_{1:N}})}).
\end{align}

\subsubsection{\textcolor{black}{Hybrid Loss Function (Integration of Tasks 1, 2, and 3)}}

The parameters of the triple model are updated using stochastic gradient descent in the following sequence. First, \( \boldsymbol{\theta} \) is updated by maximizing \eqref{loss-rec-bnp}, \eqref{rec-wmmd-extra}, and \eqref{wmmd-pos-final} over \( \boldsymbol{\Theta} \), while keeping \( \boldsymbol{\omega} \), \( \boldsymbol{\eta} \), and \( \boldsymbol{\omega}^\prime \) fixed to approximate the Wasserstein component of the distance in \eqref{w-mmd-pos}. During this step, posterior MMD-based components of \eqref{w-mmd-pos} that are independent of \( \boldsymbol{\theta} \) are ignored. Next, \( \boldsymbol{\omega}^\prime \) is updated by optimizing \eqref{cgen-obj} with \( \boldsymbol{\eta} \) held constant. Finally, \( \boldsymbol{\omega} \) and \( \boldsymbol{\eta} \) are updated simultaneously by minimizing \eqref{reg-mmd}, \eqref{loss-rec-bnp}, \eqref{rec-wmmd-extra}, and \eqref{wmmd-pos-final} over \( \boldsymbol{\Omega} \) and \( \boldsymbol{\mathfrak{H}} \), while keeping \( \boldsymbol{\theta} \) and \( \boldsymbol{\omega}^\prime \) fixed.

These updates collectively minimize the hybrid loss function \eqref{hybrid-loss}:

\begin{subequations}\label{hybrid-loss}
\begin{align}
    \mathcal{L}_{\text{Dis}}(\boldsymbol{\theta})&=\sum_{i=1}^{N}\bigg( \dfrac{Dis_{\boldsymbol{\theta}}(Gen_{\boldsymbol{\omega}}(\mathbf{Z}_{r_{i}}))
    +Dis_{\boldsymbol{\theta}}(Gen_{\boldsymbol{\omega}}(\mathbf{Z}_{e_{i}}))+Dis_{\boldsymbol{\theta}}(Gen_{\boldsymbol{\omega}}(\widetilde{\mathbf{Z}}_{e_{i}}))}{N}\nonumber\\
    &~~~~~~~~~~~~~~~-3J^{\mathrm{pos}}_{i,N}Dis_{\boldsymbol{\theta}}(\mathbf{X}_{i}^{\mathrm{pos}})\bigg)+\lambda_{p} L_{\text{GP-Dis}},\\
    \mathcal{L}_{\text{CGen}}(\boldsymbol{\omega}^{\prime})&=\text{MMD}(F_{\mathbf{Z}_{e_{1:N}}},F_{\widetilde{\mathbf{Z}}_{e_{1:N}}}),\\
        \mathcal{L}_{\text{EGen}}(\boldsymbol{\omega}, \boldsymbol{\eta})&=-\sum_{i=1}^{N}\left( \dfrac{Dis_{\boldsymbol{\theta}}(Gen_{\boldsymbol{\omega}}(\mathbf{Z}_{r_{i}}))
    +Dis_{\boldsymbol{\theta}}(Gen_{\boldsymbol{\omega}}(\mathbf{Z}_{e_{i}}))+Dis_{\boldsymbol{\theta}}(Gen_{\boldsymbol{\omega}}(\widetilde{\mathbf{Z}}_{e_{i}}))}{N}\right)\nonumber\\
    &+\text{MMD}(F^{\mathrm{pos}}_N,F_{Gen_{\boldsymbol{\omega}}(\mathbf{Z}_{r_{1:N}})})+\text{MMD}(F^{\mathrm{pos}}_N,F_{Gen_{\boldsymbol{\omega}}(\mathbf{Z}_{e_{1:N}})})\label{pos-mmd-Legenn}\nonumber\\
    &+\text{MMD}(F^{\mathrm{pos}}_N,F_{Gen_{\boldsymbol{\omega}}(\widetilde{\mathbf{Z}}_{e_{1:N}})})+\text{MMD}(F_{\mathbf{Z}_{r_{1:N}}},F_{\mathbf{Z}_{e_{1:N}}}),
\end{align}
\end{subequations}
 where terms that are independent of $\boldsymbol{\omega}$ and $\boldsymbol{\eta}$ in \eqref{loss-rec-bnp}, \eqref{rec-wmmd-extra}, and \eqref{wmmd-pos-final} are excluded from \eqref{pos-mmd-Legenn}, as they do not contribute to gradient descent with respect to these parameters.
\textcolor{black}{Here, the gradient penalty is given by
\begin{align*}
    L_{\text{GP-Dis}}=\frac{1}{N}\sum_{\ell=1}^{3}\sum_{i=1}^{N} \left(\left \| \nabla_{\widehat{\mathbf{x}}^{\mathrm{pos}}_{\ell i}}Dis_{\boldsymbol{\theta}}(\widehat{\mathbf{X}}^{\mathrm{pos}}_{\ell i})\right \|_2-1\right)^2,
\end{align*}
to constrain \( Dis_{\boldsymbol{\theta}} \) to satisfy the 1-Lipschitz family \eqref{space-dis-bnp}, where
\[
\widehat{\mathbf{X}}^{\mathrm{pos}}_{\ell i} =(1-u)\mathbf{X}^{\mathrm{pos}}_{i} +u\times
\begin{cases}
   Gen_{\boldsymbol{\omega}}(\mathbf{Z}_{r_i}) & \text{for $\ell=1$}, \\[10pt]
  Gen_{\boldsymbol{\omega}}(\mathbf{Z}_{e_i})& \text{for $\ell=2$}, \\[10pt]
  Gen_{\boldsymbol{\omega}}(\widetilde{\mathbf{Z}}_{e_i}) & \text{for $\ell=3$},
\end{cases}
\]
for $0\leq u\leq1$. We set $\lambda_p$ to 10, a widely recommended value based on \cite{gulrajani2017improved}, as it performed well across a range of architectures and datasets, from simpler tasks to large-scale models. Algorithm \ref{alg2} provides a detailed outline of this training procedure.} 

The distance in \eqref{w-mmd-pos}, forming the central component of Algorithm \ref{alg2}, is interpreted as an estimator of \( \text{WMMD}(F, F_{Gen_{\boldsymbol{\omega}}}) \). The Theorem \ref{thm-Generror} in Appendix \ref{app:A} examines the asymptotic accuracy of \( \text{WMMD}(F, F_{Gen_{\boldsymbol{\omega}}}) \) based on the optimized generator parameter obtained through the BNPL procedure. Furthermore, to understand the impact of outliers on the generator's performance, we consider Huber's contamination model \citep{cherief2022finite}, where the data consists of a mixture of a clean distribution and a noise distribution representing outliers. Theorem \ref{thm-robusst} in Appendix 
\ref{app:B} provides an upper bound on the discrepancy between the clean distribution and the generated distribution, demonstrating that the generator's performance is affected by the contamination rate, but remains controlled as the rate decreases.

\subsection{Selecting the BNP Hyperparameters}
The base measure $H$ in the BNP model defined by \eqref{cdf}-\eqref{DP-pos} reflects the expert's opinion about the data distribution, instead of assuming a specific distribution for the data population. A Gaussian distribution is a common choice for $H$, parameterized by the sample mean and sample covariance:
\begin{align}\label{H-parameters}
    \bar{\mathbf{X}}=\dfrac{1}{n}\sum_{i=1}^{n}\mathbf{X}_i,\hspace{1cm} S_{\mathbf{X}}=\sum_{i=1}^{n}(\mathbf{X}_i-\bar{\mathbf{X}})(\mathbf{X}_i-\bar{\mathbf{X}})^T.
\end{align}
In our proposed BNPL framework, we use this choice of $H$ in the DP prior \eqref{DP-pri} to model the data distribution. \textcolor{black}{This prior keeps the model close to the data samples while allowing adaptability to diverse points. It mitigates the influence of outliers, providing a stabilizing regularization effect across the dataset. However, this is a general choice, and one may select a specialized prior based on specific knowledge or requirements.}

Furthermore, we choose the optimal value of the concentration parameter $a$ to be its \textit{maximum a posteriori} (MAP) estimate. This is accomplished by maximizing the log-likelihood of $F^{\mathrm{pos}}$ 
with respect to $a$.
\textcolor{black}{Specifically, let \(\boldsymbol{A}\) be a matrix with \(K \geq 2\) columns, where each row \(\boldsymbol{A}_{i,\cdot} = (A_{i,1}, \ldots, A_{i,K})\) defines a measurable partition of the sample space \(\mathfrak{X}\) for \(i \in \{1, \ldots, n_P\}\), with the rows \(\boldsymbol{A}_{i,\cdot}\) being independent.
According to Definition \ref{DP-foraml-Def} and considering the conjugacy property of DP, given \(F^{\mathrm{pos}} \sim DP(a+n, H^{\ast})\), the likelihood for each set of observed probabilities \(\left(F^{\mathrm{pos}}\left({A}_{i,1}\right) = {p}_{i,1}, \ldots, F^{\mathrm{pos}}\left({A}_{i,K}\right) = {p}_{i,K}\right)\) is given by:
\begin{equation}
    f\left({p}_{i,1}, \ldots,  {p}_{i,K};a\right)=  \frac{\prod_{k=1}^{K} \Gamma\left((a+n)H^{\ast}\left(A_{i,k}\right)\right)}{\Gamma\left(\sum_{k=1}^{K} (a+n)H^{\ast}\left(A_{i,k}\right)\right)}   \prod_{k=1}^{K} p_{i,k}^{(a+n)H^{\ast}(A_{i,k})-1}. 
\end{equation}}
\textcolor{black}{Given the observation matrix $\boldsymbol{p} = \left(p_{i,j} \right)_{n_p \times K}$, 
the log-likelihood for \(n_p\) independent sets of observed probabilities \(\boldsymbol{p}\) is given by:
\begin{align}\label{log-like}
    LL(a) &=\ln \prod_{i=1}^{n_p}f\left({p}_{i,1}, \ldots,  {p}_{i,K};a\right) \nonumber\\
    &=\ln \prod_{i=1}^{n_p}\left(  \frac{\prod_{k=1}^{K} \Gamma\left((a+n)H^{\ast}\left(A_{i,k}\right)\right)}{\Gamma\left(\sum_{k=1}^{K} (a+n)H^{\ast}\left(A_{i,k}\right)\right)}   \prod_{k=1}^{K} p_{i,k}^{(a+n)H^{\ast}(A_{i,k})-1} \right)\hspace{1cm}  \nonumber\\
    &=\sum_{i=1}^{n_p} \left(  \sum_{k=1}^{K}\ln \Gamma \left( (a+n)H^{\ast}\left(A_{i,k}\right) \right) - \ln \Gamma \left(\sum_{k=1}^{K} (a+n)H^{\ast}\left(A_{i,k}\right)\right) \right) \nonumber\\
    &\hspace{7cm}+ \sum_{i=1}^{n_p} \sum_{k=1}^{K} \left((a+n)H^{\ast}\left(A_{i,k}\right) - 1\right) \ln \left(p_{i,k}\right),
\end{align}}

\textcolor{black}{To maximize the log-likelihood \eqref{log-like}, we calculate the derivative of this log-likelihood with respect to $a$ as
\begin{align}\label{der-log-like}
    \frac{\partial LL(a)}{\partial a}=\sum_{i=1}^{n_p} \left( \sum_{k=1}^{K}\psi \left(  (a+n)H^{\ast}\left(A_{i,k}\right) \right) -  \psi\left(\sum_{k=1}^{K}(a+n)H^{\ast}\left(A_{i,k}\right)\right) \right) \nonumber\\
    + \sum_{i=1}^{n_p} \sum_{k=1}^{K} H^{\ast}\left(A_{i,k}\right) \ln \left(p_{i,k}\right),
\end{align}
where $\psi(a)=\partial\Gamma(a)/\partial a$ is the digamma function. we set the derivative in \eqref{der-log-like} to zero to find the solution. Since there is no closed-form solution for $a$, we use optimize using L-BFGS-B \citep{zhu1997algorithm}, which supports bound constraints and allows us to specify the range of $a$, thereby allowing us to select the appropriate level of intensity for our prior knowledge.}

To limit the impact of the prior $H$ on the results, we consider an upper bound less than $n/2$ for values of $a$ during the optimization process, as mentioned in \cite{fazeli2023semi}. By adhering to this upper bound, we can ensure that the chance of drawing samples from the observed data is at least twice as likely as that of generating samples from $H$. 
\textcolor{black}{Furthermore, the $K$-means algorithm is applied $n_P$ times to define the matrix of partitions $\boldsymbol{A}=(A_{i,j})_{n_P \times K}$ within the data. Each cluster corresponds to a partition, with sample points assigned to the nearest cluster centroid. The steps outlined above are summarized in Algorithm \ref{alg:bayesian_optimization}.} 

\begin{algorithm}
\caption{Optimization Process to Approximate Concentration Parameter}
\label{alg:bayesian_optimization}
\KwIn{Dataset $\mathbf{x}_{1:n}$; $H=N(\bar{\mathbf{x}}, S_{\mathbf{x}})$; partition matrix $\boldsymbol{A}=(A_{i,j})_{n_P \times K}$;  \( a_{u}=\frac{n}{2} \);  $\epsilon$.}
Initialize $a\gets a_0$\;
\While{$a$ has not converged}{
    $F^{\mathrm{pos}}\gets$ Fit a DP posterior $\mathrm{DP}(a+n,H^{\ast})$ to the given dataset $\mathbf{x}$\;
    $N \gets$ Adjust the number of active terms in the DP approximation using the random stopping rule defined in \eqref{random-stopping}, retaining only those terms with non-negligible weights\;
    $F^{\mathrm{pos}}_{N}\gets$  Approximate $F^{\mathrm{pos}}$ using Eq~\eqref{approx of DP}\;
    $p_{i,j}\gets F^{\mathrm{pos}}_{N}(A_{i,j})$\;
    Compute gradient, $\nabla_{a} LL(a)$, from Eq.~\eqref{der-log-like}, and Hessian of Eq.~\eqref{log-like} w.r.t. $a$, $\mathbb{H}^{-1}$\;
    $a\gets a-\mathbb{H}^{-1}\nabla_{a} LL(a)$\;
    \If{$a>a_u$}{break}
}
\KwOut{$a$}
\end{algorithm}


\subsection{Network Architecture}
The architecture of our networks is inspired by the network structure proposed by \cite{kwon2019generation}. 
Specifically, we uses the same layers shown in Figure \ref{diagram-3d-Gan} to construct a five 2D convolutional layers network for each of $Enc_{\boldsymbol{\eta}} $, $Gen_{\boldsymbol{\omega}} $, and $Dis_{\boldsymbol{\theta}} $, as the main task of this paper is to generate samples in 2-dimensional space. We also follow \cite{kwon2019generation} in setting the dimension of the latent space to be $p = 1000$. In the code space, we use a more sophisticated network architecture for generator $CGen_{\boldsymbol{\omega}^{\prime}}$, which includes three 2D convolutional layers and a fully connected layer, as opposed to the simple network architecture depicted in Figure \ref{diagram-rb-BNP1}. 

To begin, we set the sub-latent input vector size to $q=100$ in the first layer. Each convolutional layer is accompanied by a batch normalization layer, a ReLU activation function, and a max pooling (MaxPool) layer. The MaxPool layer reduces the spatial size of the feature maps and allows for the extraction of the most significant features of the code samples. Furthermore, it imparts the translation invariance property to the code samples, making $Gen_{\boldsymbol{\omega}} $ more robust to variations in the code space. We then use a fully connected layer to transform the feature maps into a code of size $p=1000$, followed by the addition of a Hyperbolic tangent activation function to the final layer, which squashes the outputs between -1 and 1. To ensure fair comparisons in the code space, we rescale $\mathbf{Z}_{r_{1:N}}$, $\mathbf{Z}_{e_{1:N}}$, and $\widetilde{\mathbf{Z}}_{e_{1:N}}$ to be between -1 and 1.
The overall architecture of our model is depicted in Figure \ref{diagram-2d-Gan-BNP}.
\begin{algorithm}
\KwIn{Training data $\mathbf{x}_{1:n}$; $n_{\mathrm{iter}}$--iterations; $n_{mb}$--mini-batch size; learning rates: $\alpha_{CGen}=\alpha_{Dis}=\alpha_{Gen}=2\mathrm{e}^{-4}$; penalty $\lambda_{p}=10$; threshold $\epsilon=1\mathrm{e}^{-4}$ in \eqref{random-stopping}; Adam optimizer parameters: $(\varepsilon=1\mathrm{e}^{-8}, \beta_1=0.9, \beta_2=0.999)$.}
\KwOut{Optimized parameter of the generator}


\For{$t =1$ to \( n_{\mathrm{iter}} \)}{
            sample $ \mathbf{x}_{1:n_{mb}}\overset{i.i.d.}{\sim} F$\;
            $a\gets$ implement Algorithm \ref{alg:bayesian_optimization} for batch sample  $\mathbf{x}_{1:n_{mb}}$\;
            $N\gets$ apply $a$ in \eqref{approx of DP}\;
            \Proc{\textnormal{BNPL}}{
            \Fn{}{sample $(J^{\mathrm{pos}}_{1,N},\ldots, J^{\mathrm{pos}}_{N,N})\sim \mbox{Dir}((a+n_{mb})/N,\ldots,(a+n_{mb})/N)$\;
            sample $\mathbf{x}^{\mathrm{pos}}_{1:N}\overset{i.i.d.}{\sim}H^{\ast}$\;
            $F^{\mathrm{pos}}_{N}=\sum_{i=1}^{N}J^{\mathrm{pos}}_{i,N}\delta_{\mathbf{X}^{\mathrm{pos}}_i}$\;}
            \Distrain{}{
            sample noises $ \mathbf{z}_{r_{1:N}}\overset{i.i.d.}{\sim}F_{\mathbf{Z}_r}$ and $ \mathbf{z}^{\prime}_{r_{1:N}}\overset{i.i.d.}{\sim}F_{\mathbf{Z}^{\prime}_{r}}$\;
            $\widetilde{\mathbf{z}}_{e_{1:N}}= CGen_{\boldsymbol{\omega}^{\prime}}\left(  \mathbf{z}^{\prime}_{r_{1:N}}\right)$,
            $\mathbf{z}_{e_{1:N}} = Enc_{\boldsymbol{\eta}}\left(  \mathbf{x}^{\mathrm{pos}}_{1:N}\right)$\;

            $ ( \widetilde{\mathbf{x}}^{\mathrm{pos}}_{1:N},\widetilde{\mathbf{x}}^{\mathrm{pos}}_{N+1:2N},\widetilde{\mathbf{x}}^{\mathrm{pos}}_{2N+1:3N})=\left(Gen_{\boldsymbol{\omega}}\left(\mathbf{z}_{r_{1:N}}\right), Gen_{\boldsymbol{\omega}}\left(\mathbf{z}_{e_{1:N}}\right), Gen_{\boldsymbol{\omega}}\left(\widetilde{\mathbf{z}}_{e_{1:N}}\right)\right) $\;
            \gp{}{
             $u\sim U(0,1)$, $L_{\text{GP-Dis}}\gets0$\;
            \For{$\ell =1$ to \( 3 \)}{
            $\widehat{\mathbf{x}}^{\mathrm{pos}}_{\ell,1:N}=(1-u) \mathbf{x}^{\mathrm{pos}}_{1:N}+u\widetilde{\mathbf{x}}^{\mathrm{pos}}_{((\ell-1)N+1):\ell N}$\;
            $L_{\text{GP-Dis}}\gets N^{-1}\sum_{i=1}^{N}( \| \nabla_{\widehat{\mathbf{x}}^{\mathrm{pos}}_{\ell,i}}Dis_{\boldsymbol{\theta}}(\widehat{\mathbf{x}}^{\mathrm{pos}}_{\ell,i}) \|_2-1)^2+L_{\text{GP-Dis}}$\;
            }}
            $\mathcal{L}_{\text{GAN}}\gets\text{WMMD}(F^{\mathrm{pos}}_N,F_{\widetilde{\mathbf{x}}^{\mathrm{pos}}_{1:N}})$\;
            $\mathcal{L}_{\text{Rec}}\gets\text{WMMD}(F^{\mathrm{pos}}_N,F_{\widetilde{\mathbf{x}}^{\mathrm{pos}}_{N+1:2N}})$\; $\mathcal{L}^{\prime}_{\text{Rec}}\gets\text{WMMD}(F^{\mathrm{pos}}_N,F_{\widetilde{\mathbf{x}}^{\mathrm{pos}}_{2N+1:3N}})$\;
            
            $\mathcal{L}\gets -(\mathcal{L}_{\text{Rec}}+\mathcal{L}^{\prime}_{\text{Rec}}+\mathcal{L}_{\text{GAN}})+ \lambda_{p} L_{\text{GP-Dis}}$\; 
            $\nabla_{\boldsymbol{\theta}}\mathcal{L}_{\text{Dis}}(\boldsymbol{\theta})\gets\nabla_{\boldsymbol{\theta}}\mathcal{L}$\;
            $\boldsymbol{\theta}\gets \text{Adam}(\nabla_{\boldsymbol{\theta}}\mathcal{L}_{\text{Dis}}(\boldsymbol{\theta}), \boldsymbol{\theta}, \varepsilon, \beta_1, \beta_2)$\;
            }
            \CGentrain{}{
            repeat lines (11)–(12)\;
            $\mathcal{L}_{\text{CGen}}(\boldsymbol{\omega}^{\prime})\gets$  $\text{MMD}(F_{\mathbf{z}_{e_{1:N}}},F_{\widetilde{\mathbf{z}}_{e_{1:N}}})$\; $\boldsymbol{\omega}^{\prime}\gets \text{Adam}(\nabla_{\boldsymbol{\omega}^{\prime}}\mathcal{L}_{\text{CGen}}(\boldsymbol{\omega}^{\prime}), \boldsymbol{\omega}^{\prime}, \varepsilon, \beta_1, \beta_2)$;
            }
            \Gentrain{}{
            

            repeat lines (11)–(13) and then lines (19)–(21) sequentially\; 
            $\mathcal{L}_{\text{Reg}}\gets\text{MMD}(F_{\mathbf{Z}_{r_{1:N}}},F_{\mathbf{Z}_{e_{1:N}}})$\; 
            $\mathcal{L}\gets \mathcal{L}_{\text{Reg}}+\mathcal{L}_{\text{Rec}}+\mathcal{L}^{\prime}_{\text{Rec}}+\mathcal{L}_{\text{Gan}}$\;
            
            $\nabla_{(\boldsymbol{\omega},\boldsymbol{\eta})}\mathcal{L}_{\text{EGen}}(\boldsymbol{\omega},\boldsymbol{\eta})\gets\nabla_{(\boldsymbol{\omega},\boldsymbol{\eta})}\mathcal{L}$\;
            $(\boldsymbol{\omega},\boldsymbol{\eta})\gets \text{Adam}(\nabla_{(\boldsymbol{\omega},\boldsymbol{\eta})}\mathcal{L}_{\text{EGen}}(\boldsymbol{\omega},\boldsymbol{\eta}), (\boldsymbol{\omega},\boldsymbol{\eta}), \varepsilon, \beta_1, \beta_2)$\;
            }
            }
            
}
\caption{Training Triple-Generative Model via BNPL}\label{alg2}
\end{algorithm}

\begin{figure}
    \centering
    {\includegraphics[width=1\linewidth]{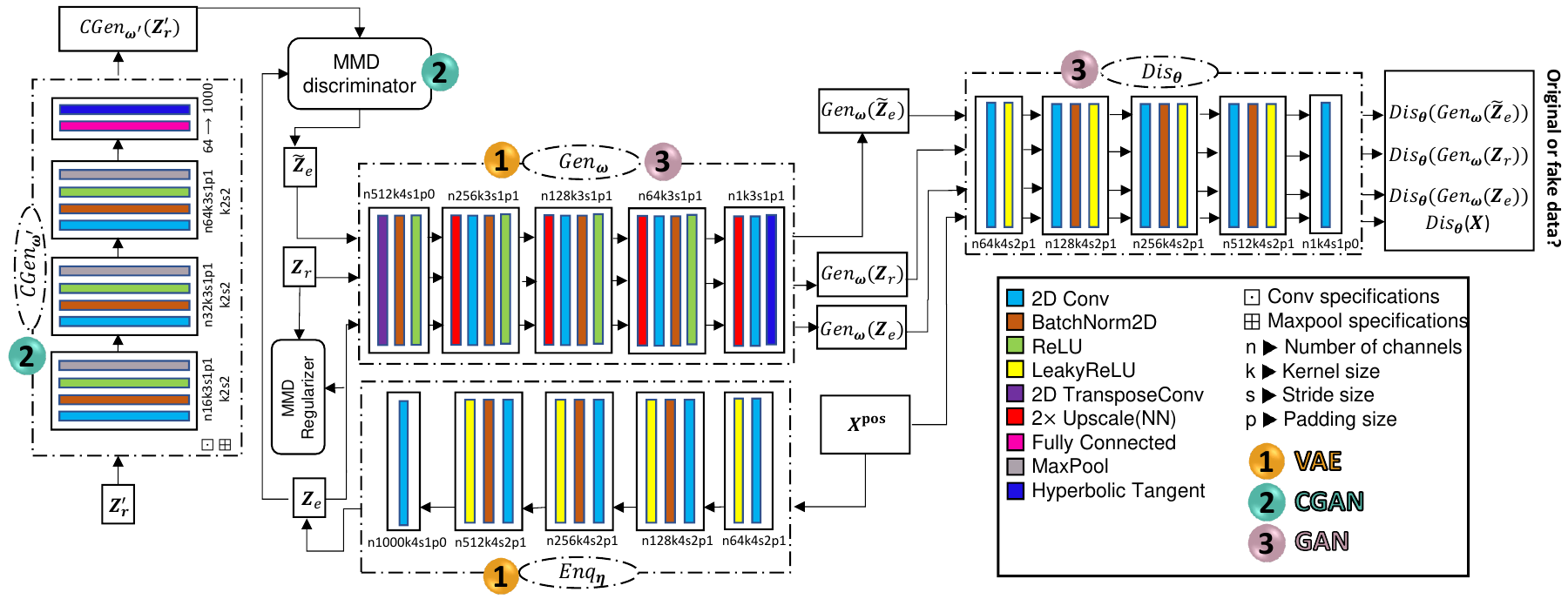}}
    \caption{The general architecture of the proposed triple model consists of four convolutional networks for the encoder, generator, code-generator, and discriminator. \textcolor{black}{The partial models—VAE, CGAN, and GAN—are labeled as 1, 2, and 3, respectively}.}
    \label{diagram-2d-Gan-BNP}%
\end{figure}
\section{Experimental Results}\label{sec:experiments}
In this section, we present the results of our experiments on four datasets to evaluate the performance of our models. 
We implement our models using the PyTorch library in Python. We set the mini-batch size to 16 and the number of workers to 64. As the Hyperbolic tangent activation function is used on the last layer of the generator, we scaled all datasets to the range of $-1$ to $1$ to ensure compatibility with the generator outputs. For comparison purposes, we evaluated our model against Semi-BNP MMD GAN \citep{fazeli2023semi} and AE+GMMN \citep{Li}\footnote{The relevant codes can be found at \url{https://github.com/yujiali/gmmn.git}}. To provide a comprehensive comparison, we also attempted to modify the 3D $\alpha$-WGPGAN \citep{kwon2019generation}\footnote{The basic code for 3D generation are availble at  \url{https://github.com/cyclomon/3dbraingen}} settings to the 2D dimension case and included its relevant results. \textcolor{black}{Table \ref{comparision-models} presents the shared and distinct attributes of these models in relation to our proposed model.}

\begin{table}[H]
\caption{Comparative analysis of generative models: Highlighting shared and distinct attributes}\label{comparision-models}
\centering
\begin{tabular}{|c|c|c|c|c|c|c|}
\hline
\textbf{Model} & \textbf{Wasserstein} & \textbf{MMD} & \textbf{CGAN} & \textbf{VAE} & \textbf{GAN} & \textbf{AE} \\ 
\hline
Ours  & \color{green}{\cmark} & \color{green}{\cmark} & \color{green}{\cmark} & \color{green}{\cmark} & \color{green}{\cmark} & \color{red}{\xmark} \\ 
\hline
Semi-BNP MMD  & \color{red}{\xmark} & \color{green}{\cmark} & \color{red}{\xmark} & \color{red}{\xmark} & \color{green}{\cmark} & \color{red}{\xmark} \\
\citep{fazeli2023semi}&&&&&&\\
\hline
AE+GMMN   & \color{red}{\xmark} & \color{green}{\cmark} & \color{green}{\cmark} & \color{red}{\xmark} & \color{red}{\xmark} & \color{green}{\cmark} \\
\citep{Li}&&&&&&\\
\hline
$\alpha$-WGPGAN  & \color{green}{\cmark} & \color{red}{\xmark} & \color{red}{\xmark}& \color{green}{\cmark} & \color{green}{\cmark} & \color{red}{\xmark} \\
\citep{kwon2019generation}&&&&&&\\
\hline
\end{tabular}
\end{table}

\subsection{Labeled Datasets}
To evaluate model performance, we analyzed two handwritten datasets comprising of numbers (MNIST) and letters (EMNIST). MNIST consists of 60,000 handwritten digits, including 10 numbers from 0 to 9 (labels), each with 784 ($28\times28$) dimensions. This dataset was divided into 50,000 training and 10,000 testing images, and we use the training set to train the network \citep{lecun1998mnist}.  EMNIST is freely available online\footnote{\urlx{https://www.kaggle.com/datasets/sachinpatel21/az-handwritten-alphabets-in-csv-format}} and is sourced from \cite{cohen2017emnist}. It contains 372,450 samples of the 26 letters A-Z (labels). Each letter is represented in a $28\times28$ dimension. We allocate 85\% of the samples to the training dataset, and the rest to the testing dataset.

\subsubsection{Evaluating Mode collapse} To examine the capability of the model in covering all modes or preventing the mode collapse, we train a convolutional network to 
predict the label of each generated sample. The structure of this network is provided in Figure \ref{classifierr}.
If the generator has effectively tackled the issue of the mode collapse, we anticipate having a similar relative frequency for labels in both training and generated datasets, indicating successful training. Plots (a) and (b) in Figure \ref{freq-target} represent the relative frequency of labels in handwritten numbers and letters datasets, respectively.

To train the classifier, we use the cross-entropy loss function and update the network's weights with the Adam optimizer over 60 epochs. 
We assess the classifier's efficacy by presenting the mean of the loss function across all mini-batch testing samples and the percentage of correct classification (accuracy rate) in Figure \ref{tbl-classifier}. The figure showcases the classifier's exceptional accuracy in classifying the dataset.
\begin{figure*}[t!]
    \centering
    {\includegraphics[width=.9\linewidth]{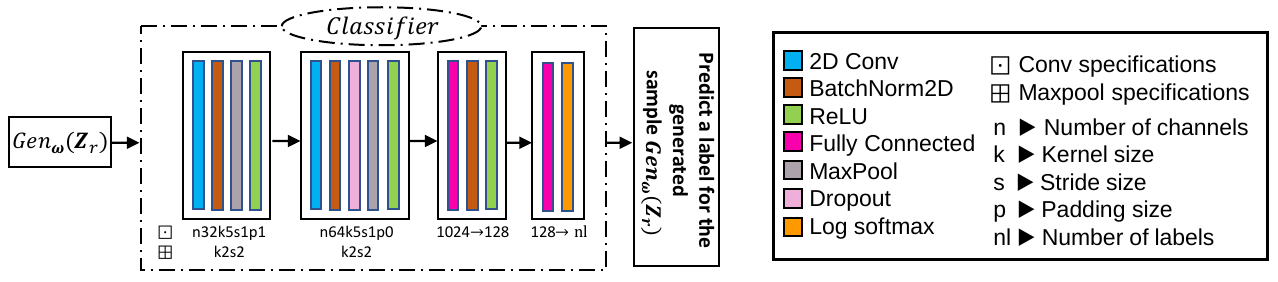}}
    \caption{The classifier network architecture for predicting the handwritten dataset's labels. The output of the fully connected layer is passed through a log softmax function to convert the raw output into a probability distribution over the classes.}
    \label{classifierr}%
\end{figure*}


\begin{figure}[ht]
    \centering
       {\includegraphics[width=.7\linewidth]{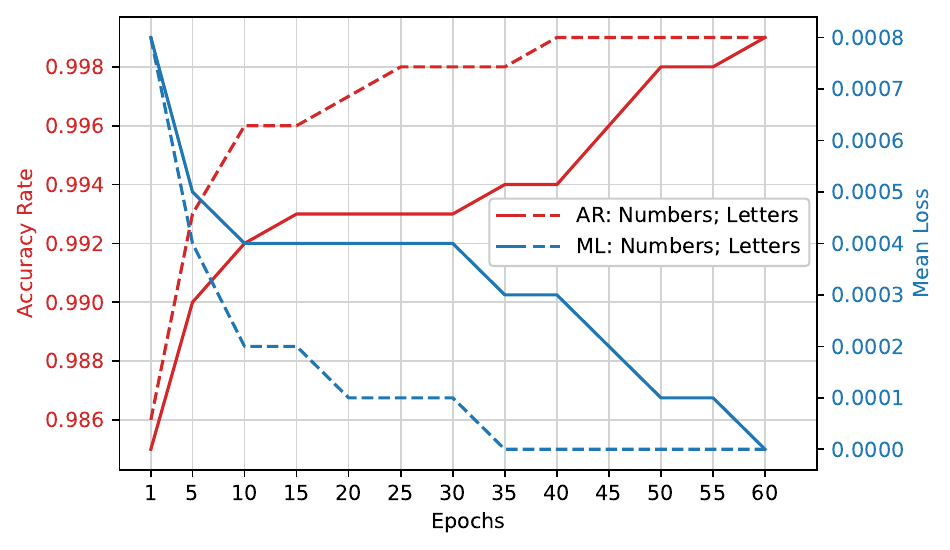}}
    \caption{The convolutional classifier's mean loss (ML) and accuracy rate (AR) with a learning rate of $0.0002$ across all mini-batch testing samples of numbers (solid line) and letters (dashed line) datasets.}\label{tbl-classifier}
\end{figure}
The relative frequency plots of predicted labels are depicted in Figure \ref{freq-target}, parts (c)-(j), for 1000 generated numbers (left-hand side) and letters (right-hand side) using various generative models.
The ratios of the numbers in the training dataset are expected to be consistent, as indicated by plot (a). Examining the plot of generated samples reveals that each model has a distinct bias towards certain digits. Specifically, our model tends to produce digit 6 at a frequency 4.64\% higher than that in the training dataset ($14.50\%-9.86\%$), while the semi-BNP MMD GAN exhibits a similar bias towards digit 3 at a 4.18\% higher frequency than the training dataset. 

Nevertheless, these differences are relatively minor compared to AE+GMMN and $\alpha$-WGPGAN, which demonstrate a significant tendency to memorize some modes and overlook certain digits, such as 4 and 8. Similar results can be observed from the relative frequency plots of predicted labels for the letters dataset. Plot (j) in Figure \ref{freq-target} clearly shows the failure of $\alpha$-WGPGAN to maintain the balance of the relative frequency of the data and generate the letter ``M".  In contrast, plot (d) indicates that our proposed model successfully preserves the proportion of modes in the generated samples and avoids mode collapse.

\subsubsection{Assessing Patterns and Evaluating Model Quality}
We employed the principal component analysis (PCA) to illustrate the latent structure among data points in a two-dimensional space. Parts (a) and (b) in Figure \ref{PCA} represent the PCA plots for numbers and letters, respectively. Each axis in the plots represents a principal component, with the relevant real dataset used as a reference. It is important to note that PCA provides a necessary condition to verify the similarity pattern between real and fake data distribution \citep{hotelling1933analysis}. The dissimilarity between PCA plots for real and generated samples indicates that they do not follow the same distribution. However, two similar PCA plots do not necessarily guarantee a similar distribution for real and generated samples. Here, the results presented in Figure \ref{PCA} demonstrate that all models follow the pattern of the relevant real datasets except for the $\alpha$-WGPGAN in the letters dataset. It indicates a different shape and orientation than the structure of the real dataset.

\begin{figure}[t!]
    \centering
        \subfloat[Numbers: Training dataset]{\includegraphics[width=.5\linewidth, height=.23\linewidth]{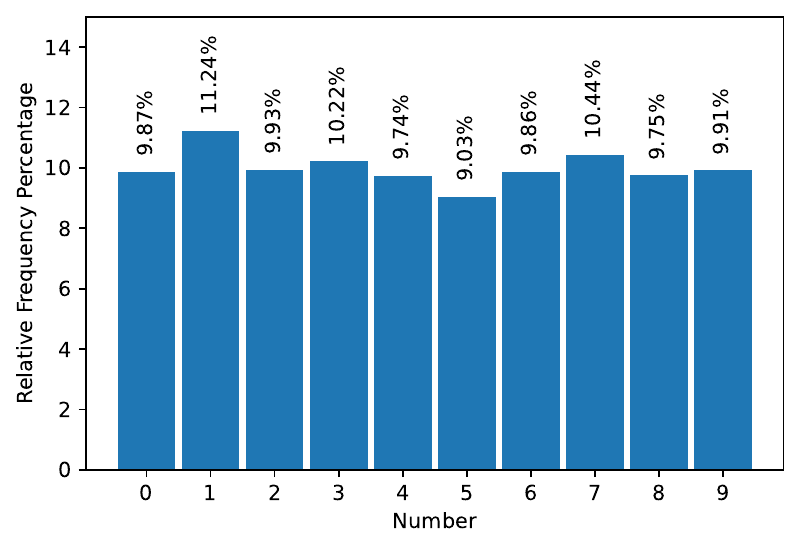}}
        \subfloat[Letters: Training dataset]{\includegraphics[width=.5\linewidth, height=.23\linewidth]{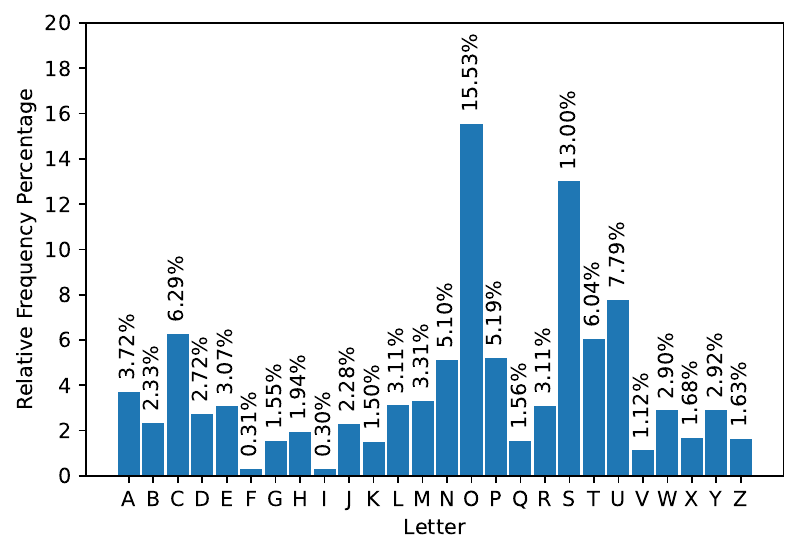}}\\
        \subfloat[Ours]
        {\includegraphics[width=.5\linewidth, height=.23\linewidth]{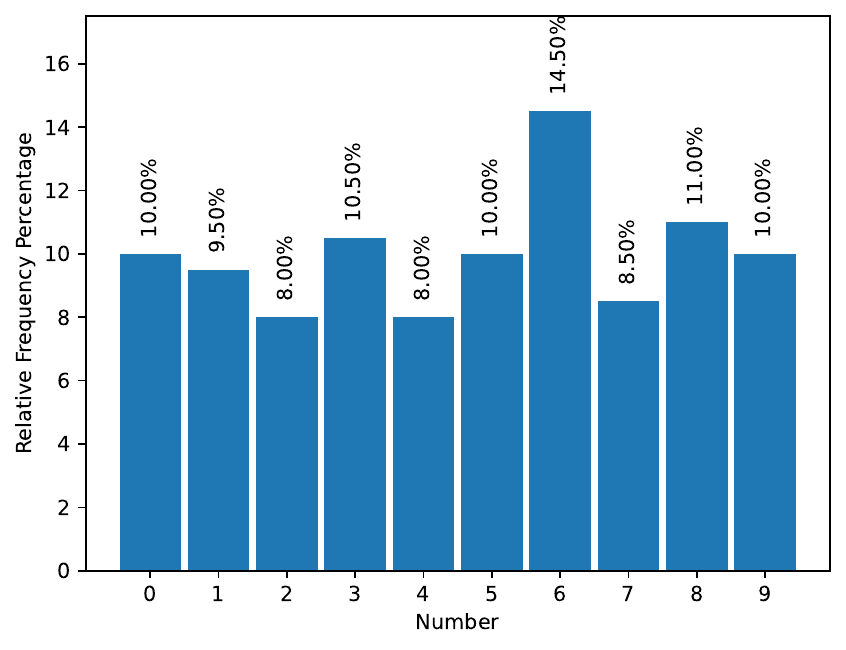}}
        \subfloat[Ours]{\includegraphics[width=.5\linewidth, height=.23\linewidth]{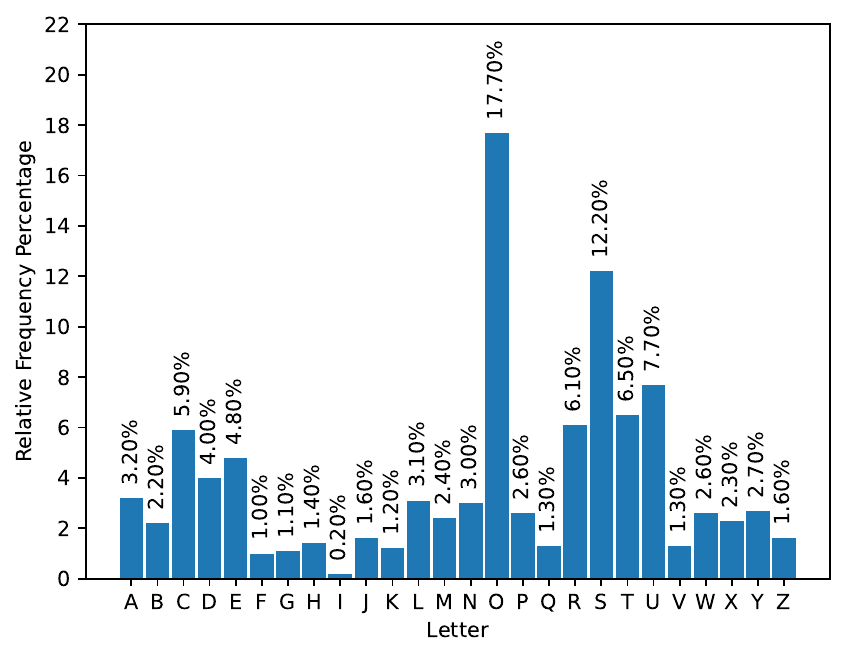}}\\
        \subfloat[Semi-BNP MMD]
        {\includegraphics[width=.5\linewidth, height=.23\linewidth]{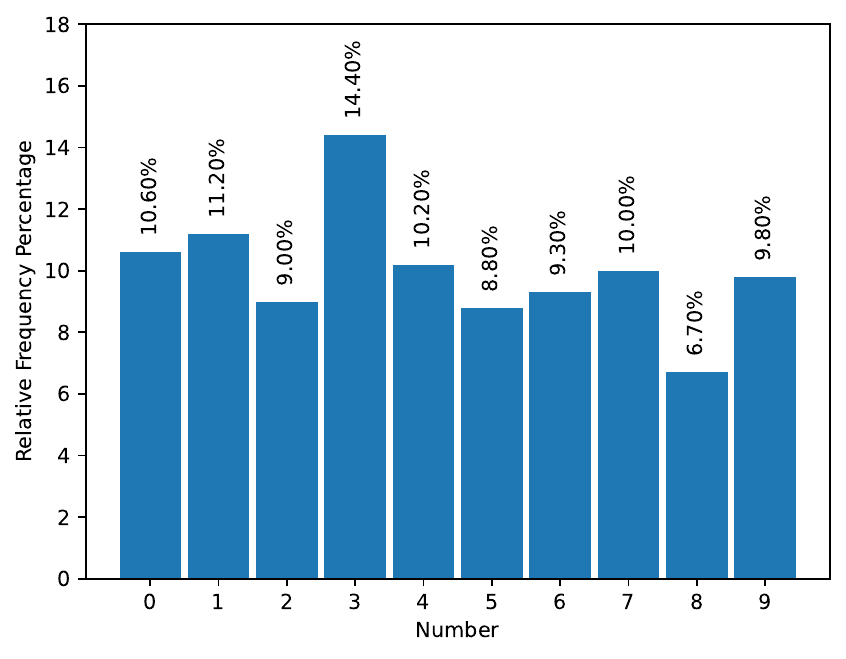}}
        \subfloat[Semi-BNP MMD]{\includegraphics[width=.5\linewidth, height=.23\linewidth]{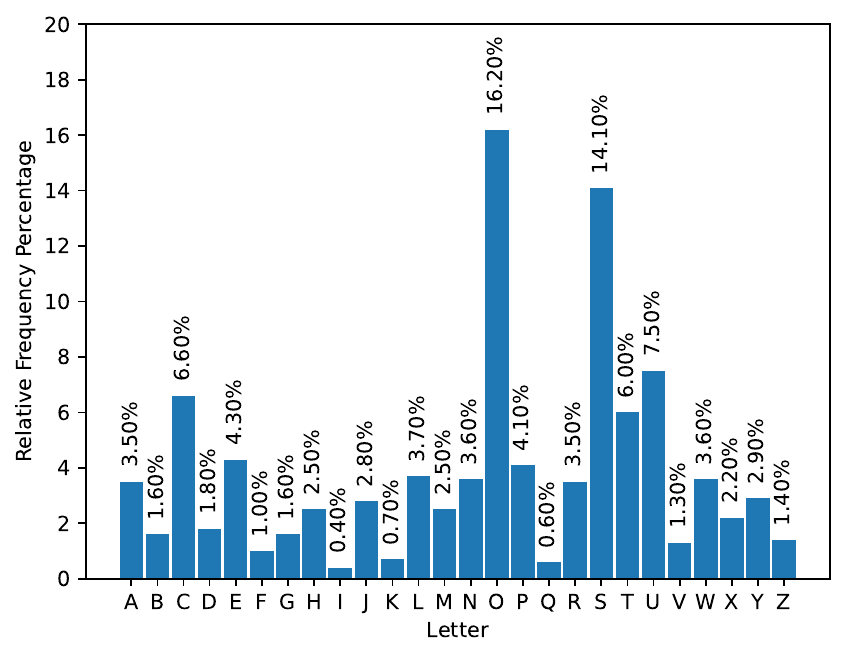}}
        \\
        \subfloat[AE+GMMN]{\includegraphics[width=.5\linewidth, height=.23\linewidth]{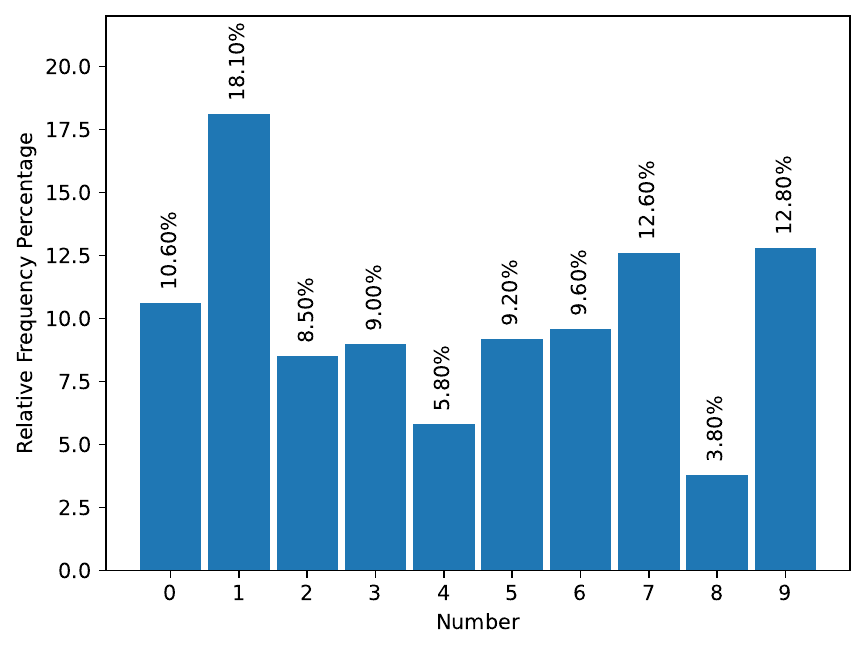}}
        \subfloat[AE+GMMN]{\includegraphics[width=.5\linewidth, height=.23\linewidth]{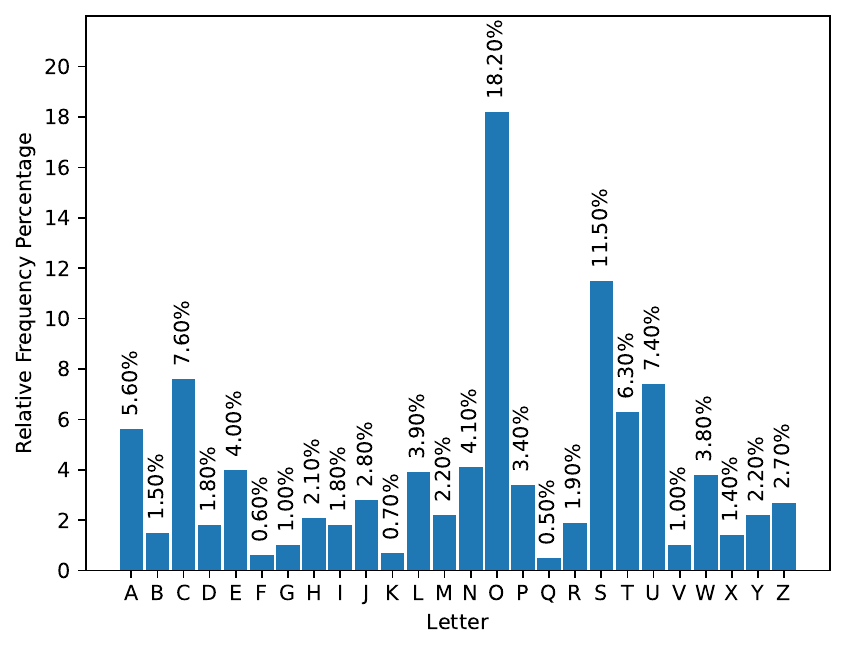}}
        \\
        \subfloat[$\alpha$-WGPGAN]
        {\includegraphics[width=.5\linewidth, height=.23\linewidth]{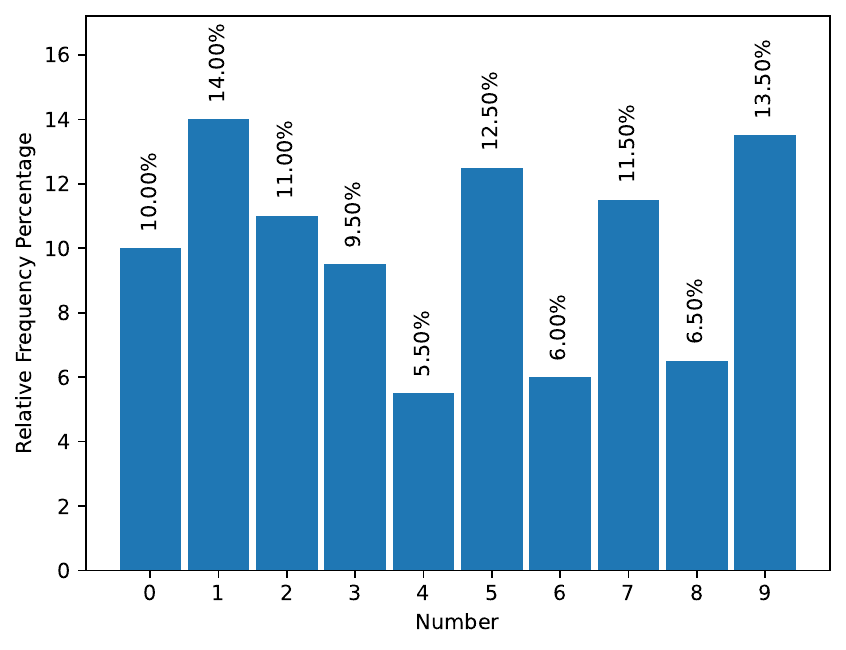}} 
        \subfloat[$\alpha$-WGPGAN]{\includegraphics[width=.5\linewidth, height=.23\linewidth]{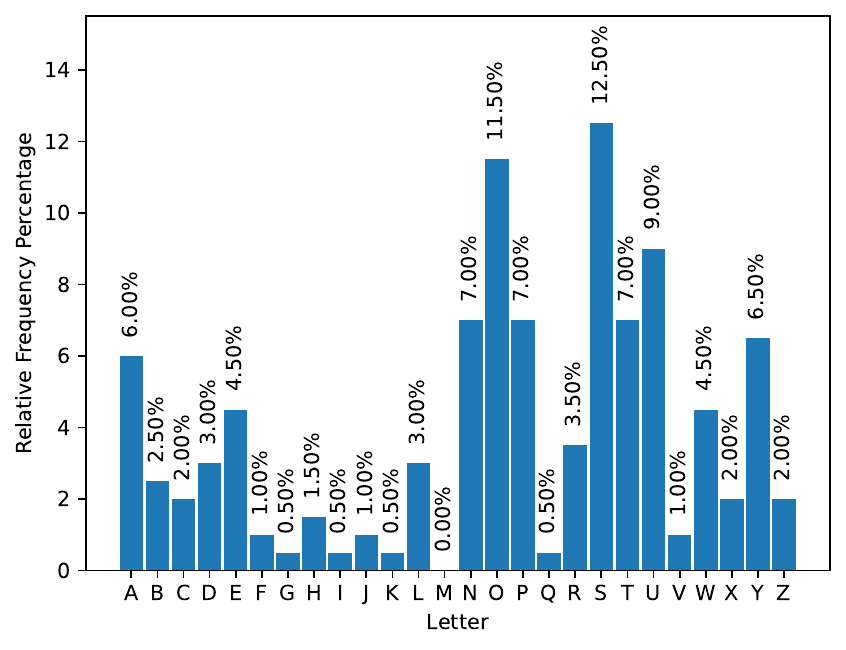}} 
    \caption{ The frequency percentage of true and predicted labels for handwritten datasets.}\label{freq-target}
\end{figure}
 
For a more comprehensive analysis, we adopt a mini-batch strategy suggested by \cite{fazeli2023semi} to compute the MMD score, as given by Equation \eqref{MMD-ecdf}, between the generated and real samples. We present the discrepancy scores in density and box plots, along with violin plots in Figure \ref{PCA}, parts (c) and (d). Overall, the MMD scores of all four models suggest some level of convergence around zero. However, the results of the proposed model and the semi-BNP model are comparable, with both models showing better convergence than the other two models. Specifically, part (d) shows our proposed model demonstrates even better convergence than the semi-BNP model, highlighting an improvement of the semi-BNP MMD model by extending it to the VAE+WMMD model.
\begin{figure}[t!]
    \centering
        \subfloat[Numbers: PCA plot]{\includegraphics[width=.5\linewidth, height=.35\linewidth]{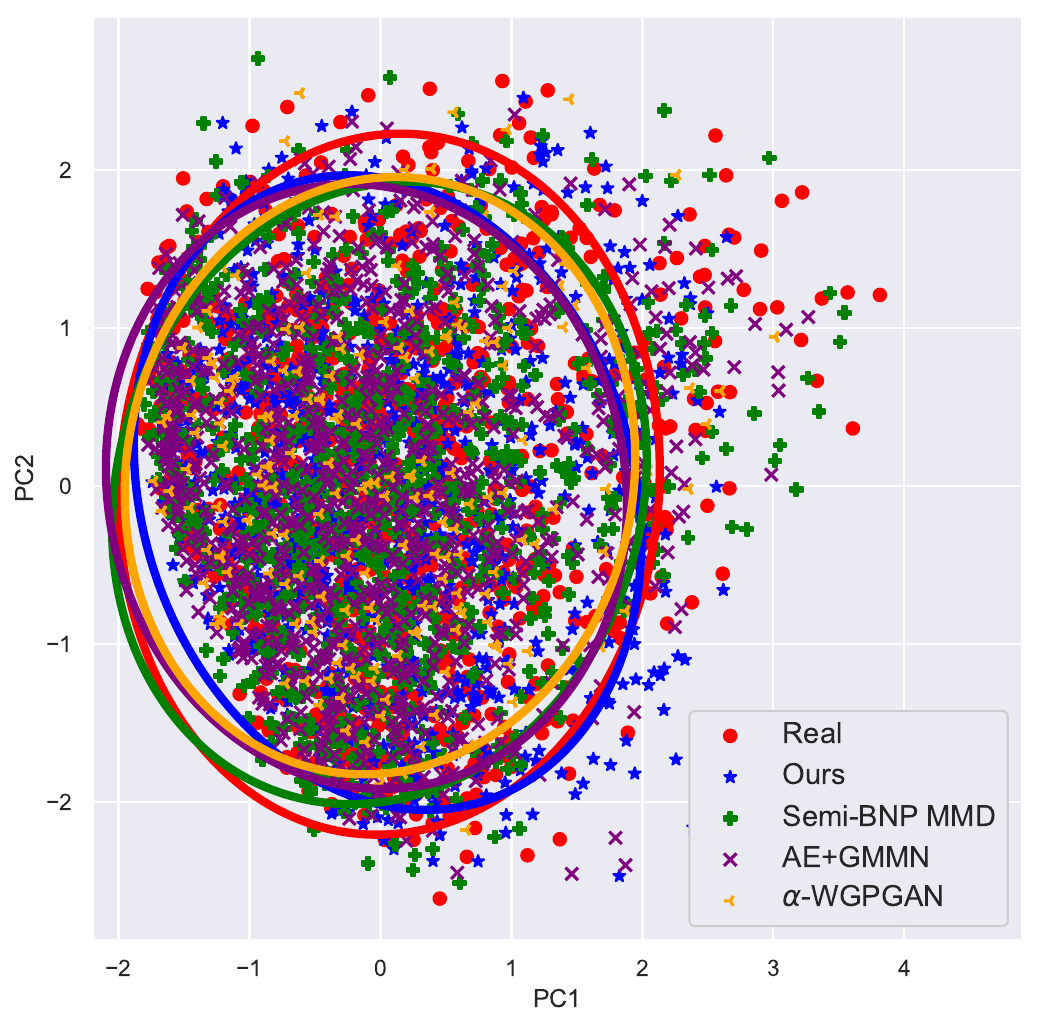}}
        \subfloat[Letters: PCA plot]{\includegraphics[width=.5\linewidth, height=.35\linewidth]{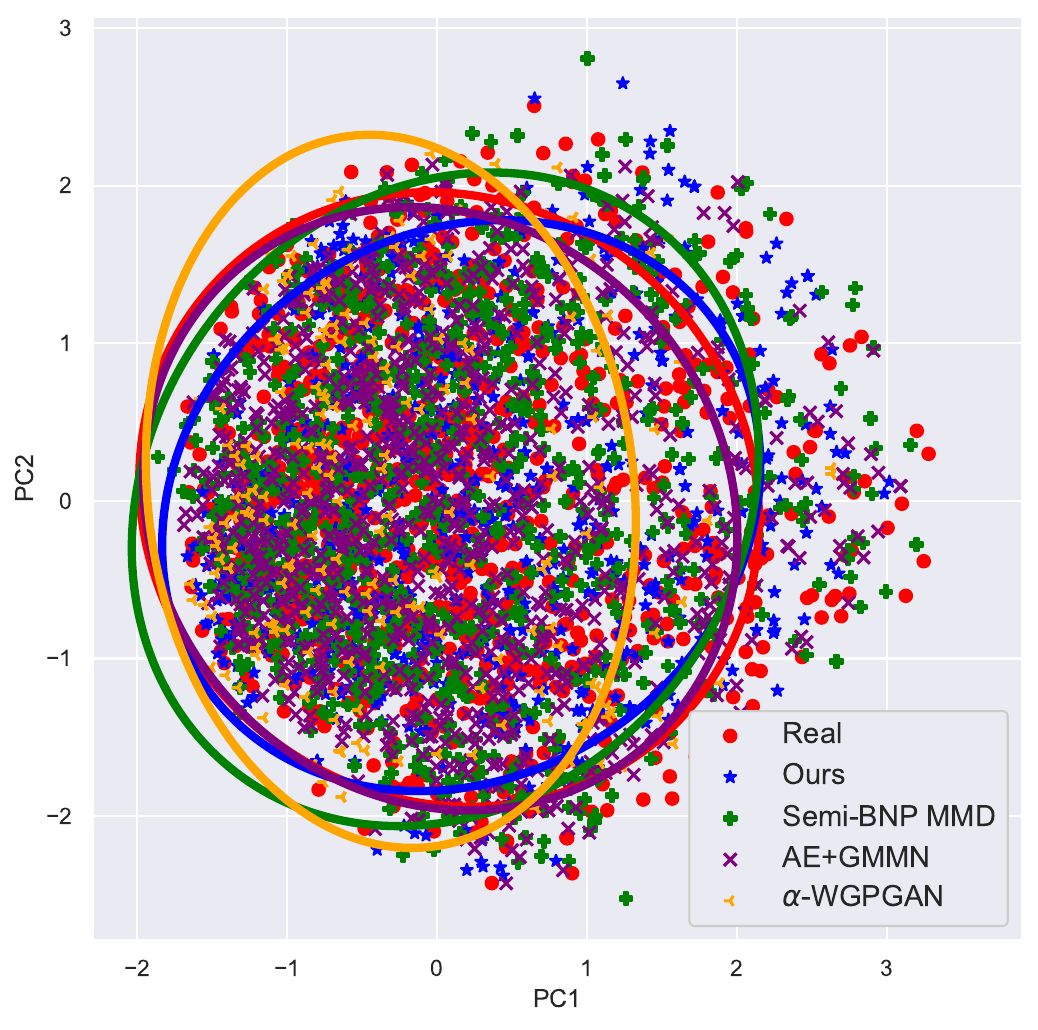}}\\
        \subfloat[Numbers: Violin plot]
        {\includegraphics[width=.5\linewidth, height=.4\linewidth]{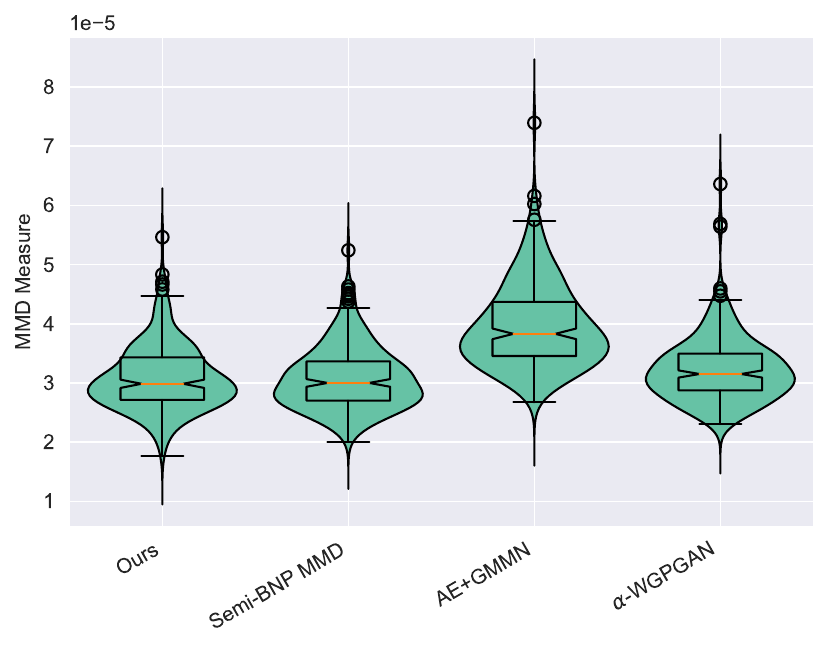}}
        \subfloat[Letters: Violin plot]{\includegraphics[width=.5\linewidth, height=.4\linewidth]{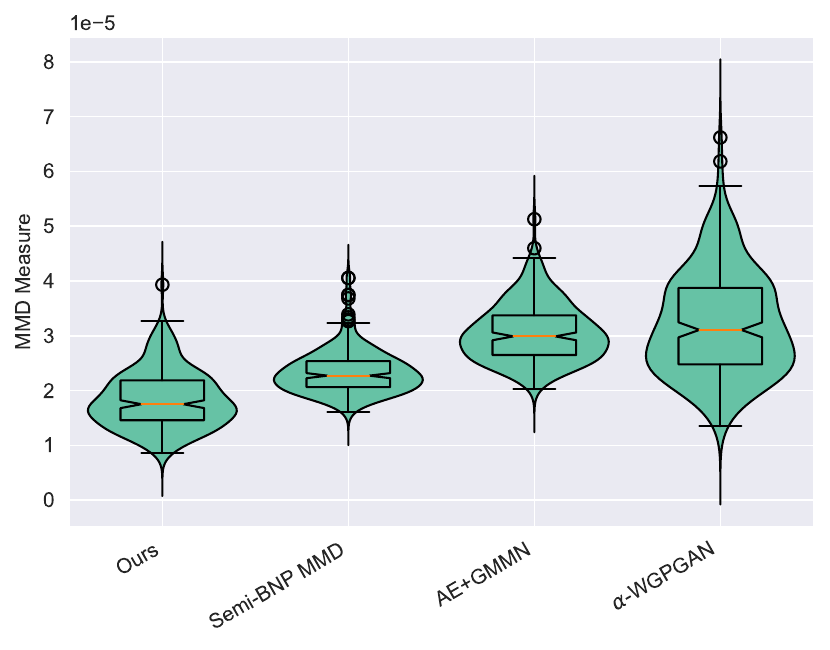}}
    \caption{Top: PCA plots of 1000 generated samples versus real samples with fitted ellipse curves for handwritten datasets, indicating the spread of the samples in the corresponding directions. Bottom: Violin plots of MMD scores including density and box plots. }\label{PCA}
\end{figure}

\subsubsection{Visualisation}
To better demonstrate the visual capabilities of our proposed model in generating samples, we have displayed 60 samples generated from the model and have compared them to the samples generated by other models, as depicted in Figure \ref{visual-MNIST}. While the semi-BNP MMD model displays a range of generated characters in parts (e) and (f) of Figure \ref{visual-MNIST}, the images contain some noise that detracts from their quality. On the other hand, the results of AE+GMMN, displayed in parts (g) and (h), reveal blurry outputs that fall short of our desired standards. 
\begin{figure}[t!]
    \centering
        \subfloat[Numbers: Training dataset]{\includegraphics[width=.45\linewidth, height=.22\linewidth]{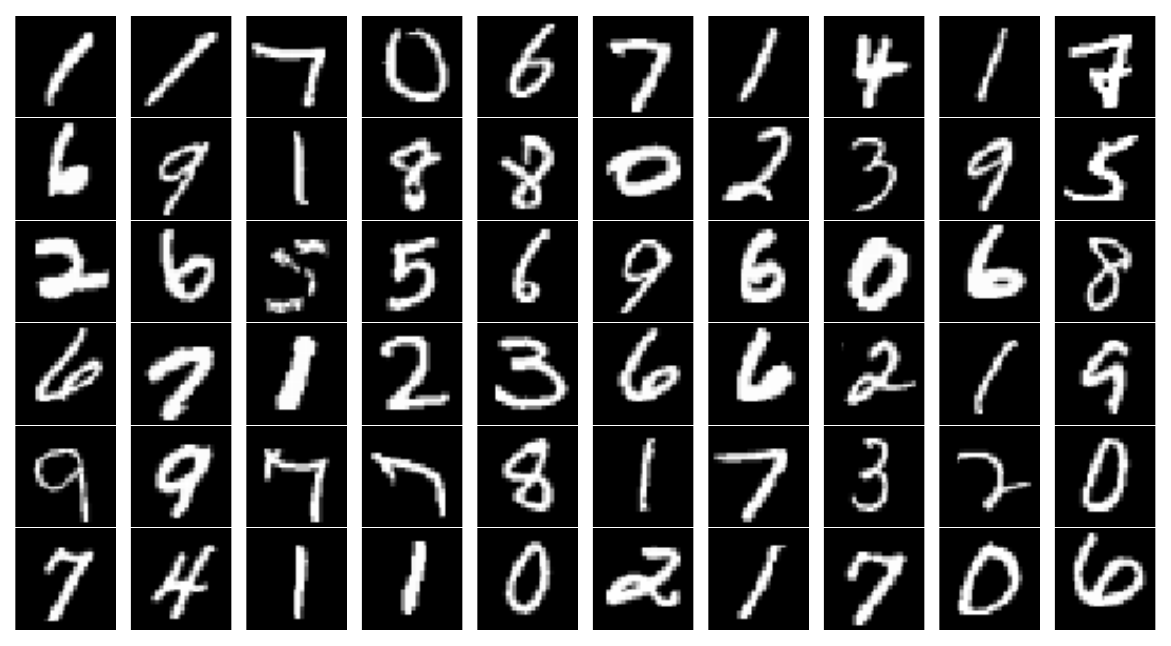}}
        \hspace{1cm}
        \subfloat[Letters: Training dataset]{\includegraphics[width=.45\linewidth, height=.22\linewidth]{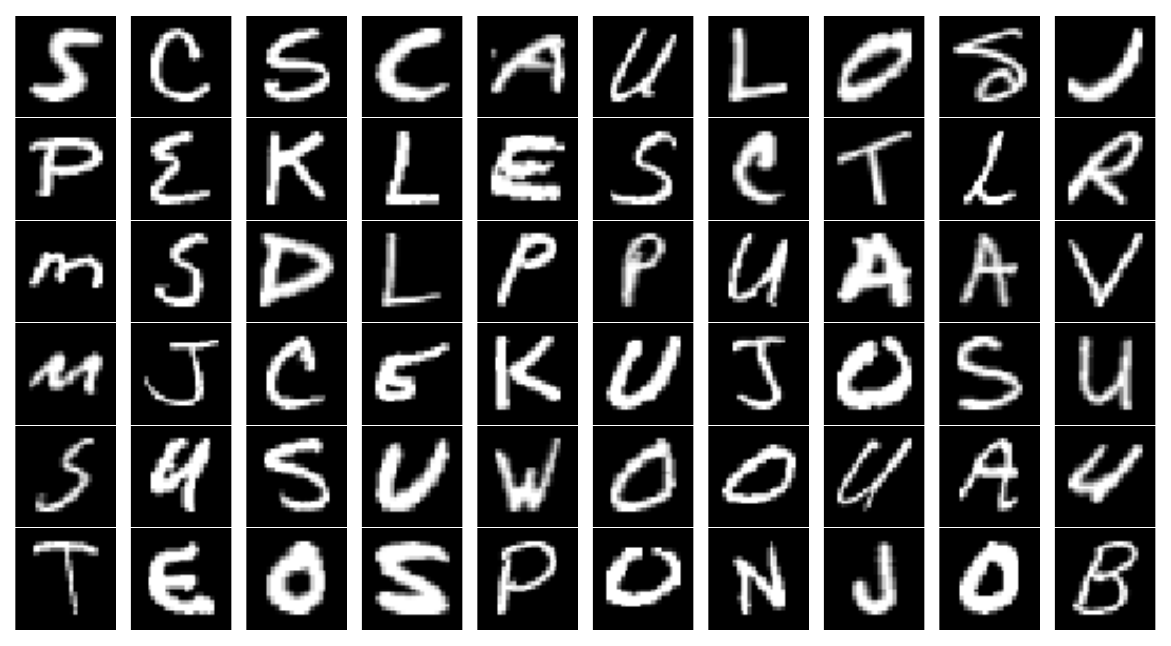}}\\
        \subfloat[Ours]{\includegraphics[width=.45\linewidth, height=.22\linewidth]{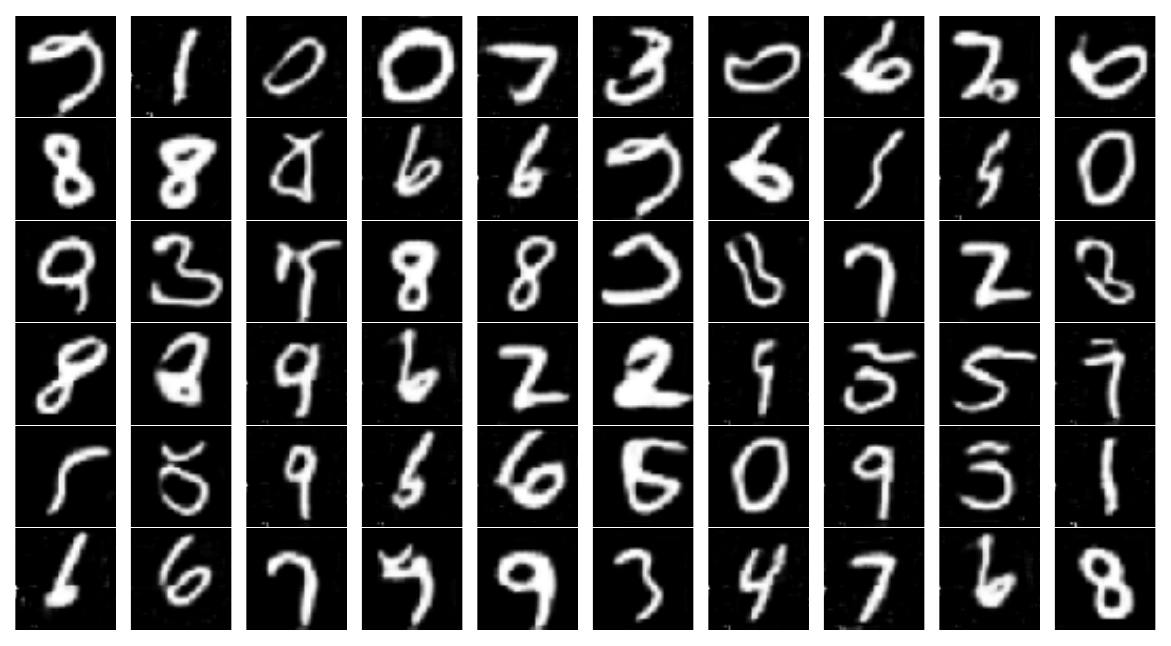}}
        \hspace{1cm}
        \subfloat[Ours]{\includegraphics[width=.45\linewidth, height=.22\linewidth]{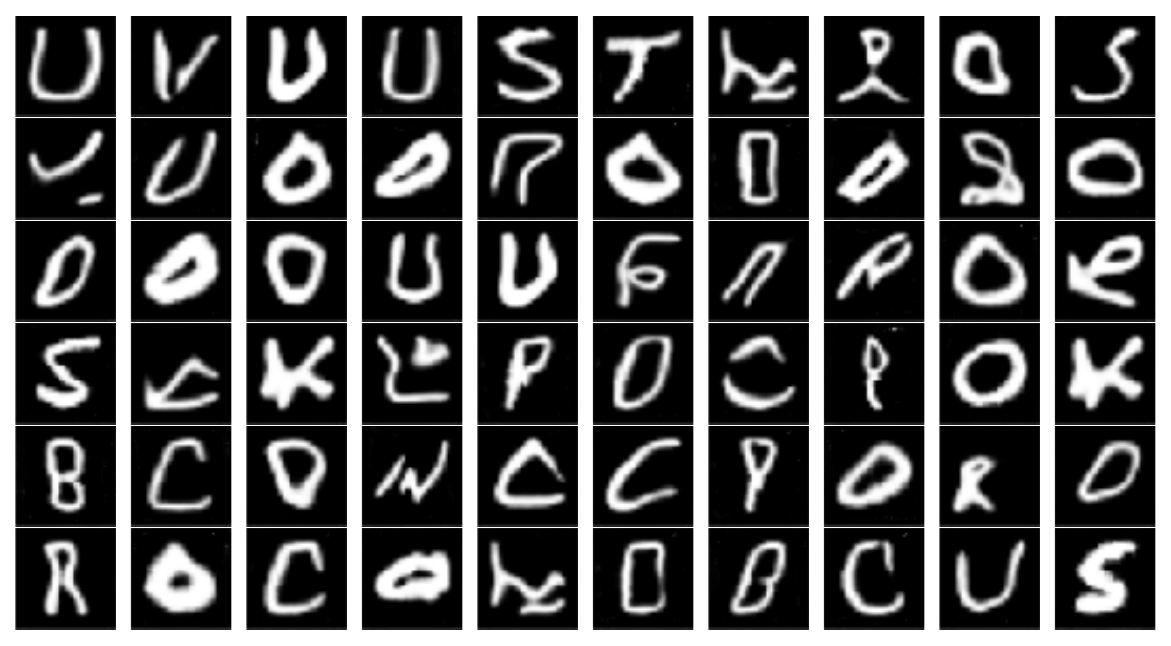}}\\
        \subfloat[Semi-BNP MMD]{\includegraphics[width=.45\linewidth, height=.22\linewidth]{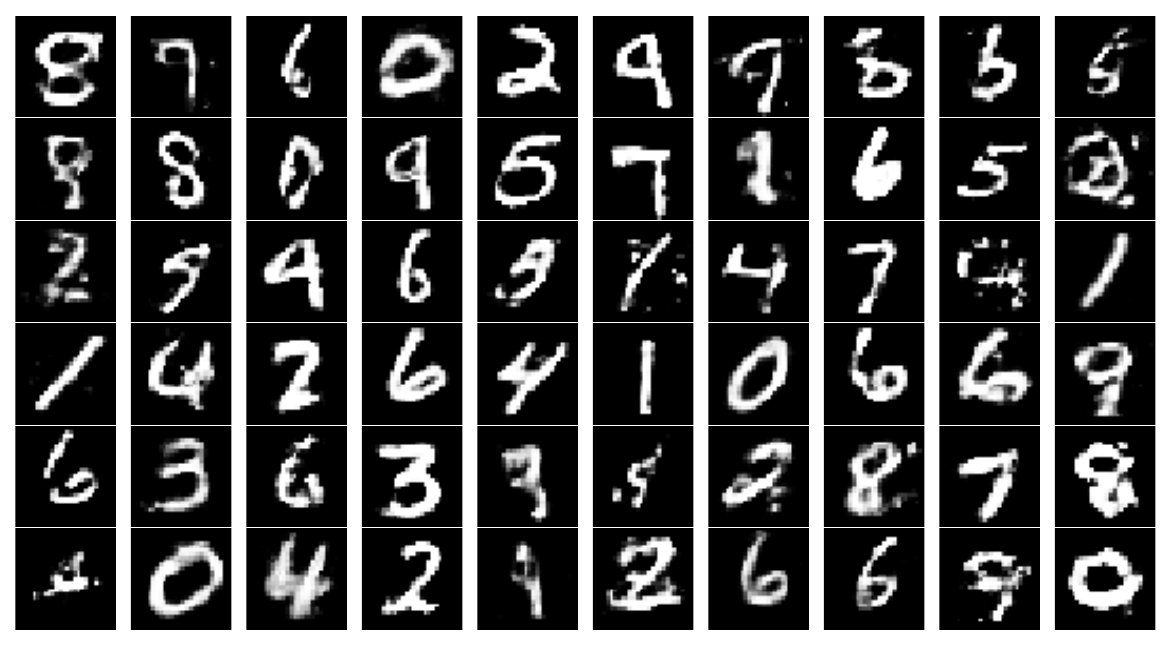}}
        \hspace{1cm}
        \subfloat[Semi-BNP MMD]{\includegraphics[width=.45\linewidth, height=.22\linewidth]{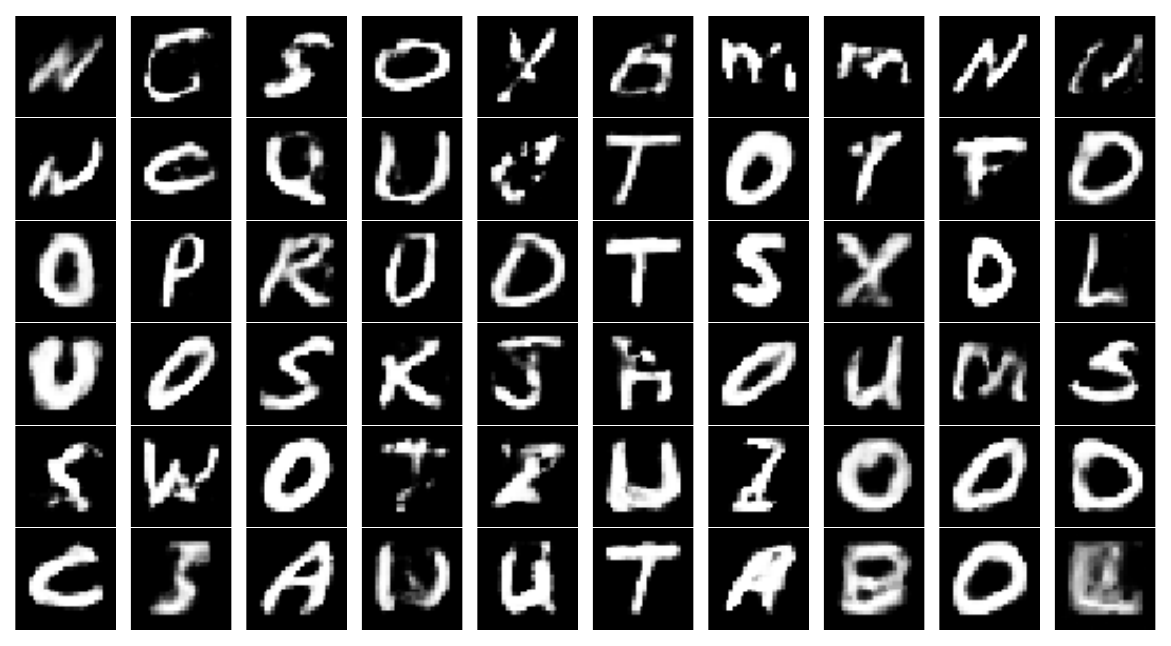}}
        \\
        \subfloat[AE+GMMN]{\includegraphics[width=.45\linewidth, height=.22\linewidth]{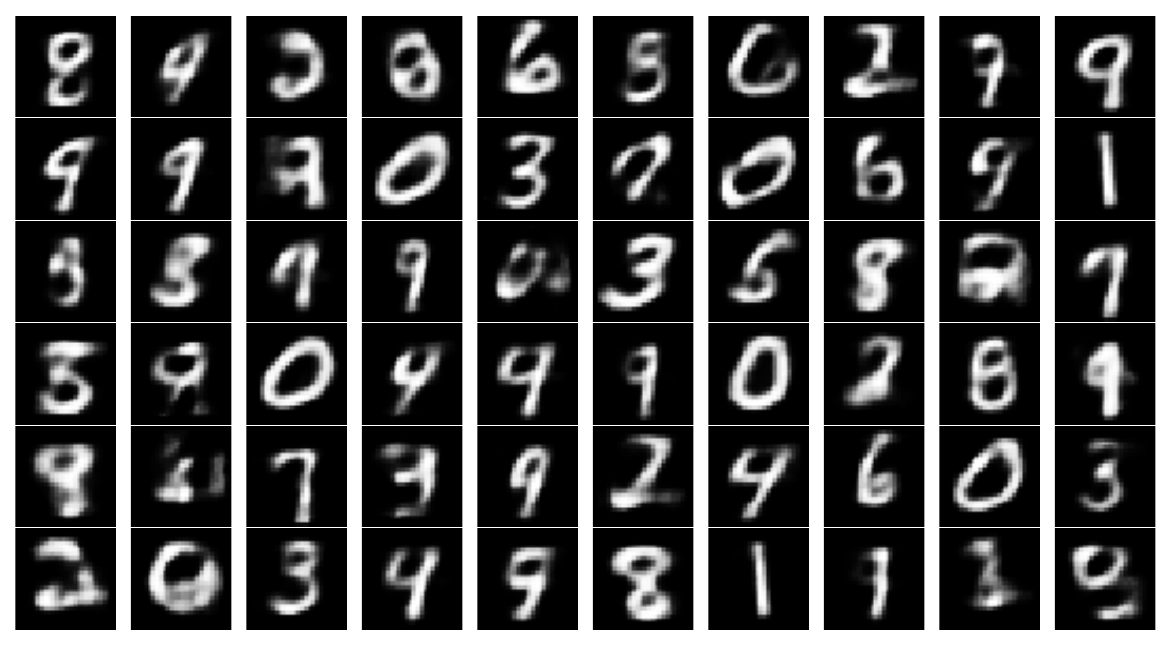}}
        \hspace{1cm}
        \subfloat[AE+GMMN]{\includegraphics[width=.45\linewidth, height=.22\linewidth]{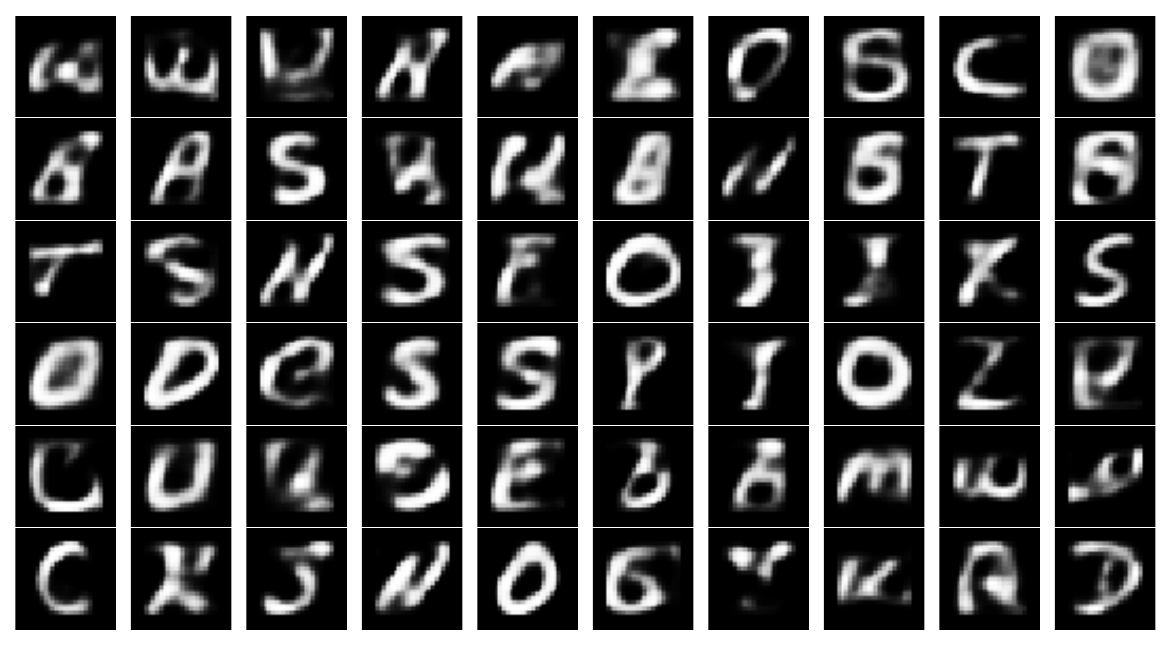}}
        \\
        \subfloat[$\alpha$-WGPGAN]{\includegraphics[width=.45\linewidth, height=.22\linewidth]{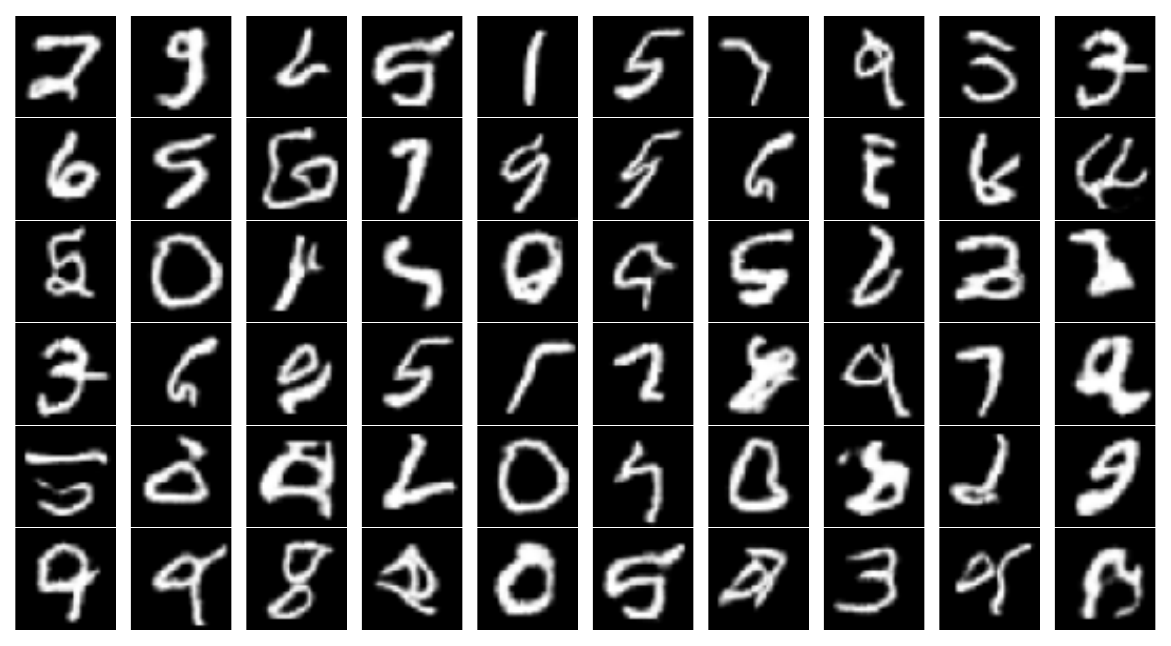}}
        \hspace{1cm}
        \subfloat[$\alpha$-WGPGAN]{\includegraphics[width=.45\linewidth, height=.22\linewidth]{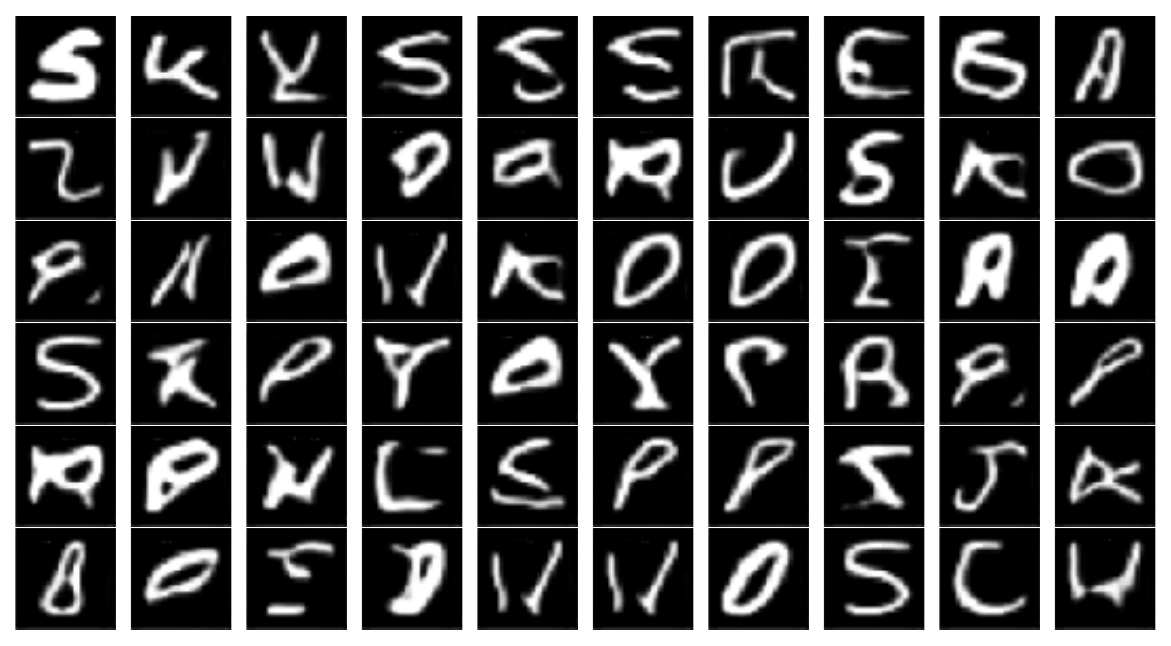}}
    \caption{Visualisation of handwritten training samples and generated samples using various generative models after 400000 iterations. }\label{visual-MNIST}
\end{figure}
In contrast, the outputs of our model and $\alpha$-WGPGAN exhibit higher-resolution samples without any noise. However, it appears that the generated samples of the $\alpha$-WGPGAN model contain slightly more ambiguous images compared to ours, suggesting that our model converges faster than $\alpha$-WGPGAN.

\subsection{Unlabaled Datasets}
The performance of a GAN can vary depending on the characteristics of the training dataset, including its complexity, diversity, quality, and size. Thus, it is crucial to assess the model's effectiveness on more intricate datasets. For a comprehensive evaluation of the model's performance, facial and medical images are the two most important datasets to consider. In this regard, we use the following two main data sources and resize all images within them to 64$\times64$ pixels to train all models. 
\subsubsection{Brain MRI Dataset}
The brain MRI images present a complex medical dataset that poses a significant challenge for researchers.  
These images can be easily accessed online\footnote{\url{https://www.kaggle.com/dsv/2645886}}, with both training and testing sets available, comprising a total of 7,023 images of human brain MRI. The dataset includes glioma, meningioma, no tumor, and pituitary tumors \citep{msoudnickparvar2021}. The training set is composed of 5,712 images of varying sizes, each with extra margins.  To ensure consistency and reduce noise in the training data, a pre-processing code\footnote{\url{https://github.com/masoudnick/Brain-Tumor-MRI-Classification/blob/main/Preprocessing.py}} is used to remove margins before feeding them into the networks for training.

Part (a) of Figure \ref{PCA-mri-celebA} illustrates the PCA plots of the generated samples for all models, highlighting that the dispersion and direction of samples generated by $\alpha$-WGPGAN model differ the most from the real dataset compared to the other models. Meanwhile, part (c) of Figure \ref{PCA-mri-celebA} shows almost identical convergence of MMD scores around zero for the compared models. However, Figure \ref{visual-MRI} portrays noisy and blurry outputs generated by the semi-BNP and AE+GMMN models, whereas our model and the $\alpha$-WGPGAN produce clear outputs. 

\subsubsection{CelebFaces Attributes Dataset}
The CelebFaces attributes dataset (CelebA), collected by \cite{liu2015faceattributes}, includes 202,599 images of celebrities that are publicly available online\footnote{\url{http://mmlab.ie.cuhk.edu.hk/projects/CelebA.html}}. The dataset features people in various poses, with different backgrounds, hairstyles and colors, skin tones, and wearing or not wearing glasses and hats, providing a rich resource for evaluating the performance of data augmentation models. 
\begin{figure}[tht]
    \centering
        \subfloat[MRI: PCA plot]{\includegraphics[width=.5\linewidth, height=.35\linewidth]{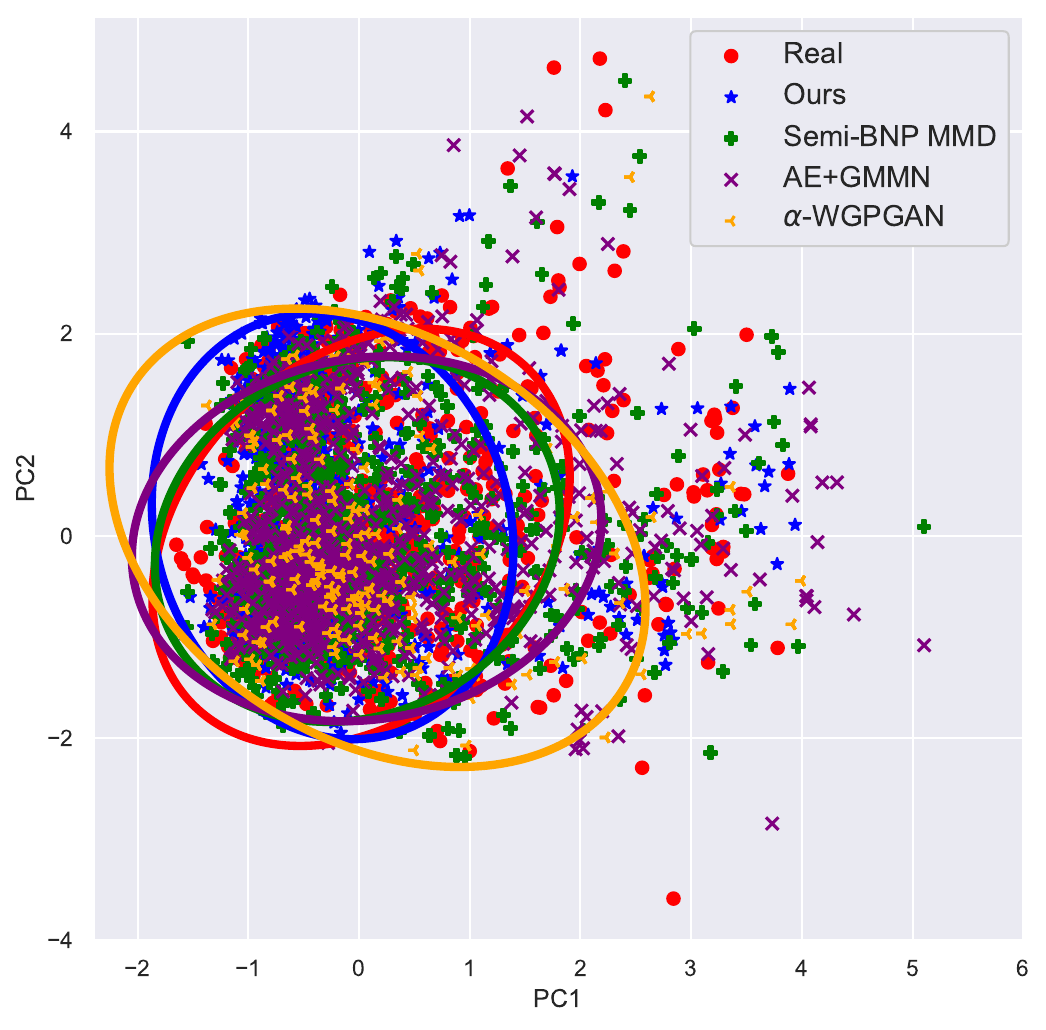}}
        \subfloat[CelebA: PCA plot]{\includegraphics[width=.5\linewidth, height=.35\linewidth]{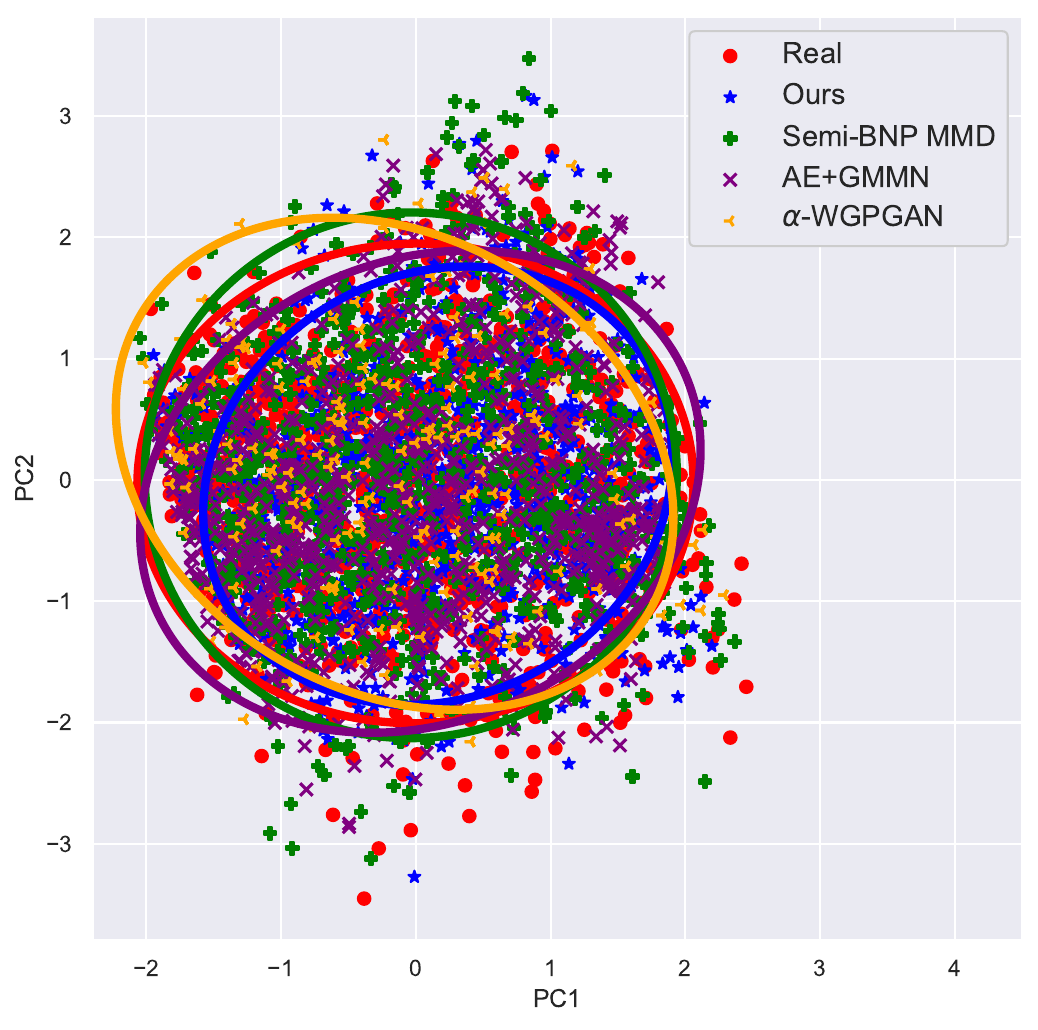}}
        \\
        \subfloat[MRI: Violin plot]
        {\includegraphics[width=.5\linewidth, height=.4\linewidth]{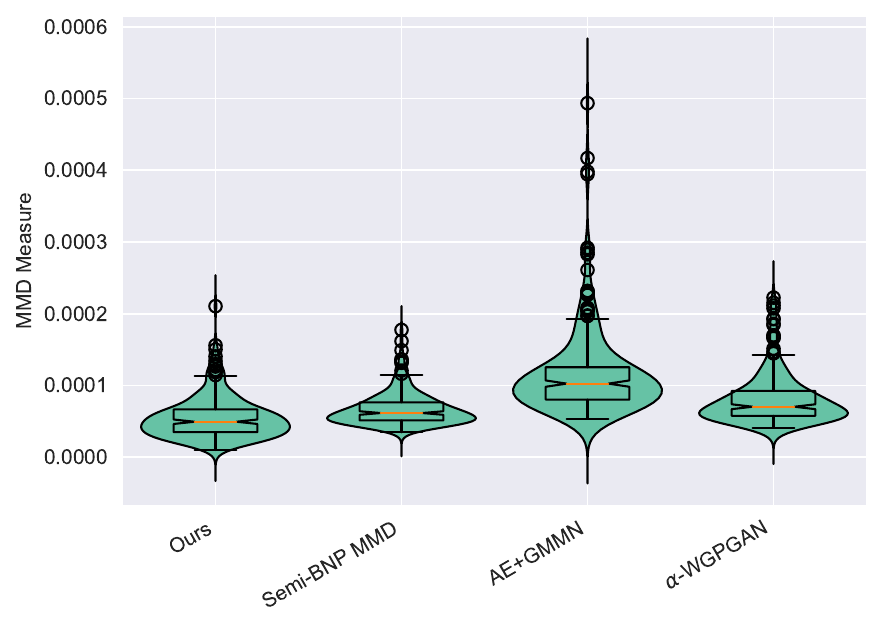}}
        \subfloat[CelebA: Violin plot]{\includegraphics[width=.5\linewidth, height=.4\linewidth]{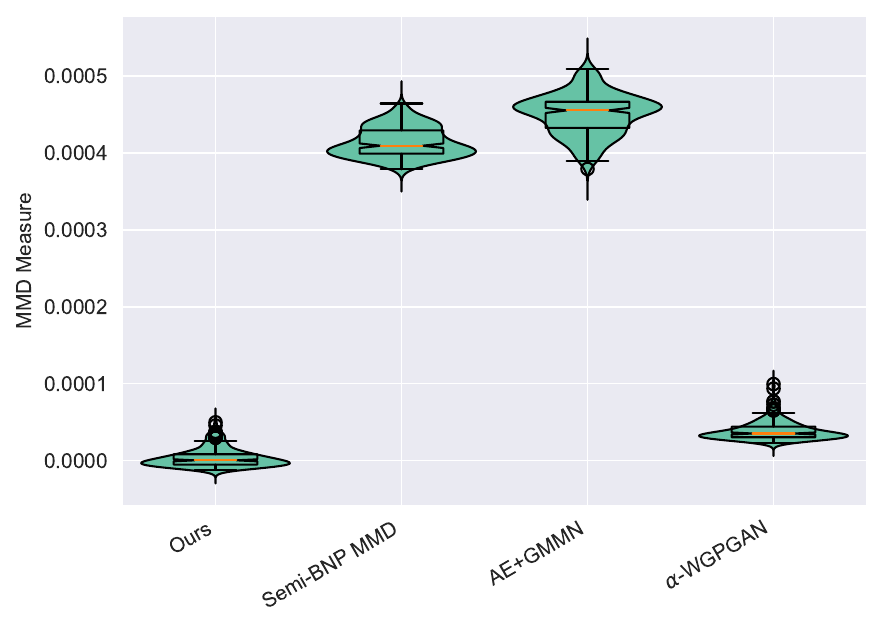}}
    \caption{Top: PCA plots of 1000 generated samples versus real samples with fitted ellipse curves for MRI and celebA datasets, indicating the spread of the samples in the corresponding directions. Bottom: Violin plots of MMD scores including density and box plots.}\label{PCA-mri-celebA}
\end{figure}
While part (b) of Figure \ref{PCA-mri-celebA} shows a slight variation in the direction of the generated sample pattern, part (d) of the same figure highlights a significant gap between the convergence of MMD scores of our proposed approach and $\alpha$-WGPGAN compared to semi-BNP MMD and AE+GMMN around zero.  Our proposed model yields even lower MMD scores than $\alpha$-WGPGAN.

On the other hand, despite the limited number of samples shown in Figure \ref{visual-celebA}, the generated images by our proposed model exhibit a remarkable diversity that encompasses a range of hair colors and styles, skin tones, and accessories such as glasses and hats. This variety suggests that our results are comparable to those produced by $\alpha$-WGPGAN. However, it is worth noting that the images generated by the AE+GMMN model not only suffer from blurriness but also appear to be heavily biased toward female faces, indicating a potential issue with mode collapse in this type of dataset. While the samples generated by the semi-BNP MMD model displayed better results than the AE+GMMN, there is still a level of noise present, indicating that more iterations are needed to ensure model convergence.

\begin{figure}[tht]
    \centering
        \subfloat[Training dataset]{\includegraphics[width=1\linewidth, height=.23\linewidth]{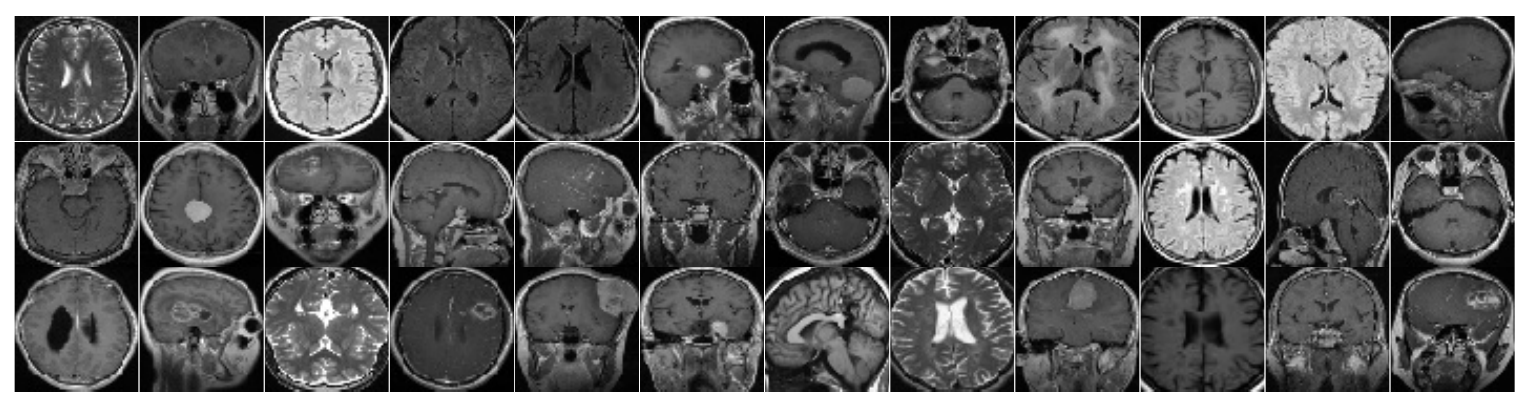}}
        \hspace{1.5cm}\\
        \subfloat[Ours]{\includegraphics[width=1\linewidth, height=.23\linewidth]{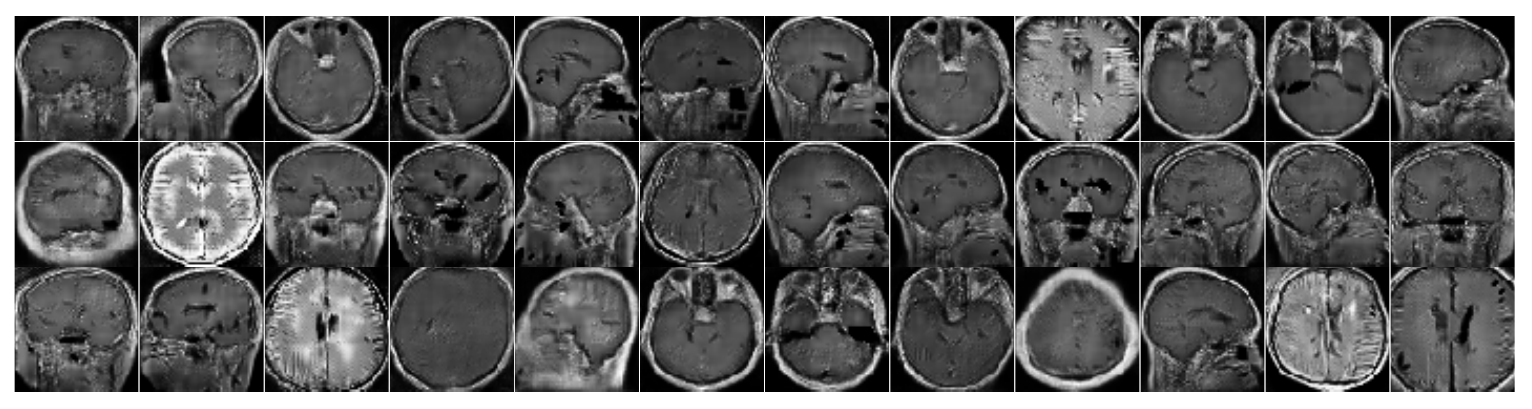}}
        \hspace{1.5cm}\\
        \subfloat[Semi-BNP MMD]{\includegraphics[width=1\linewidth, height=.23\linewidth]{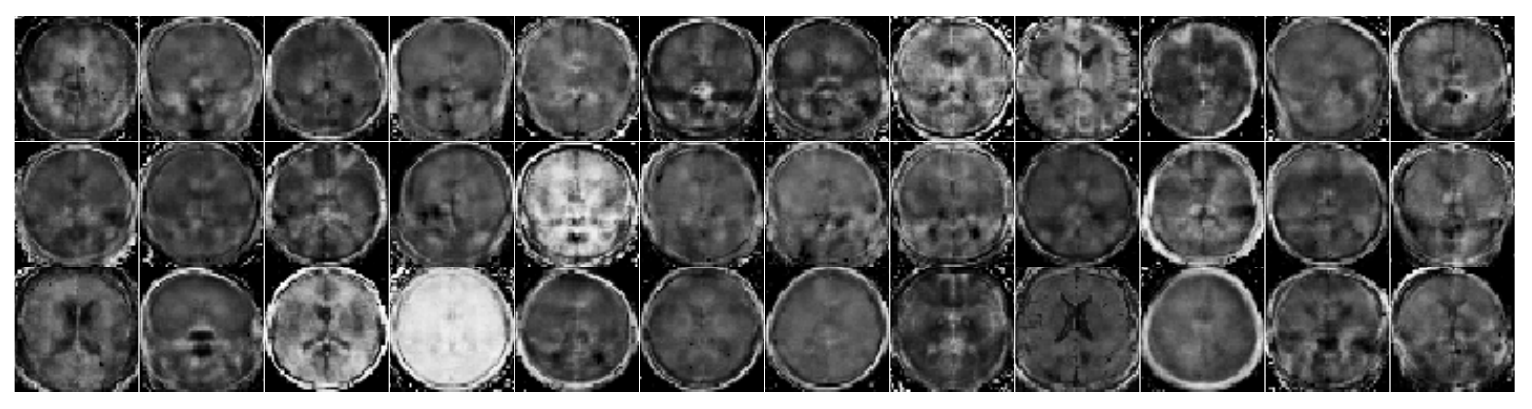}}
        \hspace{1.5cm}\\
        \subfloat[AE+GMMN]{\includegraphics[width=1\linewidth, height=.23\linewidth]{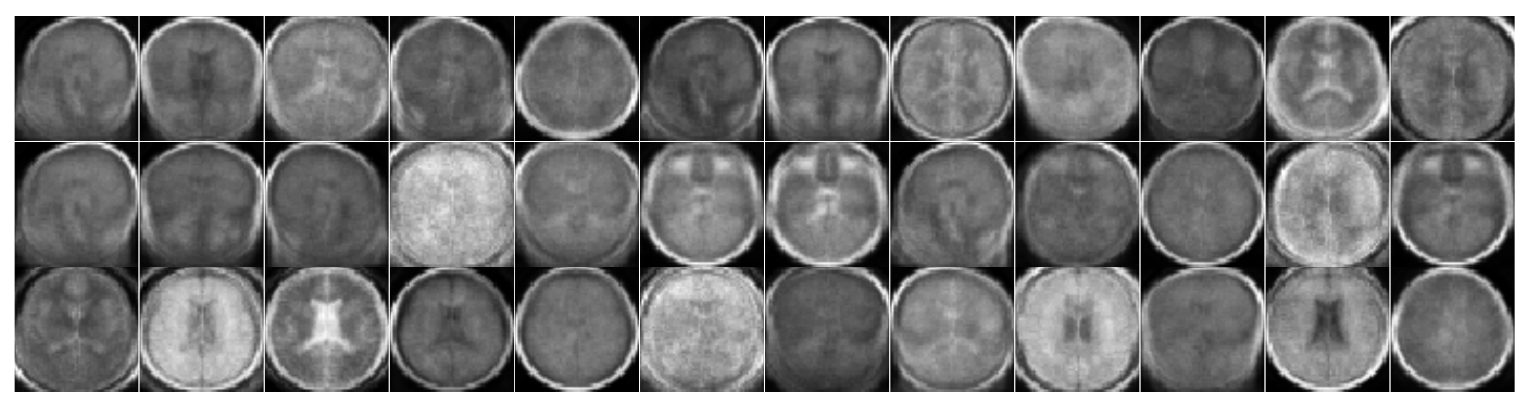}}
        \hspace{1.5cm}\\
        \subfloat[$\alpha$-WGPGAN]{\includegraphics[width=1\linewidth, height=.23\linewidth]{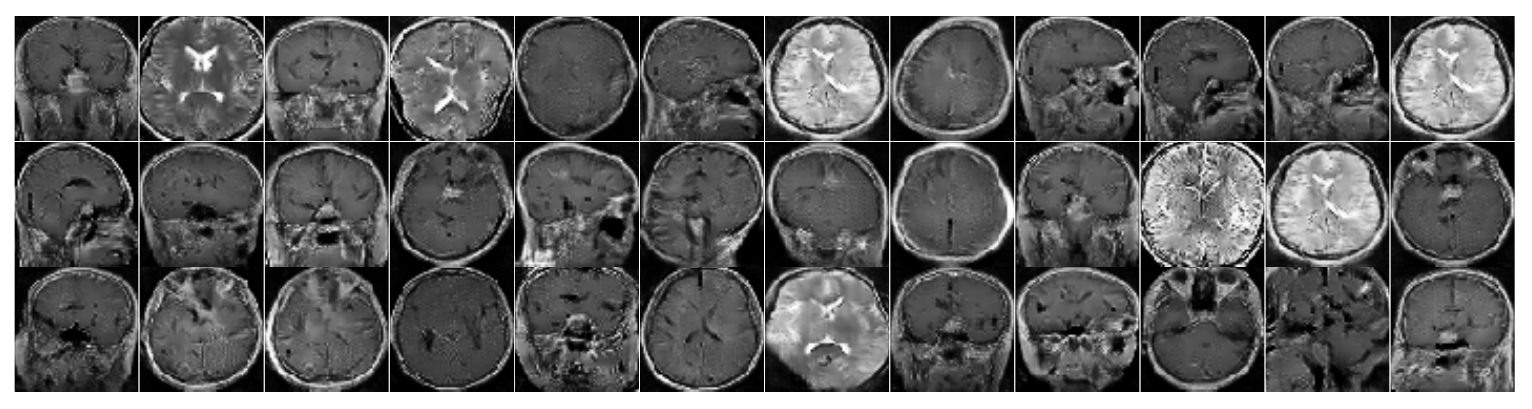}}
        \hspace{1.5cm}\\
    \caption{ Visualisation of MRI training samples and generated samples using various generative models after 400000 iterations.}\label{visual-MRI}
\end{figure}

\begin{figure}[tht]
    \centering
        \subfloat[Training dataset]{\includegraphics[width=1\linewidth, height=.23\linewidth]{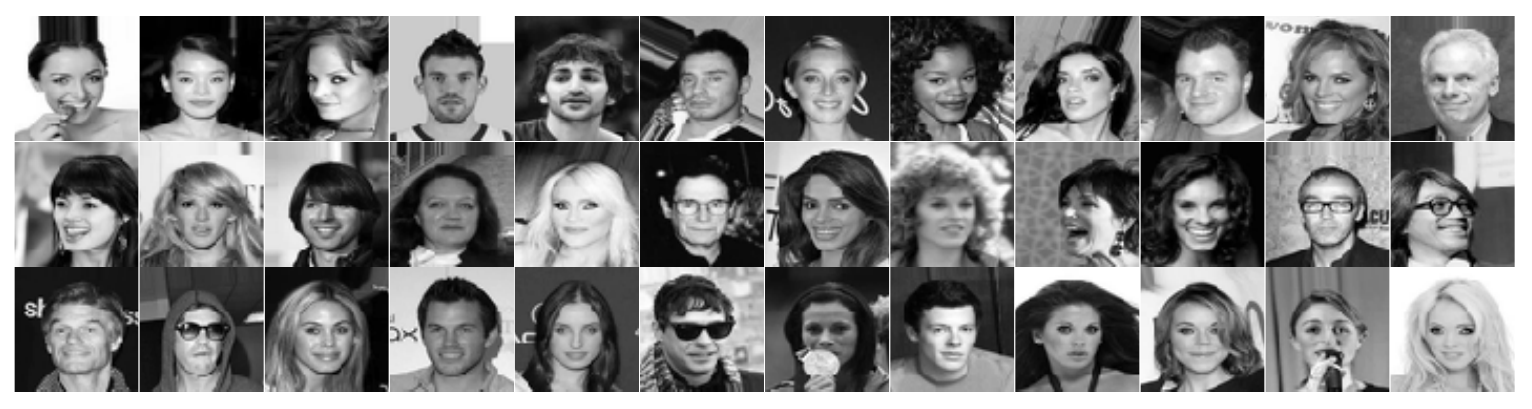}}
        \hspace{1.5cm}\\
        \subfloat[Ours]{\includegraphics[width=1\linewidth, height=.23\linewidth]{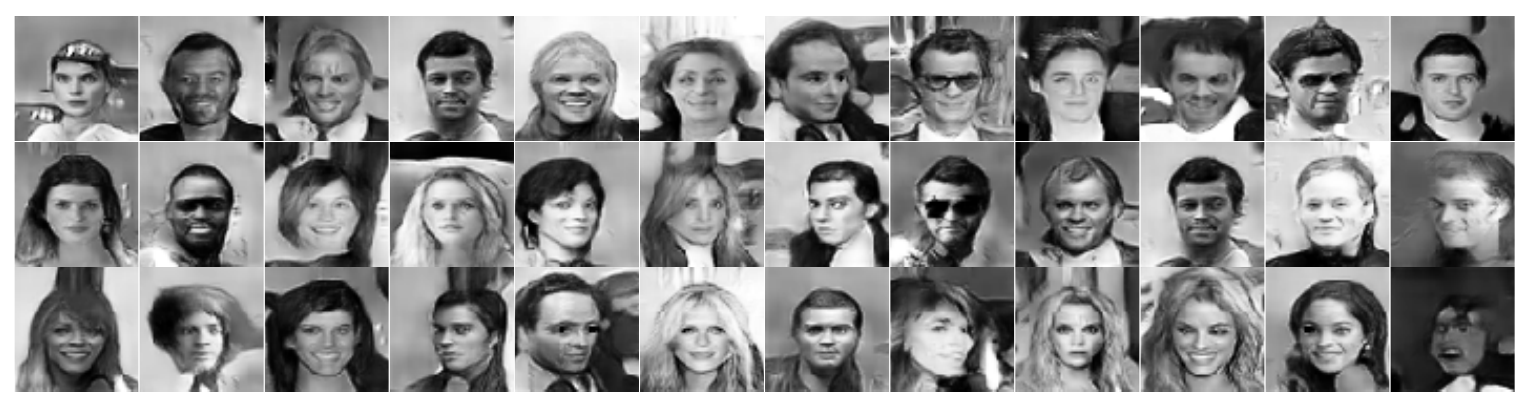}}
        \hspace{1.5cm}\\
        \subfloat[Semi-BNP MMD]{\includegraphics[width=1\linewidth, height=.23\linewidth]{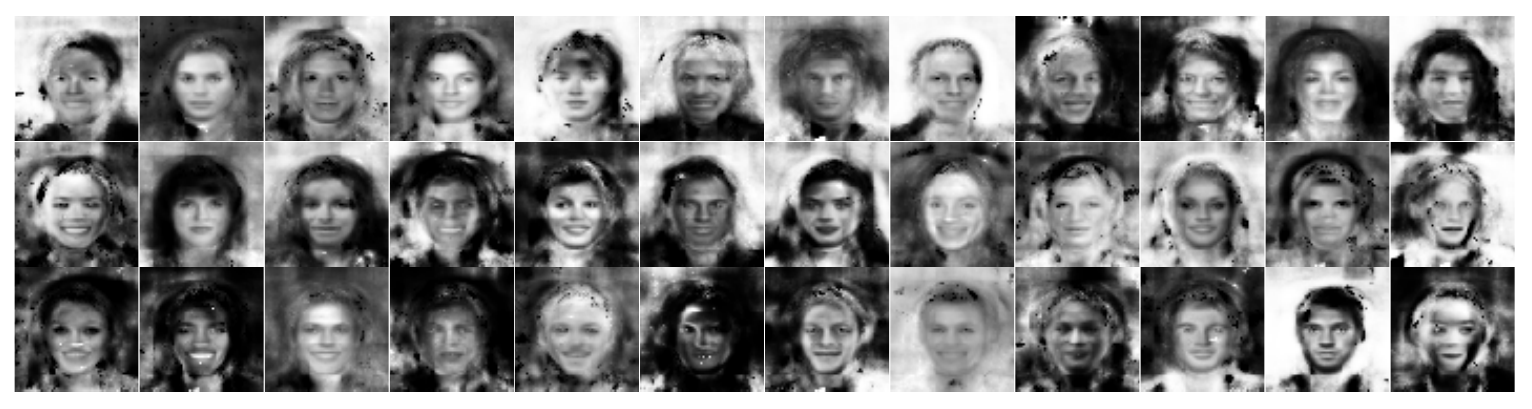}}
        \hspace{1.5cm}\\
        \subfloat[AE+GMMN]{\includegraphics[width=1\linewidth, height=.23\linewidth]{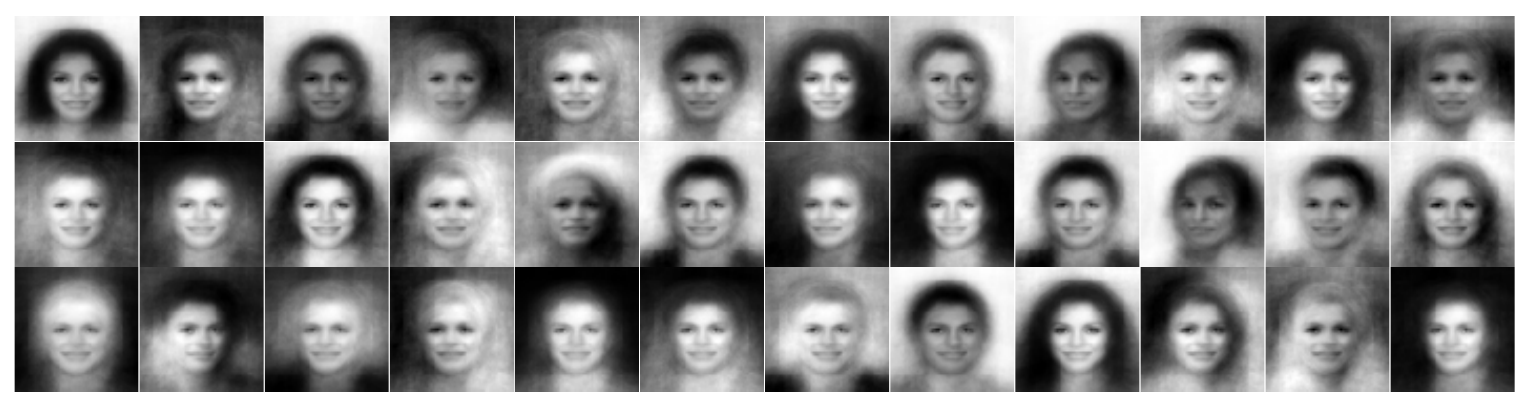}}
        \hspace{1.5cm}\\
        \subfloat[$\alpha$-WGPGAN]{\includegraphics[width=1\linewidth, height=.23\linewidth]{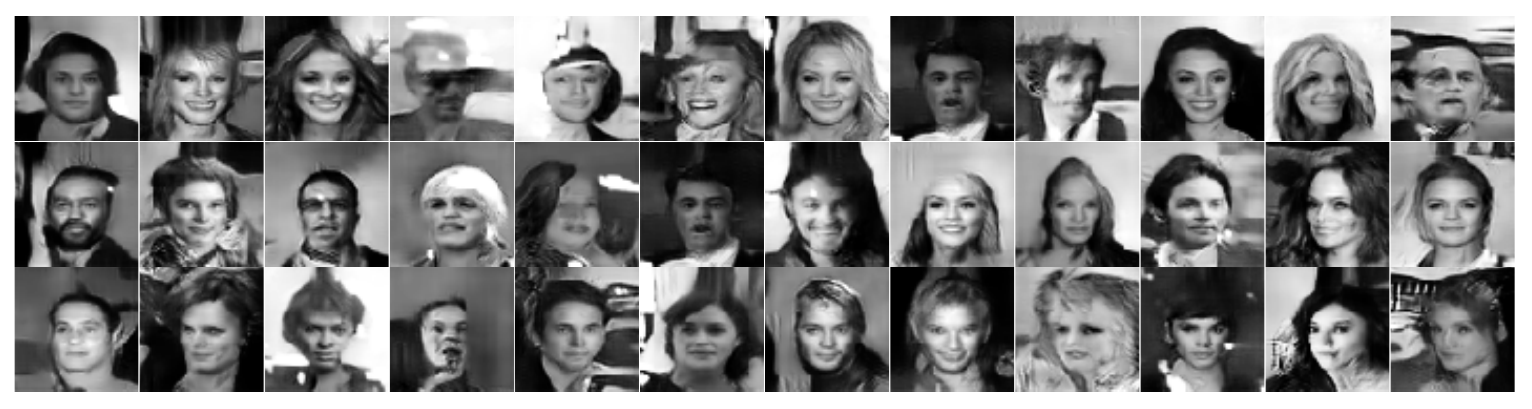}}
        \hspace{1.5cm}\\
    \caption{ Visualisation of celebA training samples and generated samples using various generative models after 400000 iterations.}\label{visual-celebA}
\end{figure}

\subsection{\textcolor{black}{Comparative Analysis of the BNP Model}}
Although the proposed model outperformed well-known baselines, further investigation is required to understand the individual contributions of each component within the triple model. Additionally, there is interest in exploring how the BNPL framework impacts training stability. These concerns are addressed in the following sections.
\subsubsection{\textcolor{black}{triple model vs. partial models}}
We evaluate 1000 samples generated by the triple model against two partial models within the BNPL framework: VAE-GAN (dual model) and GAN (single model). Figure \ref{sub-component} presents box plots of MMD scores for each model across all datasets considered in this study. The results demonstrate the significant advantage of the triple model over the dual model, as well as the dual model over the single model. Progressing from the single model to the triple model results in reduced MMD scores, indicating improved performance and greater diversity in the generated samples.

To further illustrate the differences among the triple, dual, and single models, we particularly provide relative frequency plots of predicted labels for generated handwritten digits as determined by the trained classifier shown in Figure \ref{classifierr}. Figure \ref{fig:freq_percentage} demonstrates that the triple model more effectively preserves the frequency percentage of labels compared to the other two models, highlighting its superior performance. Figure \ref{fig:abs_diff} depicts the absolute differences between the frequency percentages of real and generated samples across all labels. The results show that minimizing these absolute differences across labels is most effectively achieved by the triple model, whereas the single GAN model performs the worst.

With particular attention, although the dual model shows improved diversity compared to the single GAN, as evidenced in Figures \ref{fig:freq_percentage} and \ref{fig:abs_diff}, it still produces unclear samples (highlighted by the red square within the green frame in Figure \ref{fig:gen_samples}). These unclear samples confuse the classifier in predicting digit labels, ultimately affecting the diversity of the generated samples. This highlights the limitation of the dual model, further emphasizing the advantage of the triple model in generating both diverse and clear samples.


\begin{figure}[ht]
    \centering
        \subfloat[Numbers]{\includegraphics[width=.5\linewidth, height=.35\linewidth]{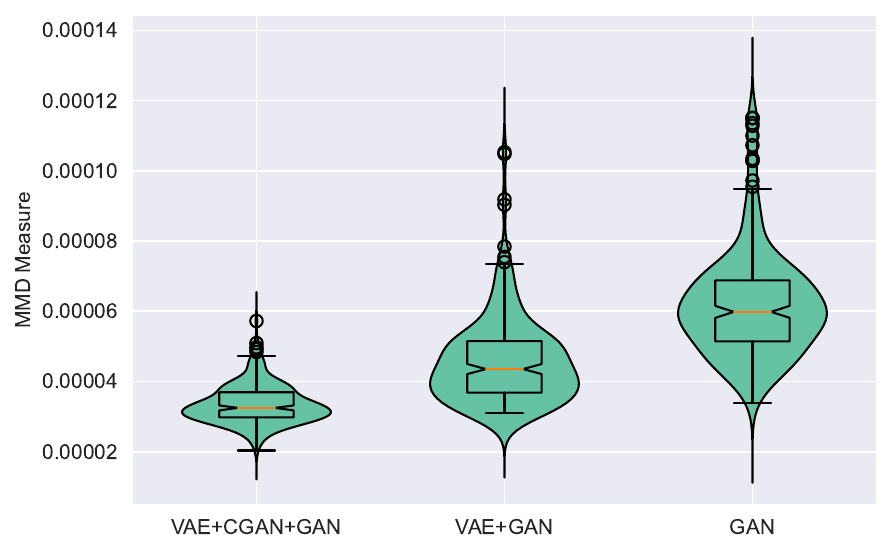}}
        \subfloat[Letters]{\includegraphics[width=.5\linewidth, height=.35\linewidth]{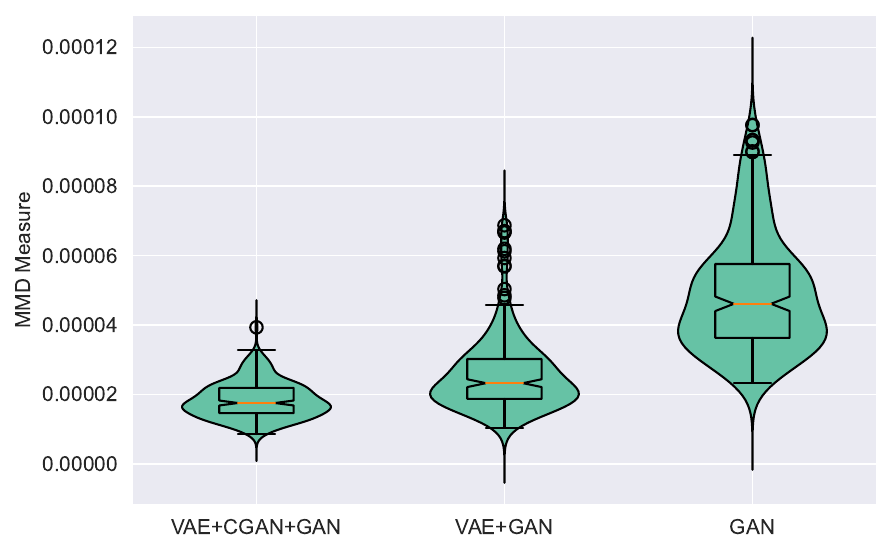}}\\
        \subfloat[MRI]
        {\includegraphics[width=.5\linewidth, height=.4\linewidth]{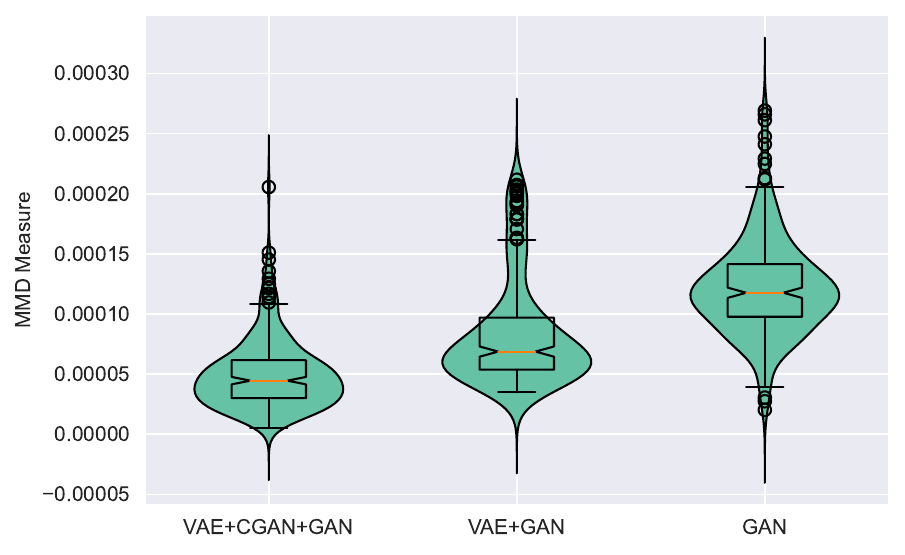}}
        \subfloat[CelebA]{\includegraphics[width=.5\linewidth, height=.4\linewidth]{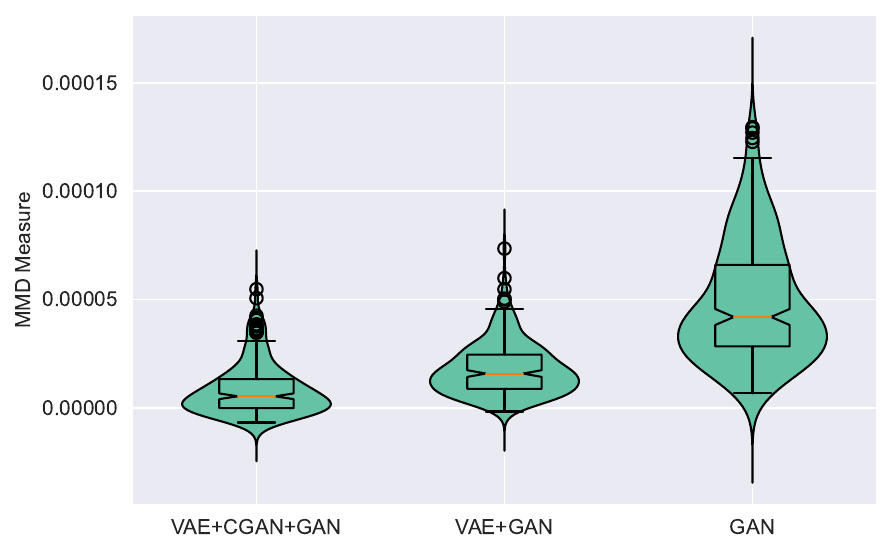}}
    \caption{\textcolor{black}{Comparison of MMD scores for various model configurations across datasets.}}\label{sub-component}
\end{figure}

\begin{figure}[ht]
    {\subfloat[Frequency Percentage]{\includegraphics[width=.49\textwidth,height=4cm]{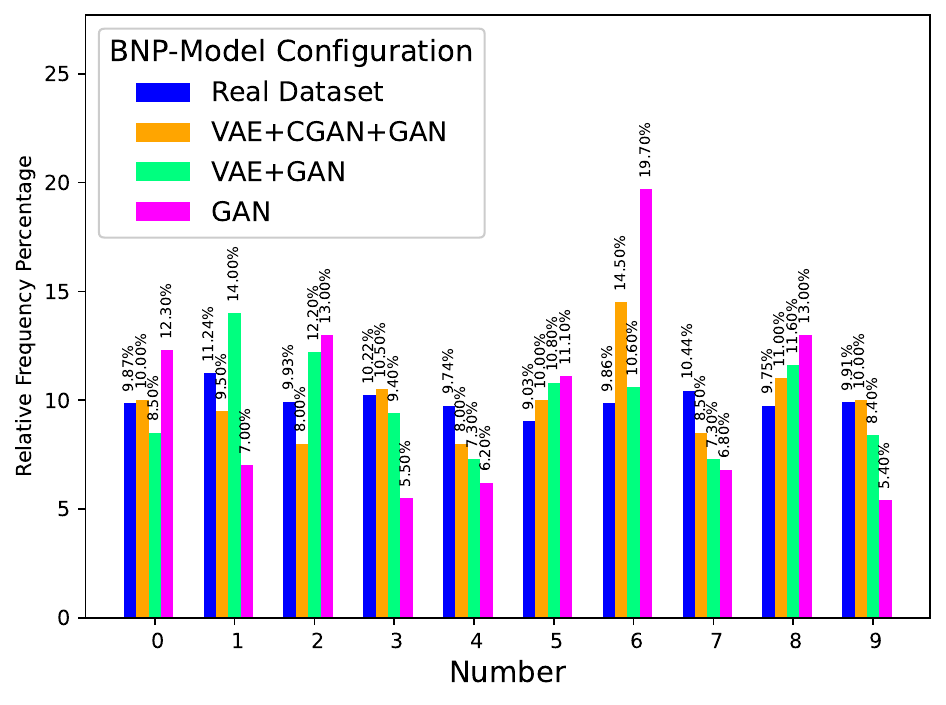}
    \label{fig:freq_percentage}}}
    {\subfloat[Absolute Difference of Frequency Percentage]{\includegraphics[width=.49\textwidth,height=4cm]{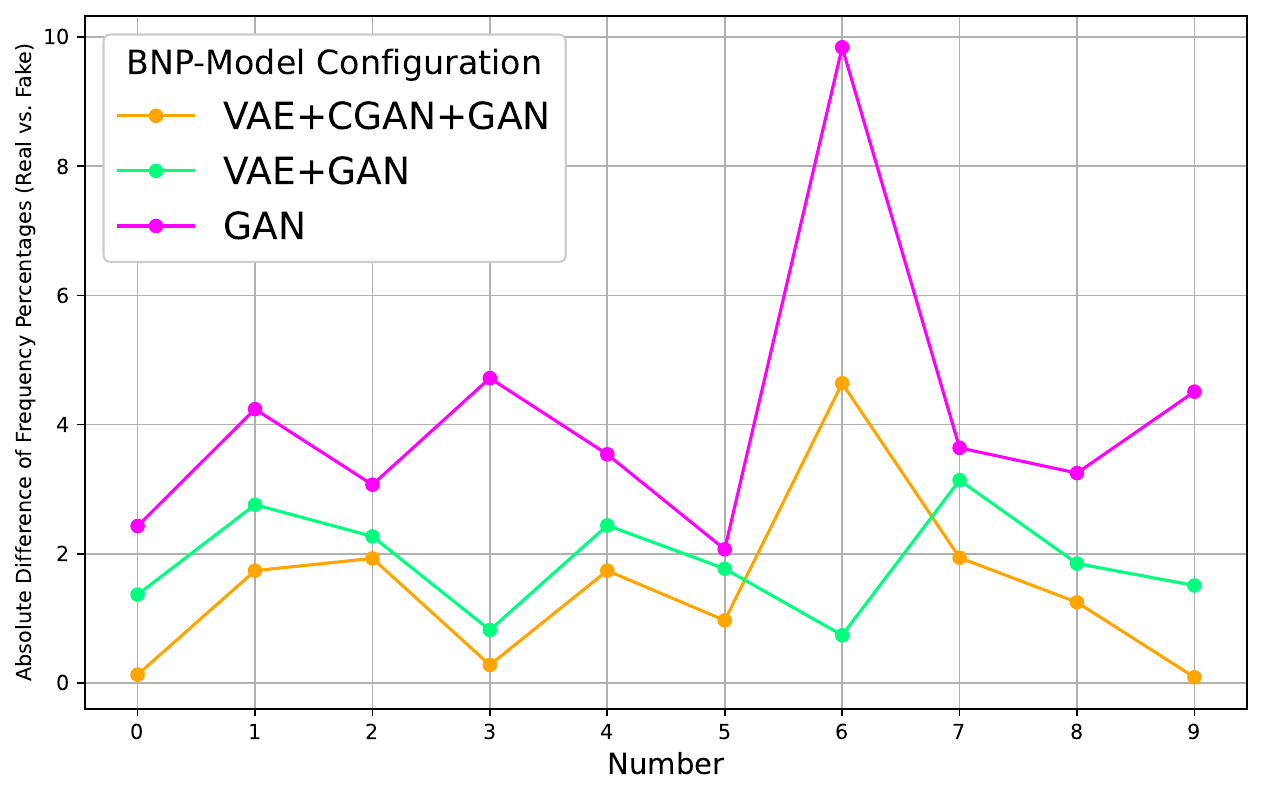}
    \label{fig:abs_diff}}}\\
    {\subfloat[Generated Samples]{\includegraphics[width=1\textwidth,height=3cm]{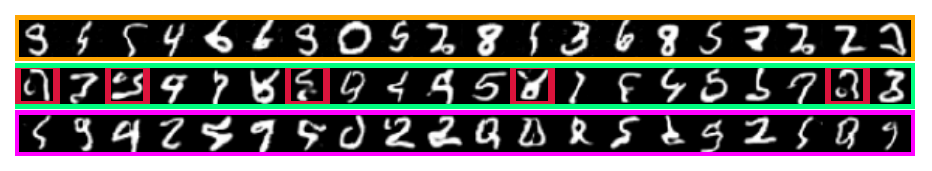}
    \label{fig:gen_samples}}}\\
    \caption{\textcolor{black}{Comparison of the triple BNP model vs. partial models: The triple model—VAE+CGAN+GAN (orange)—is compared to two partial models: the dual model VAE+GAN (green) and the single model GAN (pink). All models are trained by minimizing the combined distance $\text{WMMD}(F^{\mathrm{pos}}, F_{Gen_{\boldsymbol{\omega}}})$. The top panel shows statistical analysis on 1000 generated samples, while the bottom panel visualizes 20 generated samples.}}
    \label{fig:comparison}
\end{figure}


\subsubsection{\textcolor{black}{BNPL Compared to the FNP Counterpart Using Resampling Technique}}\label{sec:BNP vs FNP}
 The training performance of BNPL is compared to its exact FNP counterpart, which employs a resampling technique, as described in the final paragraph of Section \ref{sec:wasserstein}. During each iteration of the training process with a mini-batch sample $\mathbf{X}_{1:n_{mb}}$, instead of applying the BNP approximation \eqref{w-mmd-pos} for training the generator, a simple resampling technique with replacement is used. This involves drawing samples $\mathbf{X}^{\prime}_{1}, \ldots, \mathbf{X}^{\prime}_{2n_{mb}} \overset{i.i.d.}{\sim} F_{\mathbf{X}_{1:n_{mb}}}$ to compute the exact FNP counterpart $\text{WMMD}(F_{\mathbf{X}^{\prime}_{1:2n_{mb}}},F_{Gen_{\boldsymbol{\omega}}(\mathbf{Z}_{1:2n_{mb}})})$, where $\mathbf{Z}_i$ is a latent variable sampled from either the noise distribution or the encoder distribution.

It is worth emphasizing that $N$ in BNP approximation \eqref{w-mmd-pos} is a random variable, as defined by \eqref{random-stopping} within the DP context. Therefore, it is not meaningful to use $N$ as the resampling size in the FNP counterpart. Instead, a resampling size of $2n_{mb}$ was chosen for these experiments, as there is no specific rule for determining this parameter.

Figure \ref{loss-rate} presents the WMMD values over different iterations for both strategies across all datasets included in this study. The results demonstrate notably reduced variation in learning rates for BNPL compared to its FNP counterpart, along with improved convergence. These findings suggest more stabilized training under the BNPL framework.

\begin{figure}[ht]
    \centering
        \subfloat[Numbers]{\includegraphics[width=.5\linewidth, height=.35\linewidth]{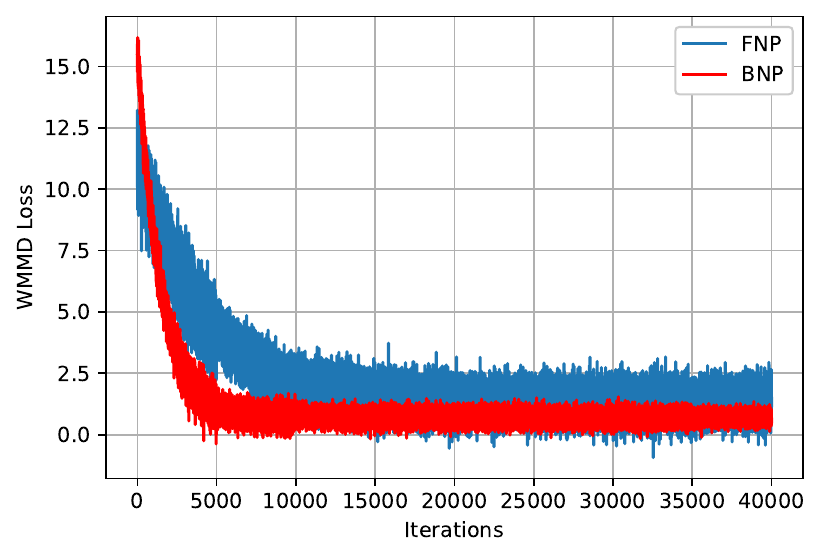}}
        \subfloat[Letters]{\includegraphics[width=.5\linewidth, height=.35\linewidth]{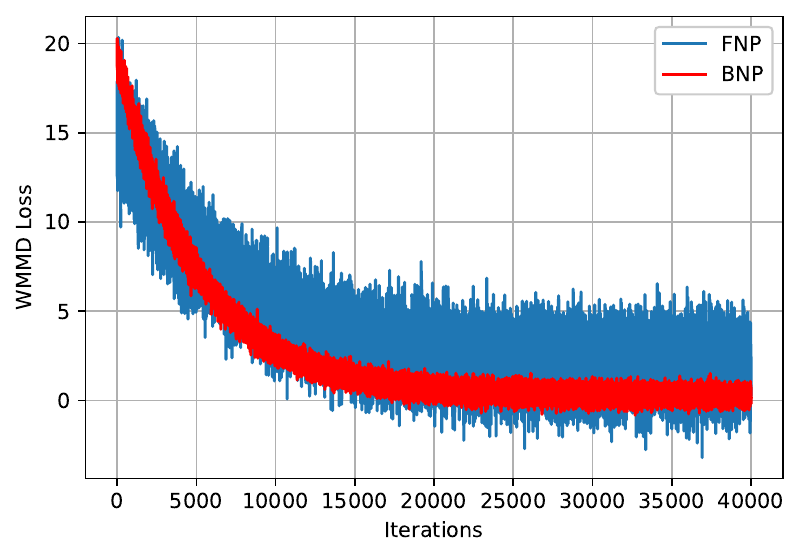}}\\
        \subfloat[MRI]
        {\includegraphics[width=.5\linewidth, height=.4\linewidth]{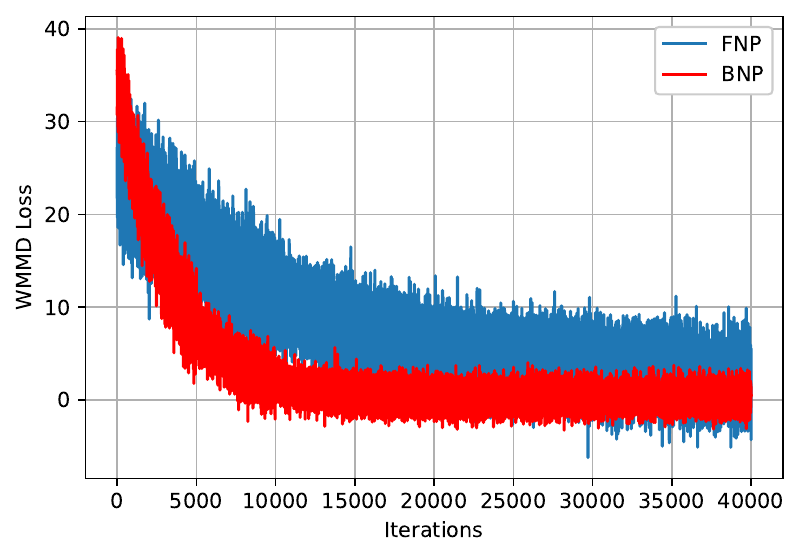}}
        \subfloat[CelebA]{\includegraphics[width=.5\linewidth, height=.4\linewidth]{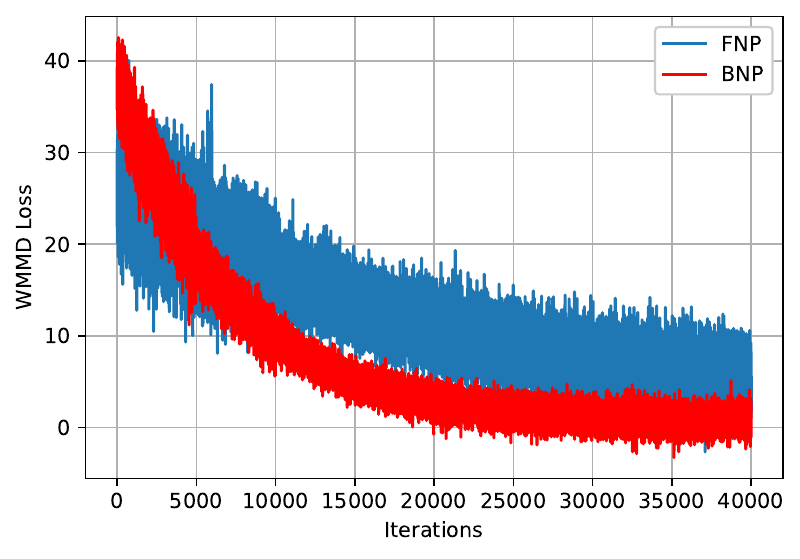}}
    \caption{\textcolor{black}{BNPL rate compared to the learning rate of the exact FNP counterpart, which employs resampling in computing the WMMD distance.}}\label{loss-rate}
\end{figure}

\section{Conclusion}\label{sec:conclusion}
We have proposed a powerful triple generative model that has produced realistic samples within the BNPL framework. We have extended the semi-BNP MMD GAN, introduced by \cite{fazeli2023semi}, by incorporating both Wasserstein and MMD measures into the GAN loss function along with a VAE and an CGAN to enhance its ability to produce diverse outputs. Different types of datasets have been used to examine the performance of the proposed model, indicating that it is a competitive model. Our model has also been compared to several generative models, such as $\alpha$-WGPGAN, and has outperformed them in terms of mitigating mode collapse and producing noise-free images. 

To improve the efficiency and effectiveness of communication between the generator and discriminator networks, we plan to employ multiplex networks in our model. Multiplex networks can tackle the problem of data sparsity in GANs by incorporating multiple types of interactions and relationships between nodes. This allows the model to learn from a larger and more diverse set of data, improving its ability to generate realistic samples. For future work, we are considering extending our proposed idea to include 3D medical datasets for detecting anomalies like Down's syndrome in fetuses. This development could prevent the birth of defective babies.  We hope that a more powerful generative model can help solve important issues in medical imaging.
\bibliographystyle{apalike}
\bibliography{GAN-bibliography}

\begin{thebibliography}{47}
\providecommand{\natexlab}[1]{#1}
\providecommand{\url}[1]{\texttt{#1}}
\expandafter\ifx\csname urlstyle\endcsname\relax
  \providecommand{\doi}[1]{doi: #1}\else
  \providecommand{\doi}{doi: \begingroup \urlstyle{rm}\Url}\fi

\bibitem[Arjovsky and Bottou(2017)]{Arjovsky}
M.~Arjovsky and L.~Bottou.
\newblock Towards principled methods for training generative adversarial
  networks.
\newblock \emph{CoRR}, abs/1701.04862, 2017.
\newblock URL \url{http://arxiv.org/abs/1701.04862}.

\bibitem[Arjovsky et~al.(2017)Arjovsky, Chintala, and
  Bottou]{arjovsky2017wasserstein}
M.~Arjovsky, S.~Chintala, and L.~Bottou.
\newblock Wasserstein generative adversarial networks.
\newblock In \emph{International Conference on Machine Learning}, pages
  214--223. PMLR, 2017.

\bibitem[Bariletto and Ho(2024)]{bariletto2024bayesian}
N.~Bariletto and N.~Ho.
\newblock Bayesian nonparametrics meets data-driven distributionally robust
  optimization.
\newblock In \emph{The Thirty-eighth Annual Conference on Neural Information
  Processing Systems}, 2024.
\newblock URL \url{https://openreview.net/forum?id=8CguPoe3TP}.

\bibitem[Bengio et~al.(2006)Bengio, Lamblin, Popovici, and
  Larochelle]{bengio2006greedy}
Y.~Bengio, P.~Lamblin, D.~Popovici, and H.~Larochelle.
\newblock Greedy layer-wise training of deep networks.
\newblock \emph{Advances in Neural Information Processing Systems}, 19, 2006.

\bibitem[Bondesson(1982)]{bondesson1982simulation}
L.~Bondesson.
\newblock On simulation from infinitely divisible distributions.
\newblock \emph{Advances in Applied Probability}, 14\penalty0 (4):\penalty0
  855--869, 1982.

\bibitem[Ch{\'e}rief-Abdellatif and Alquier(2022)]{cherief2022finite}
B.-E. Ch{\'e}rief-Abdellatif and P.~Alquier.
\newblock Finite sample properties of parametric {MMD} estimation: {R}obustness
  to misspecification and dependence.
\newblock \emph{Bernoulli}, 28\penalty0 (1):\penalty0 181--213, 2022.

\bibitem[Cohen et~al.(2017)Cohen, Afshar, Tapson, and
  Van~Schaik]{cohen2017emnist}
G.~Cohen, S.~Afshar, J.~Tapson, and A.~Van~Schaik.
\newblock {EMNIST}: Extending {MNIST} to handwritten letters.
\newblock In \emph{2017 International Joint Conference on Neural Networks},
  pages 2921--2926. IEEE, 2017.

\bibitem[Creswell and Bharath(2018)]{creswell2018denoising}
A.~Creswell and A.~A. Bharath.
\newblock Denoising adversarial autoencoders.
\newblock \emph{IEEE Transactions on Neural Networks and Learning Systems},
  30\penalty0 (4):\penalty0 968--984, 2018.

\bibitem[Dellaporta et~al.(2022)Dellaporta, Knoblauch, Damoulas, and
  Briol]{dellaporta2022robust}
C.~Dellaporta, J.~Knoblauch, T.~Damoulas, and F.-X. Briol.
\newblock Robust {B}ayesian inference for simulator-based models via the {MMD}
  posterior bootstrap.
\newblock In \emph{International Conference on Artificial Intelligence and
  Statistics}, pages 943--970. PMLR, 2022.

\bibitem[Donahue et~al.(2016)Donahue, Kr{\"a}henb{\"u}hl, and
  Darrell]{donahue2016adversarial}
J.~Donahue, P.~Kr{\"a}henb{\"u}hl, and T.~Darrell.
\newblock Adversarial feature learning.
\newblock \emph{arXiv preprint arXiv:1605.09782}, 2016.

\bibitem[Dumoulin et~al.(2016)Dumoulin, Belghazi, Poole, Mastropietro, Lamb,
  Arjovsky, and Courville]{dumoulin2016adversarially}
V.~Dumoulin, I.~Belghazi, B.~Poole, O.~Mastropietro, A.~Lamb, M.~Arjovsky, and
  A.~Courville.
\newblock Adversarially learned inference.
\newblock \emph{arXiv preprint arXiv:1606.00704}, 2016.

\bibitem[Dziugaite et~al.(2015)Dziugaite, Roy, and
  Ghahramani]{dziugaite2015training}
G.~K. Dziugaite, D.~M. Roy, and Z.~Ghahramani.
\newblock Training generative neural networks via maximum mean discrepancy
  optimization.
\newblock In \emph{Proceedings of the Thirty-First Conference on Uncertainty in
  Artificial Intelligence}, pages 258--267, 2015.

\bibitem[Fazeli-Asl et~al.(2024)Fazeli-Asl, Zhang, and Lin]{fazeli2023semi}
F.~Fazeli-Asl, M.~M. Zhang, and L.~Lin.
\newblock A semi-bayesian nonparametric estimator of the maximum mean
  discrepancy measure: Applications in goodness-of-fit testing and generative
  adversarial networks.
\newblock \emph{Transactions on Machine Learning Research}, 2024.
\newblock ISSN 2835-8856.
\newblock URL \url{https://openreview.net/forum?id=lUnlHS1FYT}.

\bibitem[Ferguson(1973)]{Ferguson}
T.~S. Ferguson.
\newblock A {B}ayesian analysis of some nonparametric problems.
\newblock \emph{The Annals of Statistics}, 1\penalty0 (2):\penalty0 209--230,
  1973.

\bibitem[Fong et~al.(2019)Fong, Lyddon, and Holmes]{fong2019scalable}
E.~Fong, S.~P. Lyddon, and C.~C. Holmes.
\newblock Scalable nonparametric sampling from multimodal posteriors with the
  posterior bootstrap.
\newblock In \emph{International Conference on Machine Learning}, pages
  1952--1962. PMLR, 2019.

\bibitem[Goodfellow et~al.(2014)Goodfellow, Pouget-Abadie, Mirza, Xu,
  Warde-Farley, Ozair, Courville, and Bengio]{Goodfellow}
I.~Goodfellow, J.~Pouget-Abadie, M.~Mirza, B.~Xu, D.~Warde-Farley, S.~Ozair,
  A.~Courville, and Y.~Bengio.
\newblock Generative adversarial nets.
\newblock \emph{Advances in Neural Information Processing Systems},
  27:\penalty0 2672--2680, 2014.

\bibitem[Gretton et~al.(2012)Gretton, Borgwardt, Rasch, Sch{\"o}lkopf, and
  Smola]{Gretton}
A.~Gretton, K.~M. Borgwardt, M.~J. Rasch, B.~Sch{\"o}lkopf, and A.~Smola.
\newblock A kernel two-sample test.
\newblock \emph{The Journal of Machine Learning Research}, 13\penalty0
  (1):\penalty0 723--773, 2012.

\bibitem[Gulrajani et~al.(2017)Gulrajani, Ahmed, Arjovsky, Dumoulin, and
  Courville]{gulrajani2017improved}
I.~Gulrajani, F.~Ahmed, M.~Arjovsky, V.~Dumoulin, and A.~C. Courville.
\newblock Improved training of {W}asserstein {GAN}s.
\newblock \emph{Advances in Neural Information Processing Systems}, 30, 2017.

\bibitem[Hotelling(1933)]{hotelling1933analysis}
H.~Hotelling.
\newblock Analysis of a complex of statistical variables into principal
  components.
\newblock \emph{Journal of educational psychology}, 24\penalty0 (6):\penalty0
  417, 1933.

\bibitem[Im et~al.(2017)Im, Ahn, Memisevic, and Bengio]{im2017denoising}
D.~Im, S.~Ahn, R.~Memisevic, and Y.~Bengio.
\newblock Denoising criterion for variational auto-encoding framework.
\newblock In \emph{Proceedings of the AAAI Conference on Artificial
  Intelligence}, volume~31, 2017.

\bibitem[Ishwaran and Zarepour(2002)]{Ishwaran}
H.~Ishwaran and M.~Zarepour.
\newblock Exact and approximate sum representations for the {D}irichlet
  process.
\newblock \emph{Canadian Journal of Statistics}, 30\penalty0 (2):\penalty0
  269--283, 2002.

\bibitem[Jafari et~al.(2023)Jafari, Cevik, and Basar]{jafari2023improved}
S.~M. Jafari, M.~Cevik, and A.~Basar.
\newblock Improved $\alpha$-{GAN} architecture for generating {3D} connected
  volumes with an application to radiosurgery treatment planning.
\newblock \emph{Applied Intelligence}, pages 1--27, 2023.

\bibitem[Kingma and Welling(2013)]{kingma2013auto}
D.~P. Kingma and M.~Welling.
\newblock Auto-encoding variational {B}ayes.
\newblock \emph{arXiv preprint arXiv:1312.6114}, 2013.

\bibitem[Kingma et~al.(2016)Kingma, Salimans, Jozefowicz, Chen, Sutskever, and
  Welling]{kingma2016improved}
D.~P. Kingma, T.~Salimans, R.~Jozefowicz, X.~Chen, I.~Sutskever, and
  M.~Welling.
\newblock Improved variational inference with inverse autoregressive flow.
\newblock \emph{Advances in neural information processing systems}, 29, 2016.

\bibitem[Kodali et~al.(2017)Kodali, Abernethy, Hays, and
  Kira]{kodali2017convergence}
N.~Kodali, J.~Abernethy, J.~Hays, and Z.~Kira.
\newblock On convergence and stability of {GAN}s.
\newblock \emph{arXiv preprint arXiv:1705.07215}, 2017.

\bibitem[Kwon et~al.(2019)Kwon, Han, and Kim]{kwon2019generation}
G.~Kwon, C.~Han, and D.-S. Kim.
\newblock Generation of {3D} brain {MRI} using auto-encoding generative
  adversarial networks.
\newblock In \emph{Medical Image Computing and Computer Assisted
  Intervention--MICCAI 2019: 22nd International Conference, Shenzhen, China,
  October 13--17, 2019, Proceedings, Part III 22}, pages 118--126. Springer,
  2019.

\bibitem[Larsen et~al.(2016)Larsen, S{\o}nderby, Larochelle, and
  Winther]{larsen2016autoencoding}
A.~B.~L. Larsen, S.~K. S{\o}nderby, H.~Larochelle, and O.~Winther.
\newblock Autoencoding beyond pixels using a learned similarity metric.
\newblock In \emph{International conference on machine learning}, pages
  1558--1566. PMLR, 2016.

\bibitem[LeCun(1998)]{lecun1998mnist}
Y.~LeCun.
\newblock The {MNIST} database of handwritten digits.
\newblock 1998.
\newblock URL \url{http://yann.lecun.com/exdb/mnist/}.

\bibitem[Lee et~al.(2024)Lee, Nam, Fong, and Lee]{lee2024enhancing}
H.~Lee, G.~Nam, E.~Fong, and J.~Lee.
\newblock Enhancing transfer learning with flexible nonparametric posterior
  sampling.
\newblock \emph{arXiv preprint arXiv:2403.07282}, 2024.

\bibitem[Li et~al.(2015)Li, Swersky, and Zemel]{Li}
Y.~Li, K.~Swersky, and R.~Zemel.
\newblock Generative moment matching networks.
\newblock In \emph{International Conference on Machine Learning}, pages
  1718--1727. PMLR, 2015.

\bibitem[Liu et~al.(2015)Liu, Luo, Wang, and Tang]{liu2015faceattributes}
Z.~Liu, P.~Luo, X.~Wang, and X.~Tang.
\newblock Deep learning face attributes in the wild.
\newblock In \emph{Proceedings of International Conference on Computer Vision
  (ICCV)}, December 2015.

\bibitem[Lyddon et~al.(2018)Lyddon, Walker, and
  Holmes]{lyddon2018nonparametric}
S.~P. Lyddon, S.~G. Walker, and C.~C. Holmes.
\newblock Nonparametric learning from {B}ayesian models with randomized
  objective functions.
\newblock \emph{Advances in Neural Information Processing Systems}, 31, 2018.

\bibitem[Lyddon et~al.(2019)Lyddon, Holmes, and Walker]{lyddon2019general}
S.~P. Lyddon, C.~C. Holmes, and S.~G. Walker.
\newblock General {B}ayesian updating and the loss-likelihood bootstrap.
\newblock \emph{Biometrika}, 106\penalty0 (2):\penalty0 465--478, 2019.

\bibitem[Ma et~al.(2024)Ma, Dong, Xia, He, Chen, Ren, and
  Song]{ma2024multivariate}
Q.~Ma, M.~Dong, C.~Xia, X.~He, R.~Chen, M.~Ren, and M.~Song.
\newblock A multivariate normal distribution data generative model in
  small-sample-based fault diagnosis: Taking traction circuit breaker as an
  example.
\newblock \emph{IEEE Transactions on Intelligent Transportation Systems}, 2024.

\bibitem[Makhzani et~al.(2015)Makhzani, Shlens, Jaitly, Goodfellow, and
  Frey]{makhzani2015adversarial}
A.~Makhzani, J.~Shlens, N.~Jaitly, I.~Goodfellow, and B.~Frey.
\newblock Adversarial autoencoders.
\newblock \emph{arXiv preprint arXiv:1511.05644}, 2015.

\bibitem[Mescheder et~al.(2017)Mescheder, Nowozin, and
  Geiger]{mescheder2017adversarial}
L.~Mescheder, S.~Nowozin, and A.~Geiger.
\newblock Adversarial variational {B}ayes: Unifying variational autoencoders
  and generative adversarial networks.
\newblock In \emph{International Conference on Machine Learning}, pages
  2391--2400. PMLR, 2017.

\bibitem[M{\"u}ller(1997)]{muller1997integral}
A.~M{\"u}ller.
\newblock Integral probability metrics and their generating classes of
  functions.
\newblock \emph{Advances in Applied Probability}, 29\penalty0 (2):\penalty0
  429--443, 1997.

\bibitem[Nickparvar(2021)]{msoudnickparvar2021}
M.~Nickparvar.
\newblock Brain tumor {MRI} dataset, 2021.
\newblock \url{https://www.kaggle.com/dsv/2645886}.

\bibitem[Rosca et~al.(2017)Rosca, Lakshminarayanan, Warde-Farley, and
  Mohamed]{rosca2017variational}
M.~Rosca, B.~Lakshminarayanan, D.~Warde-Farley, and S.~Mohamed.
\newblock Variational approaches for auto-encoding generative adversarial
  networks.
\newblock \emph{arXiv preprint arXiv:1706.04987}, 2017.

\bibitem[Salimans et~al.(2016)Salimans, Goodfellow, Zaremba, Cheung, Radford,
  and Chen]{salimans2016improved}
T.~Salimans, I.~Goodfellow, W.~Zaremba, V.~Cheung, A.~Radford, and X.~Chen.
\newblock Improved techniques for training {GAN}s.
\newblock \emph{Advances in Neural Information Processing Systems}, 29, 2016.

\bibitem[Sethuraman(1994)]{sethuraman1994constructive}
J.~Sethuraman.
\newblock A constructive definition of {D}irichlet priors.
\newblock \emph{Statistica sinica}, pages 639--650, 1994.

\bibitem[Theis et~al.(2015)Theis, {van den Oord}, and Bethge]{theis2015note}
L.~Theis, A.~{van den Oord}, and M.~Bethge.
\newblock A note on the evaluation of generative models.
\newblock \emph{arXiv preprint arXiv:1511.01844}, 2015.

\bibitem[Ulyanov et~al.(2018)Ulyanov, Vedaldi, and Lempitsky]{ulyanov2018takes}
D.~Ulyanov, A.~Vedaldi, and V.~Lempitsky.
\newblock It takes (only) two: Adversarial generator-encoder networks.
\newblock In \emph{Proceedings of the AAAI Conference on Artificial
  Intelligence}, volume~32, 2018.

\bibitem[Villani(2008)]{villani2008optimal}
C.~Villani.
\newblock Optimal transport, old and new.
\newblock \emph{Grundlehren der mathematischen Wissenschaften}, 3, 2008.

\bibitem[Yang et~al.(2017)Yang, Hu, Salakhutdinov, and
  Berg-Kirkpatrick]{yang2017improved}
Z.~Yang, Z.~Hu, R.~Salakhutdinov, and T.~Berg-Kirkpatrick.
\newblock Improved variational autoencoders for text modeling using dilated
  convolutions.
\newblock In \emph{International conference on machine learning}, pages
  3881--3890. PMLR, 2017.

\bibitem[Zarepour and Al-Labadi(2012)]{zarepour2012rapid}
M.~Zarepour and L.~Al-Labadi.
\newblock On a rapid simulation of the {D}irichlet process.
\newblock \emph{Statistics \& Probability Letters}, 82\penalty0 (5):\penalty0
  916--924, 2012.

\bibitem[Zhu et~al.(1997)Zhu, Byrd, Lu, and Nocedal]{zhu1997algorithm}
C.~Zhu, R.~H. Byrd, P.~Lu, and J.~Nocedal.
\newblock Algorithm 778: {L-BFGS-B}: {F}ortran subroutines for large-scale
  bound-constrained optimization.
\newblock \emph{ACM Transactions on Mathematical Software}, 23\penalty0
  (4):\penalty0 550--560, 1997.

\end{thebibliography}

\appendix
\section*{Appendix}
\section{\textcolor{black}{Integral Probability Metrics}}
  Given a data space $\mathfrak{X}$, let $\mathbf{X}$ and $\mathbf{Y}$ be random variables that follow distributions $F_1$ and $F_2$, respectively, where $F_1, F_2 \in \mathcal{P}(\mathfrak{X})$, the set of all probability measures over the Borel $\sigma$-algebra on $\mathfrak{X}$. To measure the discrepancy between $F_1$ and $F_2$, an integral probability metric (IPM) is used, defined as:
\[
\text{IPM}(F_1, F_2) = \sup_{S \in \mathcal{S}} \left| E_{F_1}[S(\mathbf{X})] - E_{F_2}[S(\mathbf{Y})] \right|,
\]
where \( \mathcal{S} \) is a class of functions chosen to emphasize significant differences between \( F_1 \) and \( F_2 \). The specific choice of \( \mathcal{S} \) determines the particular distance used \citep{muller1997integral}.

\subsection{\textcolor{black}{Maximum Mean Discrepancy Distance}}

A prominent example of an IPM is the maximum mean discrepancy (MMD), which arises when $\mathcal{S}$ is taken as the unit ball in a reproducing kernel Hilbert space (RKHS) $\mathcal{H}_k$ defined by a positive-definite kernel $k: \mathfrak{X} \times \mathfrak{X} \rightarrow \mathbb{R}$ \citep{Gretton}. Specifically:
\[
\mathcal{S}_H = \{ H \in \mathcal{H}_k \mid \|H\|_{\mathcal{H}_k} \leq 1 \},
\]
where $\| \cdot \|_{\mathcal{H}_k}$ denotes the RKHS norm. The positive-definite nature of the kernel $k$ allows functions $H \in \mathcal{H}_k$ to be represented for any $\mathbf{X} \in \mathfrak{X}$ as:
\[
H(\mathbf{X}) = \langle H, k(\mathbf{X}, \cdot) \rangle_{\mathcal{H}_k},
\]
where $\langle \cdot, \cdot \rangle_{\mathcal{H}_k}$ is the inner product in $\mathcal{H}_k$. The \textit{kernel mean embedding} of a distribution $F$, as described by \cite{Gretton}, is defined as:
\[
\mu_{F}(\cdot) = E_{F}[k(\mathbf{X}, \cdot)] \in \mathcal{H}_k.
\]

The MMD between $F_1$ and $F_2$ is then expressed by:
\[
\text{MMD}(F_1, F_2) =\sup_{H\in\mathcal{S}_H} \left( E_{F_1}[H(\mathbf{X})] - E_{F_2}[H(\mathbf{Y})] \right) =\| \mu_{F_1} - \mu_{F_2} \|_{\mathcal{H}_k},
\]
which can be expanded as:
\[
\text{MMD}^2(F_1, F_2) = E_{F_1}[k(\mathbf{X}, \mathbf{X}')] - 2 E_{F_1, F_2}[k(\mathbf{X}, \mathbf{Y})] + E_{F_2}[k(\mathbf{Y}, \mathbf{Y}')],
\]
where $\mathbf{X}, \mathbf{X}' \overset{i.i.d.}{\sim} F_1$ and $\mathbf{Y}, \mathbf{Y}' \overset{i.i.d.}{\sim} F_{2}$, assuming $ E_F[\sqrt{k(\mathbf{X}, \mathbf{X})}] < \infty $ for all $ F \in \mathcal{P}(\mathfrak{X}) $.

In real scenarios, when the distributions $F_1$ and $F_2$ are unknown, the MMD is typically approximated using samples $\mathbf{X}_{1:n}:=(\mathbf{X}_1, \ldots, \mathbf{X}_n) \sim F_1$ and $\mathbf{Y}_{1:m}:=(\mathbf{Y}_1, \ldots, \mathbf{Y}_m) \sim F_2$. The V-statistic estimator using empirical distributions  \( F_{\mathbf{X}_{1:n}}:=\frac{1}{n}\sum_{i=1}^{n}\delta_{\mathbf{X}_i} \) and \( F_{\mathbf{Y}_{1:m}}:=\frac{1}{m}\sum_{i=1}^{m}\delta_{\mathbf{Y}_i} \) is represented by:

\begin{align}\label{MMD-ecdf}
    \text{MMD}^2(F_{\mathbf{X}_{1:n}}, F_{\mathbf{Y}_{1:m}}) = \frac{1}{n^2} \sum_{i=1}^{n} \sum_{j=1}^{n} k(\mathbf{X}_i, \mathbf{X}_j) - \frac{2}{nm} \sum_{i=1}^{n} \sum_{j=1}^{m} k(\mathbf{X}_i, \mathbf{Y}_j) + \frac{1}{m^2} \sum_{i=1}^{m} \sum_{j=1}^{m} k(\mathbf{Y}_i, \mathbf{Y}_j).
\end{align}

An alternative approach involves using an unbiased U-statistic estimator. The MMD satisfies the condition $\text{MMD}^2(F_1, F_2) = 0$ if and only if $F_1 = F_2$ (equality condition), assuming that $\mathcal{H}_k$ is a universal reproducing kernel Hilbert space (RKHS) over a compact metric space $\mathfrak{X}$ with a continuous kernel function $k$ \citep[Theorem 5]{Gretton}.

\subsection{\textcolor{black}{Wasserstein Distance}}
Given the compact metric space $ \mathfrak{X} $,
 the Wasserstein distance, defined using the Kantorovich-Rubinstein duality \citep{villani2008optimal}, is expressed as:
\begin{align}\label{wasserstein-general-def}
    \text{W}(F_{1}, F_{2}) = \sup_{D\in\mathcal{S}_D} 
    \left( E_{F_{1}}[D(\mathbf{X})] - E_{F_2}[D(\mathbf{Y})] \right),
\end{align}
where 
\begin{align*}
    \mathcal{S}_D = \{ D: \mathfrak{X} \rightarrow \mathbb{R}^l \mid \|D\|_{L} \leq 1 \}
\end{align*}
represents the set of all 1-Lipschitz functions, with $\|\cdot\|_L$ denoting the Lipschitz norm. The 1-Lipschitz condition implies that for any $\mathbf{X}_1, \mathbf{X}_2 \in \mathfrak{X}$,
\[
\|D(\mathbf{X}_1) - D(\mathbf{X}_2)\|_2 \leq \|\mathbf{X}_1 - \mathbf{X}_2\|_2,
\]
where $\|\cdot\|_2$ represents the Euclidean norm.

This representation of the Wasserstein distance highlights its role as an IPM. It satisfies the equality condition $\text{W}(F_{1}, F_{2}) = 0$ if and only if $F_1 = F_2$. Moreover, when \( D \) is differentiable, the 1-Lipschitz condition can be equivalently expressed as \( \| \nabla_{\mathbf{X}} D(\mathbf{X}) \|_2 \leq 1 \) for all \( \mathbf{X} \in \mathfrak{X} \), where \( \nabla_{\mathbf{X}} D(\mathbf{X}) \) denotes the gradient of \( D \) with respect to \( \mathbf{X} \). For the optimal differentiable 1-Lipschitz function \( D^{(\mathrm{opt})} \), it can be shown that \( \| \nabla_{\widehat{\mathbf{X}}} D^{(\mathrm{opt})}(\widehat{\mathbf{X}}) \|_2 = 1 \), almost everywhere, where \( \widehat{\mathbf{X}} = (1 - u)\mathbf{X} + u\mathbf{Y} \) for \( 0 \leq u \leq 1 \) \citep[Corollary 1]{gulrajani2017improved}. Consequently, it may simplify the optimization process to focus on the subclass of 1-Lipschitz functions that maintain a gradient norm of 1 at \( \widehat{\mathbf{X}} \).

In practical applications, the distance in \eqref{wasserstein-general-def} is estimated using the empirical distributions, yielding  
\begin{align*}
    \text{W}(F_{\mathbf{X}_{1:n}}, F_{\mathbf{Y}_{1:n}}) = \sup_{D \in \mathcal{S}_D} 
    \left( \sum_{i=1}^{n} \frac{D(\mathbf{X}_i)}{n} - \sum_{i=1}^{m} \frac{D(\mathbf{Y}_i)}{m} \right).
\end{align*}

\section{\textcolor{black}{Additional Theoretical Results}}

\subsection{Assessing the Accuracy of BNPL Estimation}\label{app:A}

\begin{theorem}\label{thm-Generror}
Assuming \eqref{cdf}-\eqref{DP-pos}, consider the generator 
\(\{ Gen_{\boldsymbol{\omega}} \}_{{\boldsymbol{\omega}} \in \boldsymbol{\Omega}}\), and let \(\boldsymbol{\omega}^{\star}_{BNPL}\) 
denote the optimized parameter value obtained through Algorithm \ref{alg2}. Let \(\boldsymbol{\omega}_{\text{True}}\) be the parameter value 
satisfying 
\begin{align}\label{accuracy-cond1}
    \text{WMMD}(F, F_{\text{Gen}_{\boldsymbol{\omega}_{\text{True}}}}) = 0,
\end{align}
in the well-specified case. Let \(k(\cdot, \cdot)\) be a continuous kernel function with a feature 
space corresponding to a universal RKHS, defined on a compact metric space \(X\), and assume 
$|k(t, t')| < K$ for all $t, t' \in \mathbb{R}^d.$
Furthermore, suppose there exists a finite constant \(M > 0\) such that 
\begin{align}\label{accuracy-cond2}
\max \big(\text{W}(F, F_{\mathbf{X}_{1:n}}), W(F, H^{\ast})\big) < M,
\end{align}
for any \(n \in \mathbb{N}\). Then, as $n,N\rightarrow\infty$,
\begin{align*}
    \lim\sup E\left[\text{WMMD}\left(F,F_{\text{Gen}_{\boldsymbol{\omega}^{\star}_{\text{BNPL}}}}\right)\right]=0.
\end{align*}
\end{theorem}
\begin{proof}
    The triangular inequality implies:
\begin{align*}
   \text{WMMD}(F,F_{\text{Gen}_{\boldsymbol{\omega}^{\star}_{\text{BNPL}}}}) &\leq \text{WMMD}(F, F^{\mathrm{pos}}_{N}) + \text{WMMD}(F^{\mathrm{pos}}_{N}, F_{\text{Gen}_{\boldsymbol{\omega}^{\star}_{\text{BNPL}}}}) = I_1.
\end{align*}
Regarding definition of $\boldsymbol{\omega}^{\star}_{\text{BNPL}} $, for any $\boldsymbol{\omega} \in \boldsymbol{\Omega}$, we have:
\begin{align*}
    I_1 &\leq \text{WMMD}(F, F^{\mathrm{pos}}_{N}) + \text{WMMD}(F^{\mathrm{pos}}_{N}, F_{\text{Gen}_{\boldsymbol{\omega}}}) \nonumber \\
    &\leq \text{WMMD}(F, F^{\mathrm{pos}}_{N})  + \text{WMMD}(F^{\mathrm{pos}}_{N},F ) 
    + \text{WMMD}(F,F_{\text{Gen}_{\boldsymbol{\omega}}})~~~~~~~~~~~\quad \text{{\footnotesize(Triangle inequality)}}\nonumber \\
    &\leq 2 \left[\text{WMMD}(F, F_{\mathbf{X}_{1:n}}) + \text{WMMD}(F_{\mathbf{X}_{1:n}}, F^{\mathrm{pos}}_{N}) \right] + \text{WMMD}(F,F_{\text{Gen}_{\boldsymbol{\omega}}}) ~~~\text{{\footnotesize(Triangle inequality)}} \nonumber \\
    &\leq 2 \left[\text{WMMD}(F, F_{\mathbf{X}_{1:n}}) + \text{WMMD}(F_{\mathbf{X}_{1:n}}, H^{\ast}) + \text{WMMD}(H^{\ast}, F^{\mathrm{pos}}_N)\right] + \text{WMMD}(F,F_{\text{Gen}_{\boldsymbol{\omega}}})\nonumber\\
    &\hspace{12.2cm}\text{{\footnotesize(Triangle inequality)}}
\end{align*}

Therefore,
\begin{align}\label{inequality-wmmd1}
    \text{WMMD}(F,F_{\text{Gen}_{\boldsymbol{\omega}^{\star}_{\text{BNPL}}}})&\leq 2 \left[\text{WMMD}(F, F_{\mathbf{X}_{1:n}}) + \text{WMMD}(F_{\mathbf{X}_{1:n}}, H^{\ast}) + \text{WMMD}( F^{\mathrm{pos}}_N,H^{\ast})\right]\nonumber\\
    &\hspace{.45cm}+\min\limits_{\boldsymbol{\omega}\in\boldsymbol{\Omega}} \text{WMMD}(F,F_{\text{Gen}_{\boldsymbol{\omega}}})\nonumber\\
    &=2 \left[\text{WMMD}(F, F_{\mathbf{X}_{1:n}}) + \text{WMMD}(F_{\mathbf{X}_{1:n}}, H^{\ast}) + \text{WMMD}( F^{\mathrm{pos}}_N,H^{\ast})\right]\nonumber\\
    &\hspace{7.7cm}\text{{\footnotesize(Appling assumption \eqref{accuracy-cond1})}}\nonumber\\
    &=2 [\text{W}(F, F_{\mathbf{X}_{1:n}}) + \text{W}(F_{\mathbf{X}_{1:n}}, H^{\ast}) + \text{W}( F^{\mathrm{pos}}_N,H^{\ast})\nonumber\\
    &~~~~~~~~~+\text{MMD}(F, F_{\mathbf{X}_{1:n}}) + \text{MMD}(F_{\mathbf{X}_{1:n}}, H^{\ast}) + \text{MMD}( F^{\mathrm{pos}}_N,H^{\ast})]\nonumber\\
    &=2 [2\text{W}(F, F_{\mathbf{X}_{1:n}}) + \text{W}(F, H^{\ast}) + \text{W}( F^{\mathrm{pos}}_N,H^{\ast})\nonumber\\
    &~~~~~~~~~+\text{MMD}(F, F_{\mathbf{X}_{1:n}}) + \text{MMD}(F_{\mathbf{X}_{1:n}}, H^{\ast}) + \text{MMD}( F^{\mathrm{pos}}_N,H^{\ast})]\nonumber\\
    &\hspace{8.5cm}\text{{\footnotesize(Triangle inequality)}}.
\end{align}
 Taking the expectation of both sides of \eqref{inequality-wmmd1}, we get:
 
 \begin{align}\label{inequality-wmmd2}
    E\left[\text{WMMD}\left(F,F_{\text{Gen}_{\boldsymbol{\omega}^{\star}_{\text{BNPL}}}}\right)\right]&\leq 2 \lbrace 2E[\text{W}(F, F_{\mathbf{X}_{1:n}})] + E[\text{W}(F, H^{\ast})] + E[\text{W}( F^{\mathrm{pos}}_N,H^{\ast})]\nonumber\\
    &~~~~~~~+E[\text{MMD}(F, F_{\mathbf{X}_{1:n}})] + E[\text{MMD}(F_{\mathbf{X}_{1:n}}, H^{\ast})] \nonumber\\
    &~~~~~~~+E [\text{MMD}( F^{\mathrm{pos}}_N,H^{\ast})]\rbrace
\end{align}
Using \citet[Lemma 7.1]{cherief2022finite}, \citet[Lemma 8]{dellaporta2022robust}, and the proof of \citet[Lemma 4]{fazeli2023semi}, it follows that:
\begin{multline}\label{inequality-mmd}
    E[\text{MMD}(F, F_{\mathbf{X}_{1:n}})] + E[\text{MMD}(F_{\mathbf{X}_{1:n}}, H^{\ast})]
   +E [\text{MMD}( F^{\mathrm{pos}}_N,H^{\ast})]\leq\frac{2K}{\sqrt{n}} + \frac{4aK}{a+n}\\ + 2 \sqrt{\frac{(a+n+N)K}{(a+n+1)N}}. 
\end{multline}

Given the optimized value \( \boldsymbol{\theta}^{\star} \in \boldsymbol{\Theta} \) for approximating \( \text{W}(F^{\mathrm{pos}}_N, H^\ast) \), consider two independent copies of the random variables, \( \mathbf{X}_{1:N}^{\mathrm{pos}} \) and \( \mathbf{X}_{1:N}^{\mathrm{pos}\prime} \), where \( \mathbf{X}_{1:N}^{\mathrm{pos}}, \mathbf{X}_{1:N}^{\mathrm{pos}\prime} \overset{i.i.d.}{\sim} H^\ast \). Then, it follows that
\begin{align}\label{zero-expectation}
    E[\text{W}(F^{\mathrm{pos}}_N, H^\ast)]&=E\sum_{i=1}^{N}\left( J^{\mathrm{pos}}_{i,N}Dis_{\boldsymbol{\theta}^{\star}}(\mathbf{X}_{i}^{\mathrm{pos}})
    -\frac{Dis_{\boldsymbol{\theta}^{\star}}(\mathbf{X}_{i}^{\mathrm{pos}\prime})}{N}\right)\nonumber\\
    &=\sum_{i=1}^{N}\left( \frac{1}{N}E_{H^{\ast}}(Dis_{\boldsymbol{\theta}^{\star}}(\mathbf{X}_{i}^{\mathrm{pos}}))
    -\frac{E_{H^{\ast}}(Dis_{\boldsymbol{\theta}^{\star}}(\mathbf{X}_{i}^{\mathrm{pos}\prime}))}{N}\right)\hspace{.5cm}\text{{\footnotesize(Dirichlet variables)}}\nonumber\\
    &= E_{H^{\ast}}(Dis_{\boldsymbol{\theta}^{\star}}(\mathbf{X}_{1}^{\mathrm{pos}}))
    -E_{H^{\ast}}(Dis_{\boldsymbol{\theta}^{\star}}(\mathbf{X}_{1}^{\mathrm{pos}\prime}))\hspace{1.1cm}\text{{\footnotesize(Identical random variables)}}\nonumber\\
    &=0\hspace{7.6cm}\text{{\footnotesize(Identical random variables).}}
\end{align}

Applying \eqref{inequality-mmd} and \eqref{zero-expectation} to \eqref{inequality-wmmd2}, and then taking the $\limsup$ of both sides of \eqref{inequality-wmmd2} as $n$ and $N$ go to infinity, we obtain:
\begin{align}\label{faou-lemma}
    0\leq\limsup E\left[\text{WMMD}\left(F,F_{\text{Gen}_{\boldsymbol{\omega}^{\star}_{\text{BNPL}}}}\right)\right]&\leq 4\limsup E(\text{W}(F, F_{\mathbf{X}_{1:n}})) + 2\limsup E(\text{W}(F, H^{\ast}))\nonumber\\
    &\leq 4E(\limsup \text{W}(F, F_{\mathbf{X}_{1:n}}))+2E(\limsup\text{W}(F, H^{\ast})),
\end{align}
where the last inequality follows from considering assumption \eqref{accuracy-cond2} and then applying the reverse Fatou's lemma.

Finally, considering the Glivenko-Cantelli theorem, which ensures that $F_{\mathbf{X}_{1:n}}$ converges to $F$ and, consequently, that $H^{\ast}$ converges to $F$, we then use \citet[Theorem 2]{arjovsky2017wasserstein} on the right-hand side of \eqref{faou-lemma}, completing the proof.
\end{proof}

\subsection{Assessing BNPL Estimation in the Presence of Outliers}\label{app:B}
\begin{theorem}\label{thm-robusst}
    Consider Huber’s contamination model \citep{cherief2022finite}, defined as \( F = (1 - \varepsilon)F_0 + \varepsilon Q \), where \( F_0 \) represents the clean distribution, \( Q \) denotes the noise distribution accounting for outliers, and \( \varepsilon \in (0, \frac{1}{2}) \) is the contamination rate. \( \mathbf{X}_i \sim F_0 \) if \( \zeta_i = 0 \); otherwise, \( X_i \sim Q \), where \( \zeta_1, \dots, \zeta_n\overset{i.i.d.}{\sim} \text{Bernoulli}(\varepsilon) \). Assuming $\text{W}(F_0,Q)=c$, $c<\infty$, then, under assumption of Theorem \ref{thm-Generror}, as $n,N\rightarrow\infty$,

    \begin{align*}
    \lim\sup E\left[\text{WMMD}\left(F_0,F_{\text{Gen}_{\boldsymbol{\omega}^{\star}_{\text{BNPL}}}}\right)\right]\leq(2+c)\varepsilon.
    \end{align*}
\end{theorem}
\begin{proof}
    Regarding \citet[Lemma 3.3]{cherief2022finite}, for $\boldsymbol{\omega}^{\star}_{\text{BNPL}}\in \boldsymbol{\Omega}$,
    \begin{align*}
        \mid \text{MMD}(F,F_{\text{Gen}_{\boldsymbol{\omega}^{\star}_{\text{BNPL}}}})-\text{MMD}(F_0,F_{\text{Gen}_{\boldsymbol{\omega}^{\star}_{\text{BNPL}}}})\mid\leq2\varepsilon.
    \end{align*}

It implies that
\begin{align}\label{upperbound}
    \text{WMMD}\left(F_0,F_{\text{Gen}_{\boldsymbol{\omega}^{\star}_{\text{BNPL}}}}\right)&\leq \text{W}\left(F_0,F_{\text{Gen}_{\boldsymbol{\omega}^{\star}_{\text{BNPL}}}}\right)+\text{MMD}\left(F,F_{\text{Gen}_{\boldsymbol{\omega}^{\star}_{\text{BNPL}}}}\right)+2\varepsilon\nonumber\\
    &\leq \text{W}\left(F_0,F\right)+\text{WMMD}\left(F,F_{\text{Gen}_{\boldsymbol{\omega}^{\star}_{\text{BNPL}}}}\right)+2\varepsilon\hspace{.5cm}\text{{\footnotesize(Triangle inequality)}}.
\end{align}

On the other hand, since $F=(1-\varepsilon)F_{0}+\varepsilon Q$, for any $\boldsymbol{\theta}\in\boldsymbol{\Theta}$, we can derive the following:
\begin{multline}\label{mix-wasser}
    E_{F_0}[Dis_{\boldsymbol{\theta}}(\mathbf{X})]
    -E_{F}[Dis_{\boldsymbol{\theta}}(\mathbf{Y})]=\int_{\mathfrak{X}}Dis_{\boldsymbol{\theta}}(\mathbf{x})\,dF_{0}(\mathbf{x})-\int_{\mathfrak{X}}Dis_{\boldsymbol{\theta}}(\mathbf{y})\,dF(\mathbf{y})\\
    =\int_{\mathfrak{X}}Dis_{\boldsymbol{\theta}}(\mathbf{x})\,dF_{0}(\mathbf{x})
    -(1-\varepsilon)\int_{\mathfrak{X}}Dis_{\boldsymbol{\theta}}(\mathbf{y})\,dF_{0}(\mathbf{y})-\varepsilon\int_{\mathfrak{X}}Dis_{\boldsymbol{\theta}}(\mathbf{y})\,dQ(\mathbf{y})\\
    =\varepsilon \left(\int_{\mathfrak{X}}Dis_{\boldsymbol{\theta}}(\mathbf{y})\,dF_{0}(\mathbf{y})-\int_{\mathfrak{X}}Dis_{\boldsymbol{\theta}}(\mathbf{y})\,dQ(\mathbf{y})\right)\\
    = \varepsilon\left(E_{F_0}[Dis_{\boldsymbol{\theta}}(\mathbf{Y})]
    -E_{Q}[Dis_{\boldsymbol{\theta}}(\mathbf{Y})]\right).
\end{multline}

Taking the supremum of both sides of \eqref{mix-wasser} over $\boldsymbol{\Theta}$, we obtain:
\begin{align}\label{cEps}
    \text{W}(F_0,F)=\varepsilon\text{W}(F_0,Q)=\varepsilon c.
\end{align}

Now, considering \eqref{cEps} in \eqref{upperbound} and taking the expectation, we get:
\begin{align}\label{result-upperb}
    E\left(\text{WMMD}\left(F_0,F_{\text{Gen}_{\boldsymbol{\omega}^{\star}_{\text{BNPL}}}}\right)\right)&\leq(2+c)\varepsilon+E\left(\text{WMMD}\left(F,F_{\text{Gen}_{\boldsymbol{\omega}^{\star}_{\text{BNPL}}}}\right)\right).
\end{align}

Finally, taking the $\limsup$
 of both sides of \eqref{result-upperb} and applying Theorem \ref{thm-Generror} completes the proof.

\end{proof}
\end{document}